\documentclass{article}
\usepackage{amsfonts} 
\usepackage{nips14submit_e,times}
\usepackage[utf8]{inputenc} 
\usepackage[T1]{fontenc}    
\usepackage{hyperref}       
\usepackage{url}            
\usepackage{amsfonts}       
\usepackage{nicefrac}       
\usepackage{microtype}      
\usepackage{amssymb}
\usepackage{amsmath}
\usepackage{amsthm}
\usepackage{array}
\usepackage{algorithm}
\usepackage{algorithmic}
\usepackage[tight,footnotesize]{subfigure}
\usepackage{graphicx}
\usepackage{mathrsfs}
\usepackage{multirow}
\usepackage{mdwmath}
\usepackage{mdwtab}

\newtheorem{theorem}{Theorem}
\newtheorem{lemma}{Lemma}
\newtheorem{corollary}{Corollary}
\newtheorem{property}{Property}
\newtheorem{assumption}{Assumption}

\title{Guaranteed Sufficient Decrease for Variance Reduced Stochastic Gradient Descent}

\author{
Fanhua Shang$^{\dag}$, \,Yuanyuan Liu$^{\dag}$, \,James Cheng$^{\dag}$, \,Kelvin K.W. Ng$^{\dag}$, \,Yuichi Yoshida$^{\ddag}$\\
$^{\dag}\;\!$Department of Computer Science and Engineering, The Chinese University of Hong Kong\\
$^{\ddag}\;\!$National Institute of Informatics and Preferred Infrastructure, Inc., Tokyo, Japan}

\nipsfinalcopy 

\begin{document}

\maketitle

\begin{abstract}
In this paper, we propose a novel sufficient decrease technique for variance reduced stochastic gradient descent methods such as SAG, SVRG and SAGA. In order to make sufficient decrease for stochastic optimization, we design a new sufficient decrease criterion, which yields sufficient decrease versions of variance reduction algorithms such as SVRG-SD and SAGA-SD as a byproduct. We introduce a coefficient to scale current iterate and satisfy the sufficient decrease property, which takes the decisions to shrink, expand or move in the opposite direction, and then give two specific update rules of the coefficient for Lasso and ridge regression. Moreover, we analyze the convergence properties of our algorithms for strongly convex problems, which show that both of our algorithms attain linear convergence rates. We also provide the convergence guarantees of our algorithms for non-strongly convex problems. Our experimental results further verify that our algorithms achieve significantly better performance than their counterparts.
\end{abstract}

\section{Introduction}
Stochastic gradient descent (SGD) has been successfully applied to many large scale machine learning problems~\cite{krizhevsky:deep,zhang:sgd}, by virtue of its low per-iteration cost. However, standard SGD estimates the gradient from only one or a few samples, and thus the variance of the stochastic gradient estimator may be large~\cite{johnson:svrg,zhao:prox-smd}, which leads to slow convergence and poor performance. In particular, even under the strongly convex (SC) condition, the convergence rate of standard SGD is only sub-linear. Recently, the convergence rate of SGD has been improved by various variance reduction methods, such as SAG~\cite{roux:sag}, SDCA~\cite{shalev-Shwartz:sdca}, SVRG~\cite{johnson:svrg}, SAGA~\cite{defazio:saga}, Finito~\cite{defazio:Finito}, MISO~\cite{mairal:miso}, and their proximal variants, such as~\cite{schmidt:sag}, \cite{shalev-Shwartz:prox-sdca} and \cite{xiao:prox-svrg}. Under the SC condition, these variance reduced SGD (VR-SGD) algorithms achieve linear convergence rates.

Very recently, many techniques were proposed to further speed up the VR-SGD methods mentioned above. These techniques include importance sampling~\cite{zhao:prox-smd}, exploiting neighborhood structure in the training data to share and re-use information about past stochastic gradients~\cite{hofmann:vrsg}, incorporating Nesterov's acceleration techniques~\cite{lin:vrsg,nitanda:svrg} or momentum acceleration tricks~\cite{zhu:Katyusha}, reducing the number of gradient computations in the early iterations~\cite{zhu:univr,babanezhad:vrsg,zhang:svrg}, and the projection-free property of the conditional gradient method~\cite{hazan:svrf}. \cite{zhu:vrnc} and \cite{reddi:saga} proved that SVRG and SAGA with minor modifications can asymptotically converge to a stationary point for non-convex problems.

So far the two most popular stochastic gradient estimators are the SVRG estimator independently introduced by~\cite{johnson:svrg,zhang:svrg} and the SAGA estimator~\cite{defazio:saga}. All these estimators may be very different from their full gradient counterparts, thus moving in the direction may not decrease the objective function anymore, as stated in~\cite{zhu:Katyusha}. To address this problem, inspired by the success of sufficient decrease methods for deterministic optimization such as~\cite{li:apg,wolfe:sdg}, we propose a novel sufficient decrease technique for a class of VR-SGD methods, including the widely-used SVRG and SAGA methods. Notably, our method with partial sufficient decrease achieves average time complexity per-iteration as low as the original SVRG and SAGA methods. We summarize our main contributions below.

\begin{itemize}
  \item For making sufficient decrease for stochastic optimization, we design a sufficient decrease strategy to further reduce the cost function, in which we also introduce a coefficient to take the decisions to shrink, expand or move in the opposite direction.
  \item We incorporate our sufficient decrease technique, together with momentum acceleration, into two representative SVRG and SAGA algorithms, which lead to SVRG-SD and SAGA-SD. Moreover, we give two specific update rules of the coefficient for Lasso and ridge regression problems as notable examples.
  \item  Moreover, we analyze the convergence properties of SVRG-SD and SAGA-SD, which show that SVRG-SD and SAGA-SD converge linearly for SC objective functions. Unlike most of the VR-SGD methods, we also provide the convergence guarantees of SVRG-SD and SAGA-SD for non-strongly convex (NSC) problems.
  \item Finally, we show by experiments that SVRG-SD and SAGA-SD achieve significantly better performance than SVRG~\cite{johnson:svrg} and SAGA~\cite{defazio:saga}. Compared with the best known stochastic method, Katyusha~\cite{zhu:Katyusha}, our methods also have much better performance in most cases.
\end{itemize}

\section{Preliminary and Related Work}
In this paper, we consider the following composite convex optimization problem:
\begin{equation}\label{equ1}
\min_{x\in\mathbb{R}^{d}} F(x)\stackrel{\rm{def}}{=}f(x)+r(x)=\frac{1}{n}\!\sum\nolimits_{i=1}^{n}\!f_{i}(x)+r(x),
\end{equation}
where $f_{i}(x)\!:\!\mathbb{R}^{d}\!\rightarrow\!\mathbb{R},\,i\!=\!1,\ldots,n$ are the smooth convex component functions, and $r(x)$ is a relatively simple convex (but possibly non-differentiable) function. Recently, many VR-SGD methods~\cite{johnson:svrg,roux:sag,xiao:prox-svrg,zhang:svrg} have been proposed for special cases of \eqref{equ1}. Under smoothness and SC assumptions, and $r(x)\!\equiv\!0$, SAG~\cite{roux:sag} achieves a linear convergence rate. A recent line of work, such as~\cite{johnson:svrg,xiao:prox-svrg}, has been proposed with similar convergence rates to SAG but without the memory requirements for all gradients. SVRG~\cite{johnson:svrg} begins with an initial estimate $\widetilde{x}$, sets $x_{0}\!=\!\widetilde{x}$ and then generates a sequence of $x_{k}$ ($k=1,2,\ldots,m$, where $m$ is usually set to $2n$) using
\begin{eqnarray}\label{equ2}
x_{k}=x_{k-\!1}\!-\eta\left[\nabla\! f_{i_{k}}\!(x_{k-\!1})-\nabla\! f_{i_{k}}\!(\widetilde{x})+\widetilde{\mu}\right],
\end{eqnarray}
where $\eta\!>\!0$ is the step size, $\widetilde{\mu}\!:=\frac{1}{n}\!\sum^{n}_{i=1}\!\nabla\! f_{i}(\widetilde{x})$ is the full gradient at $\widetilde{x}$, and $i_{k}$ is chosen uniformly at random from $\{1,2,\ldots,n\}$. After every $m$ stochastic iterations, we set $\widetilde{x}\!=\!x_{m}$, and reset $k\!=\!1$ and $x_{0}\!=\!\widetilde{x}$. Unfortunately, most of the VR-SGD methods~\cite{defazio:Finito,shalev-Shwartz:sdca,xiao:prox-svrg}, including SVRG, only have convergence guarantee for smooth and SC problems. However, $F(\cdot)$ may be NSC in many machine learning applications, such as Lasso. \cite{defazio:saga} proposed SAGA, a fast incremental gradient method in the spirit of SAG and SVRG, which works for both SC and NSC objective functions, as well as in proximal settings. Its main update rule is formulated as follows:
\vspace{-1mm}
\begin{eqnarray}\label{equ3}
x_{k}=\textrm{prox}^{r}_{\eta}(x_{k-\!1}-\eta\;\![g^{k}_{i_{k}}-g^{k-1}_{i_{k}}+\frac{1}{n}\!\sum^{n}_{j=1}g^{k-1}_{j}]),
\end{eqnarray}
where $g^{k}_{j}$ is updated for all $j\!=\!1,\ldots,n$ as follows: $g^{k}_{j}\!=\!\nabla\!f_{i_{k}}\!(x^{k-\!1})$ if $i_{k}\!=\!j$, and $g^{k}_{j}\!=\!g^{k-\!1}_{j}$ otherwise, and the proximal operator is defined as: $\textrm{prox}^{r}_{\eta}(y)=\arg\min_{x}({1}/{2\eta})\!\cdot\!\|x\!-\!y\|^{2}+r(x)$.

The technique of sufficient decrease (e.g., the well-known line search technique~\cite{more:ls}) has been studied for deterministic optimization~\cite{li:apg,wolfe:sdg}. For example, \cite{li:apg} proposed the following sufficient decrease condition for deterministic optimization:
\begin{equation}\label{equ4}
F(x_{k})\leq F(x_{k-1})-\delta\|y_{k}-x_{k-1}\|^{2},
\end{equation}
where $\delta\!>\!0$ is a small constant, and $y_{k}\!=\!\textrm{prox}^{r}_{\eta_{k}}\!(x_{k-\!1}\!-\!\eta_{k}\nabla\! f(x_{k-\!1}))$.
Similar to the strategy for deterministic optimization, in this paper we design a novel sufficient decrease technique for stochastic optimization, which is used to further reduce the cost function and speed up its convergence.

\section{Variance Reduced SGD with Sufficient Decrease}
In this section, we propose a novel sufficient decrease technique for VR-SGD methods, which include the widely-used SVRG and SAGA methods. To make sufficient decrease for stochastic optimization, we design a sufficient decrease strategy to further reduce the cost function. Then a coefficient $\theta$ is introduced to satisfy the sufficient decrease condition, and takes the decisions to shrink, expand or move in the opposite direction. Moreover, we present two sufficient decrease VR-SGD algorithms with momentum acceleration: SVRG-SD and SAGA-SD. We also give two specific schemes to compute $\theta$ for Lasso and ridge regression.

\subsection{Our Sufficient Decrease Technique}
\label{sec31}
Suppose $x^{s}_{k}\!=\!\textrm{prox}^{r}_{\eta}(x^{s}_{k-\!1}\!-\!\eta[\nabla\! f_{i^{s}_{k}}\!(x^{s}_{k-\!1})\!-\!\nabla\! f_{i^{s}_{k}}\!(\widetilde{x}^{s-\!1})\!+\!\widetilde{\mu}^{s-\!1}])$ for the $s$-th outer-iteration and the $k$-th inner-iteration. Unlike the full gradient method, the stochastic gradient estimator is somewhat inaccurate (i.e., it may be very different from $\nabla\! f(x^{s}_{k-\!1})$), then further moving in the updating direction may not decrease the objective value anymore~\cite{zhu:Katyusha}. That is, $F(x^{s}_{k})$ may be larger than $F(x^{s}_{k-\!1})$ even for very small step length $\eta\!>\!0$. Motivated by this observation, we design a factor $\theta$ to scale the current iterate $x^{s}_{k-\!1}$ for the decrease of the objective function. For SVRG-SD, the cost function with respect to $\theta$ is formulated as follows:
\vspace{-1mm}
\begin{equation}\label{equ5}
\min_{\theta\in\mathbb{R}} F(\theta x^{s}_{k-\!1})\!+\!\frac{\zeta(1\!-\!\theta)^2}{2}\!\|\nabla\!f_{i^{s}_{k}}\!(x^{s}_{k-\!1})-\!\nabla\! f_{i^{s}_{k}}\!(\widetilde{x}^{s-\!1})\|^{2},
\end{equation}
where $\zeta\!=\!\frac{\delta\eta}{1-L\eta}$ is a trade-off parameter between the two terms, $\delta$ is a small constant and set to 0.1. The second term in~\eqref{equ5} involves the norm of the residual of stochastic gradients, and plays the same role as the second term of the right-hand side of~\eqref{equ4}. Different from existing sufficient decrease techniques including \eqref{equ4}, a varying factor $\theta$ instead of a constant is introduced to scale $x^{s}_{k-\!1}$ and the coefficient of the second term of~\eqref{equ5}, and $\theta$ plays a similar role as the step-size parameter optimized via a line-search for deterministic optimization. However, line search techniques have a high computational cost in general, which limits their applicability to stochastic optimization~\cite{mahsereci:sgd}.

For SAGA-SD, the cost function with respect to $\theta$ can be revised by simply replacing $\nabla\! f_{i^{s}_{k}}\!(\widetilde{x}^{s-\!1})$ with $\nabla\! f_{i^{s}_{k}}\!(\phi^{k-\!1}_{i^{s}_{k}})$ defined below. Note that $\theta$ is a scalar and takes the decisions to shrink, expand $x^{s}_{k-\!1}$ or move in the opposite direction of $x^{s}_{k-\!1}$. The detailed schemes to calculate $\theta$ for Lasso and ridge regression are given in Section~\ref{subsec33}. We first present the following sufficient decrease condition in the statistical sense for stochastic optimization.

\begin{property}
\label{prop11}
For given $x^{s}_{k-\!1}$ and the solution $\theta_{k}$ of \eqref{equ5}, then the following inequality holds
\vspace{-1mm}
\begin{equation}\label{equ6}
F(\theta_{k}x^{s}_{k-\!1})\leq F(x^{s}_{k-\!1})-\frac{\zeta(1\!-\!\theta_{k})^2}{2}\|\widetilde{p}_{i^{s}_{k}}\|^{2},
\end{equation}
where $\widetilde{p}_{i^{s}_{k}}\!=\!\nabla\!f_{i^{s}_{k}}\!(x^{s}_{k-\!1})\!-\!\nabla\! f_{i^{s}_{k}}\!(\widetilde{x}^{s-\!1})$ for SVRG-SD.
\end{property}

It is not hard to verify that $F(\cdot)$ can be further decreased via our sufficient decrease technique, when the current iterate $x^{s}_{k-\!1}$ is scaled by the coefficient $\theta_{k}$. Indeed, for the special case when $\theta_{k}\!=\!1$ for some $k$, the inequality in (\ref{equ6}) can be still satisfied. Moreover, Property~\ref{prop11} can be extended for SAGA-SD by setting $\widetilde{p}_{i^{s}_{k}}\!=\!\nabla\!f_{i^{s}_{k}}\!(x^{s}_{k-\!1})\!-\!\nabla\! f_{i^{s}_{k}}\!(\phi^{k-\!1}_{i^{s}_{k}})$, as well as for other VR-SGD algorithms such as SAG and SDCA. Unlike the sufficient decrease condition for deterministic optimization~\cite{li:apg,wolfe:sdg}, $\theta_{k}$ may be a negative number, which means to move in the opposite direction of $x^{s}_{k-\!1}$.

\subsection{Momentum Acceleration}
\label{subsec32}
In this part, we first design the update rule for the key variable $x^{s}_{k}$ with the coefficient $\theta_{k}$ as follows:
\begin{equation}\label{equ7}
x^{s}_{k}=y^{s}_{k}+(1\!-\!\sigma)(\widehat{x}^{s}_{k}-\widehat{x}^{s}_{k-1}),
\end{equation}
where $\widehat{x}^{s}_{k}\!=\!\theta_{k}x^{s}_{k-\!1}$, $\sigma\!\in\![0,1]$ is a constant and can be set to $\sigma\!=\!1/2$ which also works well in practice. In fact, the second term of the right-hand side of~\eqref{equ7} plays a momentum acceleration role as in batch and stochastic optimization~\cite{zhu:Katyusha,nesterov:co,nitanda:svrg}. That is, by introducing this term, we can utilize the previous information of gradients to update $x^{s}_{k}$. In addition, the update rule of $y^{s}_{k}$ is given by
\begin{equation}\label{equ8}
y^{s}_{k}=\textrm{prox}^{r}_{\eta}\!\left(x^{s}_{k-1}-\eta\widetilde{\nabla} f_{i^{s}_{k}}(x^{s}_{k-1})\right),
\end{equation}
where $\eta\!=\!1/(L\alpha)$, $L\!>\!0$ is a Lipschitz constant (see Assumption~\ref{assum1} below), $\alpha\!\geq\! 1$ denotes a constant, and $\widetilde{\nabla}\! f_{i^{s}_{k}}\!(x^{s}_{k-\!1})$ can be the two most popular choices for stochastic gradient estimators: the SVRG estimator~\cite{johnson:svrg,zhang:svrg} for SVRG-SD and the SAGA estimator~\cite{defazio:saga} for SAGA-SD defined as follows:
\begin{eqnarray}\label{equ9}
\widetilde{\nabla}\! f_{i^{s}_{k}}\!(x^{s}_{k-\!1})\!=\!\nabla\! f_{i^{s}_{k}}\!(x^{s}_{k-\!1})\!-\!\nabla\! f_{i^{s}_{k}}\!(\widetilde{x}^{s-\!1})\!+\!\widetilde{\mu}^{s-\!1}\;\textrm{and}\;
\widetilde{\nabla}\! f_{i^{s}_{k}}\!(x^{s}_{k-\!1})\!=\!\nabla\! f_{i^{s}_{k}}\!(x^{s}_{k-\!1})\!-\!\nabla\! f_{i^{s}_{k}}\!(\phi^{k-\!1}_{i^{s}_{k}})\!+\!\overline{\mu}^{s-\!1}\!,
\end{eqnarray}
respectively, where $\widetilde{\mu}^{s-\!1}\!:=\!\nabla\! f(\widetilde{x}^{s-\!1})$. For SAGA-SD, we need to set $\phi^{k}_{i^{s}_{k}}\!=\!x_{k-\!1}$, and store $\nabla\! f_{i^{s}_{k}}\!(\phi^{k}_{i^{s}_{k}})$ in the table similar to~\cite{defazio:saga}. All the other entries in the table remain unchanged, and $\overline{\mu}^{s-\!1}\!:=\frac{1}{n}\!\sum^{n}_{j=1}\!\!\nabla\! f_{j}(\phi^{k-\!1}_{j})$ is the table average. From~\eqref{equ8}, it is clear that our algorithms can tackle  non-smooth problems directly as in~\cite{defazio:saga}.

\begin{algorithm}[t]
\caption{SVRG-SD}
\label{alg1}
\renewcommand{\algorithmicrequire}{\textbf{Input:}}
\renewcommand{\algorithmicensure}{\textbf{Initialize:}}
\renewcommand{\algorithmicoutput}{\textbf{Output:}}
\begin{algorithmic}[1]
\REQUIRE the number of epochs $S$, the number of iterations $m$ per epoch, and step size $\eta$.\\
\ENSURE $\widetilde{x}^{0}$ for Case of SC or $\widetilde{x}^{0}\!=\widetilde{y}^{0}$ for Case of NSC.
\FOR{$s=1,2,\ldots, S$}
\STATE {Case of SC: $\,x^{s}_{0}\!=\widehat{x}^{s}_{0}\!=\widetilde{x}^{s-1}$, \;or\; \textrm{Case of NSC:} $\,x^{s}_{0}\!=\widehat{x}^{s}_{0}\!=\widetilde{y}^{s-1}$;}
\STATE {$\widetilde{\mu}^{s-\!1}=\frac{1}{n}\!\sum^{n}_{i=1}\!\nabla\!f_{i}(\widetilde{x}^{s-\!1})$;}
\FOR{$k=1,2,\ldots,m$}
\STATE {Pick $i^{s}_{k}$ uniformly at random from $\{1,2,\ldots,n\}$;}
\STATE {Compute $\widetilde{\nabla}\!f_{i^{s}_{k}}\!(x^{s}_{k-\!1})$ and $y^{s}_{k}$ by \eqref{equ9} and \eqref{equ8};}
\STATE {Update $\theta_{k}$ and $x^{s}_{k}$ by \eqref{equ5} and \eqref{equ7};}
\ENDFOR
\STATE {$\widetilde{x}^{s}=\frac{1}{m}\!\sum^{m}_{k=1}\!\widehat{x}^{s}_{k}$;\quad \textrm{Case of NSC:} $\widetilde{y}^{s}=[x^{s}_{m}\!-\!(1\!-\!\sigma)\widehat{x}^{s}_{m}]/\sigma$;}
\ENDFOR
\OUTPUT $\overline{x}\!=\!\widetilde{x}^{S}$ (SC) \,or \,$\overline{x}\!=\!\widetilde{x}^{S}$ if $F(\widetilde{x}^{S})\leq F(\frac{1}{S}\!\sum^{S}_{s=1}\!\widetilde{x}^{s})$ and $\overline{x}\!=\!\frac{1}{S}\!\sum^{S}_{s=1}\!\widetilde{x}^{s}$ otherwise (NSC)
\end{algorithmic}
\end{algorithm}

In summary, we propose a novel variant of SVRG with sufficient decrease (SVRG-SD) to solve both SC and NSC problems, as outlined in \textbf{Algorithm} \ref{alg1}. For the case of SC, $x^{s}_{0}\!=\!\widehat{x}^{s}_{0}\!=\!\widetilde{x}^{s-\!1}$, while $x^{s}_{0}\!=\!\widehat{x}^{s}_{0}\!=\!\widetilde{y}^{s-\!1}$ and $\widetilde{y}^{s}\!=\![x^{s}_{m}\!-\!(1\!-\!\sigma)\widehat{x}^{s}_{m}]/\sigma$ for the case of NSC. Similarly, we also present a novel variant of SAGA with sufficient decrease (SAGA-SD), as shown in the Supplementary Material. The main differences between them are the stochastic gradient estimators in \eqref{equ9}, and the update rule of the sufficient decrease coefficient in \eqref{equ5}.

Note that when $\theta_{k}\!\equiv\!1$ and $\sigma\!=\!1$, the proposed SVRG-SD and SAGA-SD degenerate to the original SVRG or its proximal variant (Prox-SVRG~\cite{xiao:prox-svrg}) and SAGA~\cite{defazio:saga}, respectively. In this sense, SVRG, Prox-SVRG and SAGA can be seen as the special cases of the proposed algorithms. Like SVRG and SVRG-SD, SAGA-SD is also a multi-stage algorithm, whereas SAGA is a single-stage algorithm.

\subsection{Coefficients for Lasso and Ridge Regression}
\label{subsec33}
In this part, we give the closed-form solutions of the coefficient $\theta$ for Lasso and ridge regression problems. For Lasso problems and given $x^{s}_{k-\!1}$, we have $F(\theta x^{s}_{k-\!1})\!=\!\frac{1}{2n}\!\sum^{n}_{i=1}\!(\theta a^{T}_{i}\!x^{s}_{k-\!1}\!-\!b_{i})^{2}\!+\!\lambda\|\theta x^{s}_{k-\!1}\|_{1}$. The closed-form solution of \eqref{equ5} for SVRG-SD can be obtained as follows:
\begin{equation}\label{equ10}
\theta_{k}=\mathcal{S}_{\tau}\!\left(\frac{\frac{1}{n}b^{T}\!Ax^{s}_{k-\!1}+\zeta\|\widetilde{p}_{i^{s}_{k}}\|^{2}}{\|Ax^{s}_{k-\!1}\|^{2}/n+\zeta\|\widetilde{p}_{i^{s}_{k}}\|^{2}}\right),
\end{equation}
where $A\!=\![a_{1},\ldots,a_{n}]^{T}\!$ is the data matrix containing $n$ data samples, $b\!=\![b_{1},\ldots,b_{n}]^{T}\!$, and $\mathcal{S}_{\tau}$ is the so-called soft thresholding operator~\cite{donoho:st} with the following threshold,
\begin{equation*}
\tau\!=\!\frac{\lambda\|x^{s}_{k-\!1}\|_{1}}{\|Ax^{s}_{k-\!1}\|^{2}/n+\zeta\|\widetilde{p}_{i^{s}_{k}}\|^{2}}.
\end{equation*}

For ridge regression problems, and $F(\theta x^{s}_{k-\!1})\!=\!\frac{1}{2n}\!\sum^{n}_{i=1}\!(\theta a^{T}_{i}\!x^{s}_{k-\!1}\!-\!b_{i})^{2}\!+\!\frac{\lambda}{2}\|\theta x^{s}_{k-\!1}\|^{2}$, the closed-form solution of \eqref{equ5} for SVRG-SD is given by
\begin{equation}\label{equ11}
\theta_{k}=\frac{\frac{1}{n}b^{T}\!Ax^{s}_{k-\!1}+\zeta\|\widetilde{p}_{i^{s}_{k}}\|^{2}}{\|Ax^{s}_{k-\!1}\|^{2}/n+\zeta\|\widetilde{p}_{i^{s}_{k}}\|^{2}+\lambda\|x^{s}_{k-\!1}\|^{2}}.
\end{equation}

In the same ways as in~\eqref{equ10} and~\eqref{equ11}, we can compute the coefficient $\theta_{k}$ of SAGA-SD for Lasso and ridge regression problems, and revise the update rules in~\eqref{equ10} and~\eqref{equ11} by simply replacing $\widetilde{p}_{i^{s}_{k}}\!=\!\nabla\!f_{i^{s}_{k}}\!(x^{s}_{k-\!1})\!-\!\nabla\! f_{i^{s}_{k}}\!(\widetilde{x}^{s-\!1})$ with $\widetilde{p}_{i^{s}_{k}}\!=\!\nabla\!f_{i^{s}_{k}}\!(x^{s}_{k-\!1})\!-\!\nabla\! f_{i^{s}_{k}}\!(\phi^{k-\!1}_{i^{s}_{k}})$. We can also derive the update rule of the coefficient for other loss functions using their approximations, e.g., \cite{bach:sgd} for logistic regression.

\subsection{Efficient Implementation}
\label{subsec35}
Both \eqref{equ10} and \eqref{equ11} require the calculation of $b^{T}\!A$, thus we need to precompute and save $b^{T}\!A$ in the initial stage. To further reduce the computational complexity of $\|Ax^{s}_{k-\!1}\|^{2}$ in \eqref{equ10} and \eqref{equ11}, we use the fast partial singular value decomposition to obtain the best rank-$r$ approximation $U_{r}S_{r}V^{T}_{r}$ to $A$ and save $S_{r}V^{T}_{r}$. Then $\|Ax^{s}_{k-\!1}\|\!\approx\!\|S_{r}V^{T}_{r}x^{s}_{k-\!1}\|$. In practice, e.g., in our experiments, $r$ can be set to a small number to capture 99.5\% of the spectral energy of the data matrix $A$, e.g., $r\!=\!10$ for the Covtype data set, similar to inexact line search methods~\cite{more:ls} for deterministic optimization.

The time complexity of each inter-iteration in the proposed SVRG-SD and SAGA-SD with full sufficient decrease is $O(rd)$, which is a little higher than SVRG and SAGA. In fact, we can just randomly select only a small fraction (e.g., $1/10^3$) of stochastic gradient iterations in each epoch to update with sufficient decrease, while the remainder of iterations without sufficient decrease, i.e., $\widehat{x}^{s}_{k}\!=\!x^{s}_{k-\!1}$. Let $m_{1}$ be the number of iterations with our sufficient decrease technique in each epoch. By fixing $m_{1}\!=\!\lfloor m/10^3\rfloor$ and thus without increasing parameters tuning difficulties, SVRG-SD and SAGA-SD\footnote{Note that SVRG-SD and SAGA-SD with partial sufficient decrease possess the similar convergence properties as SVRG-SD and SAGA-SD with full sufficient decrease because Property~\ref{prop11} still holds when $\theta_{k}\!=\!1$.} can always converge much faster than their counterparts: SVRG and SAGA, as shown in Figure~\ref{fig1}. It is easy to see that our algorithms are very robust with respect to the choice of $m_{1}$, and achieve average time complexity per-iteration as low as the original SVRG and SAGA. Thus, we mainly consider SVRG-SD and SAGA-SD with partial sufficient decrease.

\begin{figure}[t]
\centering
\includegraphics[width=0.245\columnwidth]{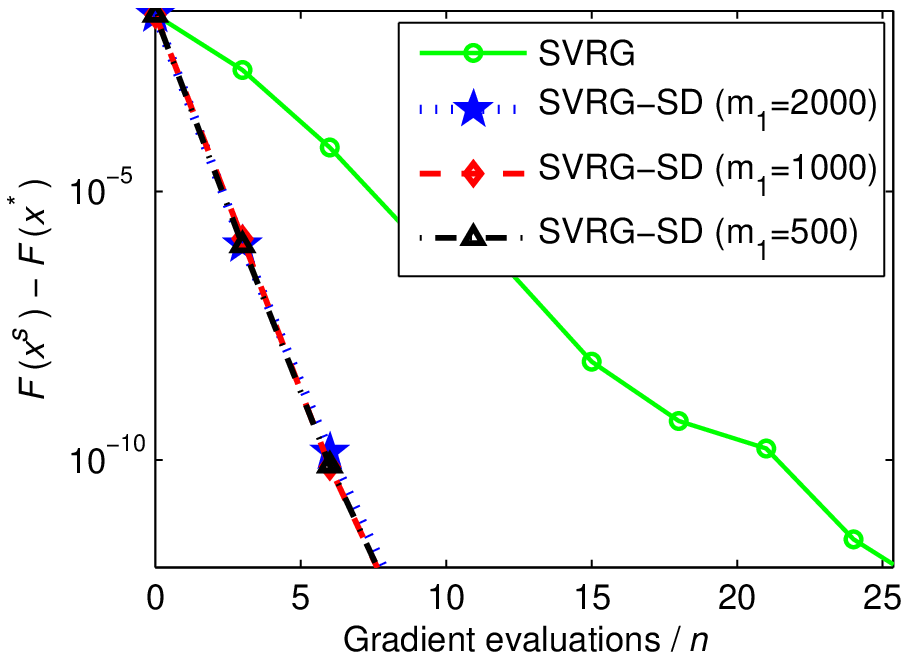}
\includegraphics[width=0.245\columnwidth]{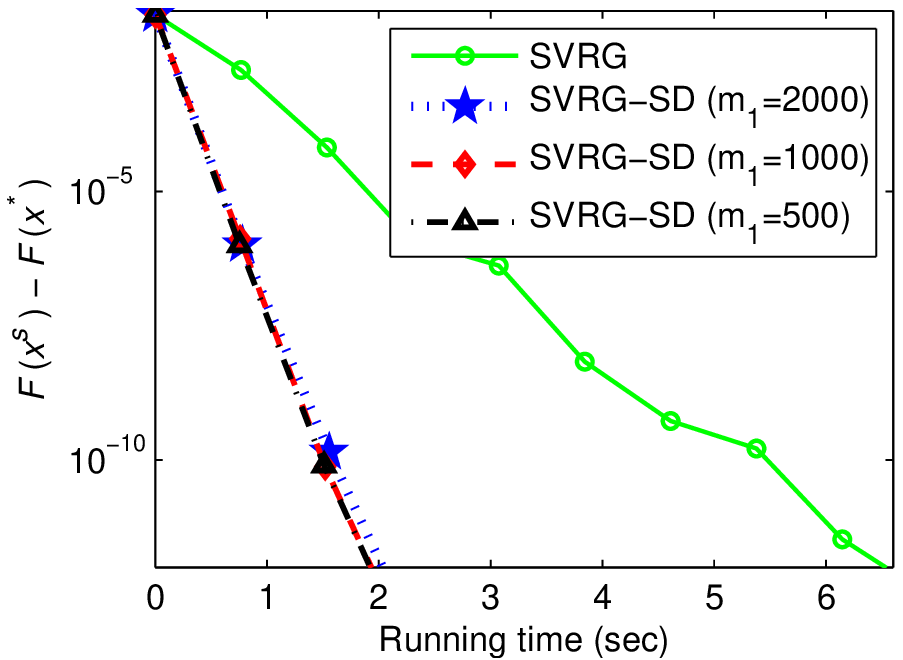}
\includegraphics[width=0.245\columnwidth]{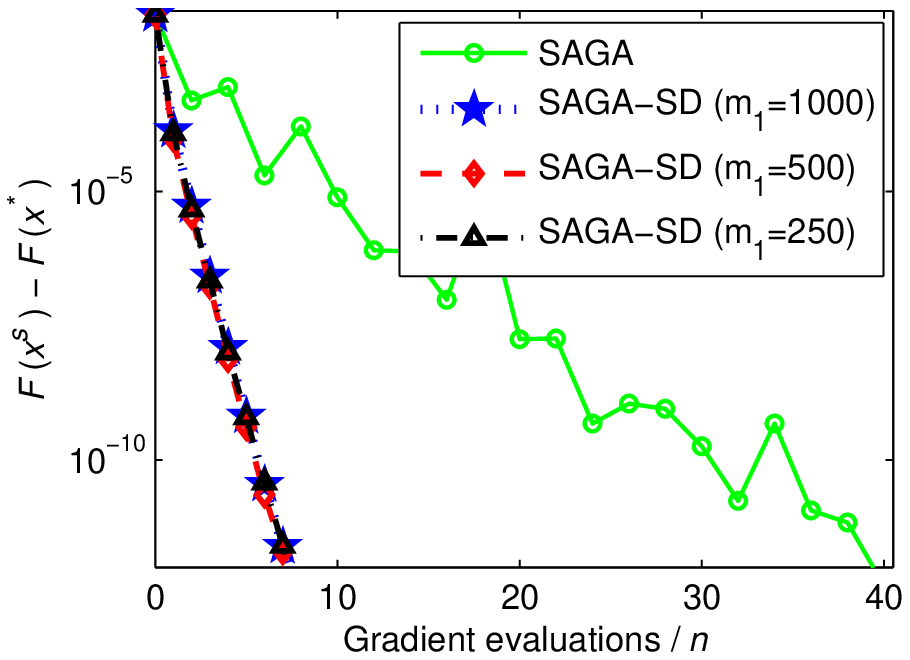}
\includegraphics[width=0.245\columnwidth]{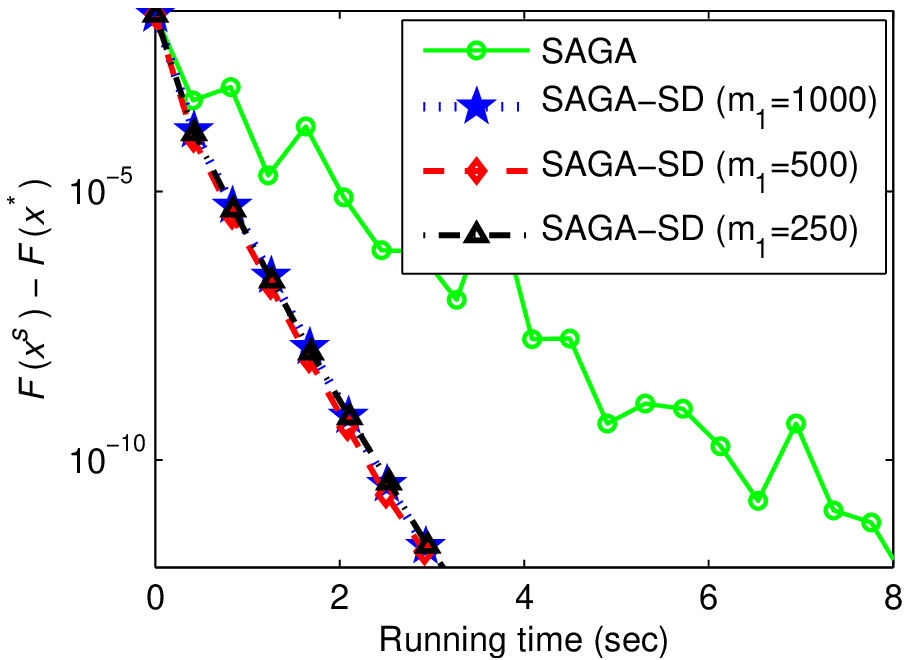}
\caption{Comparison of SVRG-SD and SAGA-SD with different values of $m_{1}$, and their counterparts for ridge regression on the Covtype data set.}
\label{fig1}
\end{figure}

\section{Convergence Guarantees}
In this section, we provide the convergence analysis of SVRG-SD and SAGA-SD for both SC and NSC cases. In this paper, we consider the problem \eqref{equ1} under the following standard assumptions.

\begin{assumption}
\label{assum1}
Each convex function $f_{i}(\cdot)$ is $L$-smooth, iff there exists a constant $L\!>\!0$ such that for any $x,y\!\in\! \mathbb{R}^{d}$, $\|\nabla f_{i}(x)-\nabla f_{i}(y)\|\leq L\|x-y\|$.
\end{assumption}

\begin{assumption}
\label{assum2}
$F(\cdot)$ is $\mu$-strongly convex, iff there exists a constant $\mu\!>\!0$ such that for any $x,y\!\in\!\mathbb{R}^{d}$,
\vspace{-1mm}
\begin{equation}\label{equ15}
F(y)\geq F(x)\!+\!\vartheta^{T}(y\!-\!x)\!+\!\frac{\mu}{2}\|y\!-\!x\|^{2},\;\;\forall\vartheta\in\partial F(x),
\end{equation}
where $\partial F(x)$ is the subdifferential of $F(\cdot)$ at $x$. If $F(\cdot)$ is smooth, we can revise the inequality~\eqref{equ15} by simply replacing the sub-gradient $\vartheta\in\partial F(x)$ with $\nabla\! F(x)$.
\end{assumption}

\subsection{Convergence Analysis of SVRG-SD}
In this part, we analyze the convergence property of SVRG-SD for both SC and NSC cases. The first main result is the following theorem, which provides the convergence rate of SVRG-SD.

\begin{theorem}
\label{theo1}
Suppose Assumption~\ref{assum1} holds. Let $x^{*}$ be the optimal solution of Problem \eqref{equ1}, and $\{(x^{s}_{k},y^{s}_{k},\theta^{s}_{k})\}$ be the sequence produced by SVRG-SD, $\eta\!=\!1/(L\alpha)$, and $\frac{2}{\alpha-1}\!<\!\sigma$, then
\vspace{-1mm}
\begin{equation*}
\begin{split}
\mathbb{E}\!\left[F(\widetilde{x}^{s})\!-\!F(x^{*})\right]\!\leq&\left(\!\frac{\frac{1-\sigma}{m}\!+\!\frac{2}{\alpha-1}}{\sigma\!-\!\frac{2}{\alpha-1}\!+\!\widehat{\beta}}\right)\!\mathbb{E}\!\left[F(\widetilde{x}^{s-\!1})\!-\!F(x^{*})\right]\\
&\quad+\frac{L\alpha\sigma^{2}}{2m\!\left(\sigma\!-\!\frac{2}{\alpha-1}\!+\!\widehat{\beta}\right)}\mathbb{E}\!\left[\|x^{*}\!-\!z^{s}_{0}\|^{2}\!-\!\|x^{*}\!-\!z^{s}_{m}\|^{2}\right]\!,
\end{split}
\end{equation*}
where $z^{s}_{0}\!=\!\left[x^{s}_{0}\!-\!(1\!-\!\sigma)\widehat{x}^{s}_{0}\right]/\sigma$, $z^{s}_{m}\!=\!\left[x^{s}_{m}\!-\!(1\!-\!\sigma)\widehat{x}^{s}_{m-\!1}\right]/\sigma$, $\widehat{\beta}\!=\!\min_{s=1,\ldots,S}\widehat{\beta}^{s}\!\geq\!0$, and $\widehat{\beta}^{s}\!=\!\mathbb{E}[\sum^{m}_{k=1}\!\!\frac{2c_{k}\beta_{k}}{\alpha-1}(F(\widehat{x}^{s}_{k})\!-\!F(x^{*}))]/\mathbb{E}[\sum^{m}_{k=1}\!(F(\widehat{x}^{s}_{k})\!-\!F(x^{*}))]$.
\end{theorem}

The proof of Theorem~\ref{theo1} and the definitions of $c_{k}$ and $\beta_{k}$ are given in the Supplementary Material. The linear convergence of SVRG-SD follows immediately.

\begin{corollary}[SC]
\label{coro1}
Suppose each $f_{i}(\cdot)$ is $L$-smooth, and $F(\cdot)$ is $\mu$-strongly convex. Setting $\alpha\!=\!19$, $\sigma\!=\!1/2$, and $m$ sufficiently large so that
\vspace{-1mm}
\begin{equation*}
\rho=\frac{9}{(7\!+\!18\widehat{\beta})m}+\frac{2}{7\!+\!18\widehat{\beta}}+\frac{171L}{(14\!+\!36\widehat{\beta})m\mu}<1,
\end{equation*}
then SVRG-SD has the geometric convergence in expectation:
\begin{equation*}
\mathbb{E}\!\left[F(\overline{x})-F(x^{*})\right]\leq\rho^{S}\!\left[F(\widetilde{x}^{0})-F(x^{*})\right]\!.
\end{equation*}
\end{corollary}

The proof of Corollary~\ref{coro1} is given in the Supplementary Material. From Corollary~\ref{coro1}, one can see that SVRG-SD has a linear convergence rate for SC problems. As discussed in~\cite{xiao:prox-svrg}, $\rho\!\approx\!\frac{L/\mu}{\nu(1-4\nu)m}\!+\!\frac{4\nu}{1-4\nu}$ for the proximal variant of SVRG~\cite{xiao:prox-svrg}, where $\nu\!=\!1/\alpha$. For a reasonable comparison, we use the same parameter settings for SVRG and SVRG-SD, e.g., $\alpha\!=\!19$ and $m\!=\!57L/\mu$. Then one can see that $\rho_{\textrm{SVRG}}\!\approx\!31/45$ for SVRG and $\rho_{\textrm{SVRG-SD}}\!\approx\!{7}/{(14\!+\!36\widehat{\beta})}\!<\!{1}/2$ for SVRG-SD, that is, $\rho_{\textrm{SVRG-SD}}$ is smaller than $\rho_{\textrm{SVRG}}$. Thus, SVRG-SD can significantly improve the convergence rate of SVRG in practice, which will be confirmed by the experimental results below.

Unlike most of VR-SGD methods~\cite{johnson:svrg,xiao:prox-svrg}, including SVRG, the convergence result of SVRG-SD for the NSC case is also provided, as shown below.

\begin{corollary}[NSC]
\label{coro2}
Suppose each $f_{i}(\cdot)$ is $L$-smooth. Setting $\alpha\!=\!19$, $\sigma\!=\!1/2$, and $m$ sufficiently large, then
\vspace{-1mm}
\begin{equation*}
\mathbb{E}[F(\overline{x})-F(x^{*})]\leq\frac{171L}{(16\!+\!40\widehat{\beta})mS}\|x^{*}\!-\!\widetilde{x}^{0}\|^{2}+\!\left(\frac{9}{(4\!+\!8\widehat{\beta})mS}\!+\!\frac{1}{(2\!+\!4\widehat{\beta})S}\right)\!\left[F(\widetilde{x}^{0})\!-\!F(x^{*})\right]\!.
\end{equation*}
\end{corollary}

The proof of Corollary \ref{coro2} is provided in the Supplementary Material. The constant $\widehat{\beta}\!\geq\!0$ is from the sufficient decrease strategy, which thus implies that the convergence bound in Corollary~\ref{coro2} can be further improved using our sufficient decrease strategy with an even larger $\widehat{\beta}$.

\subsection{Convergence Analysis of SAGA-SD}
In this part, we analyze the convergence property of SAGA-SD for both SC and NSC cases. The following lemma provides the upper bound on the expected variance of the gradient estimator in~\eqref{equ9} (i.e., the SAGA estimator~\cite{defazio:saga}), and its proof is given in the Supplementary Material.

\begin{lemma}
\label{lemm2}
Suppose Assumption~\ref{assum1} holds. Then the following inequality holds
\vspace{-1mm}
\begin{displaymath}
\begin{split}
&\mathbb{E}[\|\nabla\!f_{i^{s}_{k}}\!(x^{s}_{k-\!1})\!-\!\!\nabla\!f(x^{s}_{k-\!1})\!-\!\!\nabla\!f_{i^{s}_{k}}\!(\phi^{k-\!1}_{i^{s}_{k}})\!+\!\frac{1}{n}\!\sum^{n}_{j=1}\!\nabla\! f_{j}(\phi^{k-\!1}_{j})\|^{2}]\\
&\leq 4L[F(x^{s}_{k-\!1})-F(x^{*})+\frac{1}{n}\sum^{n}_{j=1}F_{j}(\phi^{k-\!1}_{j})-F(x^{*})].
\end{split}
\end{displaymath}
\end{lemma}

\begin{theorem}[SC]
\label{theo2}
Suppose $F(\cdot)$ is $\mu$-strongly convex and $f_{i}(\cdot)$ is $L$-smooth. With the same notation as in Theorem~\ref{theo1}, and by setting $\alpha\!=\!19$, $\sigma\!=\!1/2$, and $m$ sufficiently large such that
\begin{equation*}
\rho=\!\frac{n}{(7\!+\!9\widehat{\beta})m}+\frac{9}{(14\!+\!18\widehat{\beta})m}+\frac{171L}{(28\!+\!36\widehat{\beta})\mu m}<1,
\end{equation*}
then SAGA-SD has the geometric convergence in expectation:
\begin{equation*}
\mathbb{E}[F(\overline{x})-F(x^{*})]\leq \rho^{S}\!\left[F(\widetilde{x}^{0})-F(x^{*})\right]\!.
\end{equation*}
\end{theorem}

The proof of Theorem \ref{theo2} is provided in the Supplementary Material. Theorem \ref{theo2} shows that SAGA-SD also attains linear convergence similar to SVRG-SD. Like Corollary \ref{coro2}, we also provide the convergence guarantee of SAGA-SD for NSC problems, as shown below.

\begin{corollary}[NSC]
\label{coro3}
Suppose each $f_{i}(\cdot)$ is $L$-smooth. With the same notation as in Theorem~\ref{theo2} and by setting $\alpha\!=\!19$, $\sigma\!=\!1/2$, and $m\!=\!n$, then
\vspace{-1mm}
\begin{equation*}
\mathbb{E}[F(\overline{x})\!-\!F(x^{*})]\leq\!\frac{171L}{(49\!+\!56\widehat{\beta})nS}\|x^{*}\!-\!\widetilde{x}^{0}\|^{2}\!+\!\!\left(\frac{81}{(98\!+\!126\widehat{\beta})nS}\!+\!\frac{9}{(49\!+\!63\widehat{\beta})S}\right)\!\left[F(\widetilde{x}^{0})\!-\!F(x^{*})\right]\!.
\end{equation*}
\end{corollary}

The proof of Corollary \ref{coro3} is provided in the Supplementary Material. Due to $\widehat{\beta}\!\geq\!0$, Theorem \ref{theo2} and Corollary \ref{coro3} imply that SAGA-SD can significantly improve the convergence rate of SAGA~\cite{defazio:saga} for both SC and NSC cases, which will be confirmed by our experimental results.

As suggested in~\cite{frostig:sgd} and \cite{lin:vrsg}, one can add a proximal term into a non-strongly convex objective function $F(x)$ as follows: $F_{\tau}(x,y)\!=\!f(x)\!+\!\frac{\tau}{2}\|x\!-\!y\|^{2}\!+\!r(x)$, where $\tau\!\geq\!0$ is a constant that can be determined as in~\cite{frostig:sgd,lin:vrsg}, and $y\!\in\! \mathbb{R}^{d}$ is a proximal point. Then the condition number of this proximal function $F_{\tau}(x,y)$ can be much smaller than that of the original function $F(x)$, if $\tau$ is sufficiently large. However, adding the proximal term may degrade the performance of the involved algorithms both in theory and in practice~\cite{zhu:univr}. Therefore, we directly use SVRG-SD and SAGA-SD to solve non-strongly convex objectives.

\begin{figure}[t]
\centering
\includegraphics[width=0.326\columnwidth]{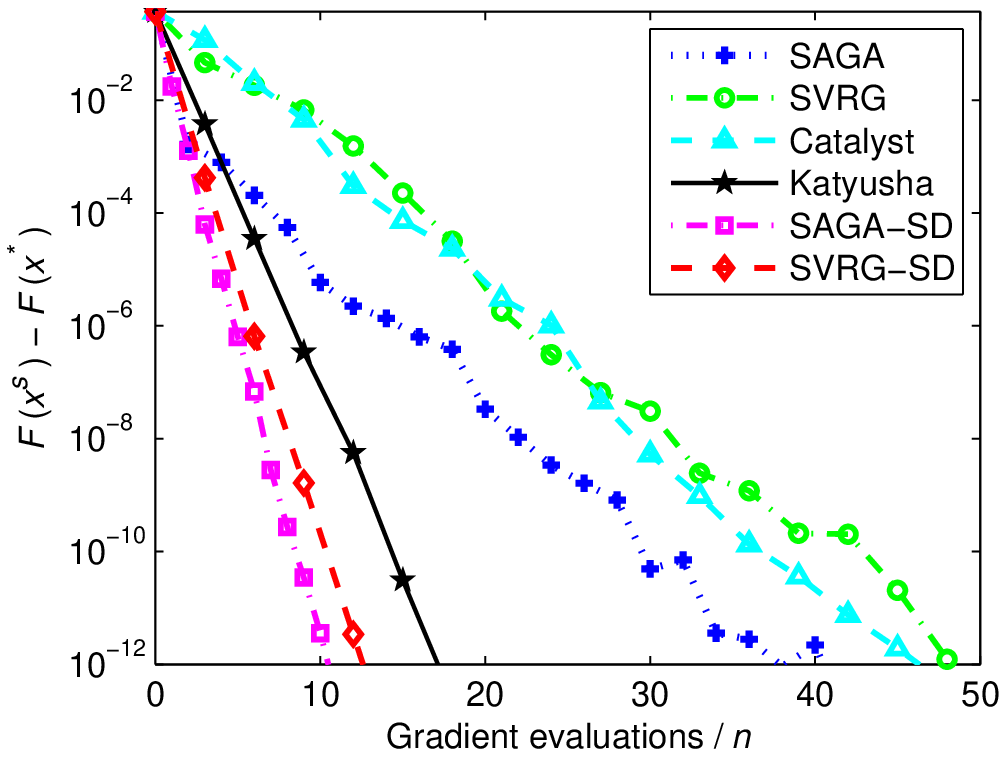}\,
\includegraphics[width=0.326\columnwidth]{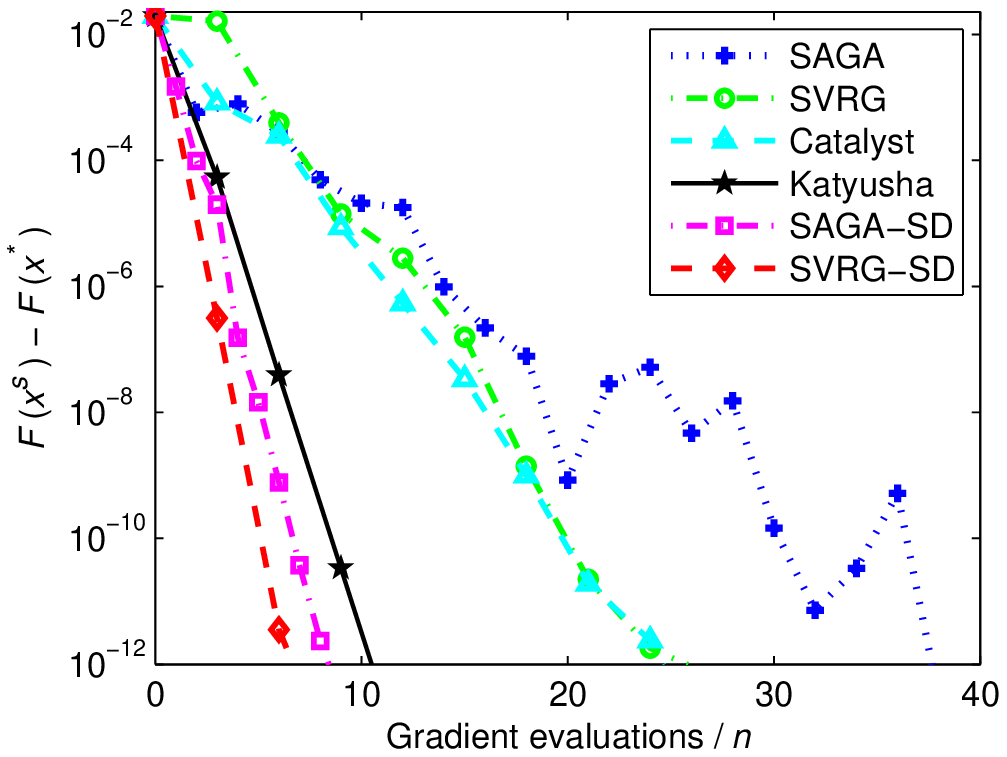}\,
\includegraphics[width=0.326\columnwidth]{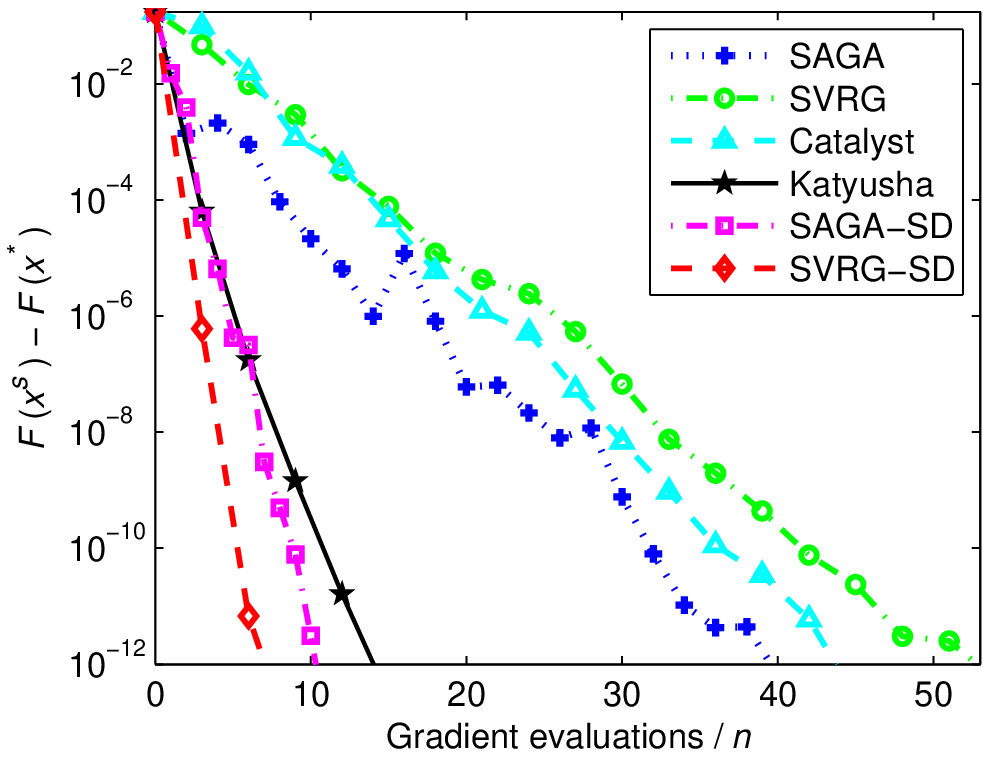}
\caption{Comparison of different VR-SGD methods for solving ridge regression problems ($\lambda\!=\!10^{-4}$) on Ijcnn1 (left), Covtype (center), and SUSY (right). The vertical axis is the objective value minus the minimum, and the horizontal axis denotes the number of effective passes over the data.}
\label{fig2}
\end{figure}

\section{Experimental Results}
In this section, we evaluate the performance of SVRG-SD and SAGA-SD, and compare their performance with their counterparts including SVRG~\cite{johnson:svrg}, its proximal variant (Prox-SVRG)~\cite{xiao:prox-svrg}, and SAGA~\cite{defazio:saga}. Moreover, we also report the performance of the well-known accelerated VR-SGD methods, Catalyst~\cite{lin:vrsg} and Katyusha~\cite{zhu:Katyusha}. For fair comparison, we implemented all the methods in C++ with a Matlab interface (all codes are made available, see link in the Supplementary Materials), and performed all the experiments on a PC with an Intel i5-2400 CPU and 16GB RAM.

\subsection{Ridge Regression}
Our experiments were conducted on three popular data sets: Covtype, Ijcnn1 and SUSY, all of which were obtained from the LIBSVM Data website{\footnote{\url{https://www.csie.ntu.edu.tw/~cjlin/libsvm/}}} (more details and regularization parameters are given in the Supplementary Material). Following~\cite{xiao:prox-svrg}, each feature vector of these date sets has been normalized so that $\|a_{i}\|\!=\!1$ for all $i=1,\ldots,n$, which leads to the same upper bound on the Lipschitz constants $L_{i}$. This step is for comparison only and not necessary in practice. We focus on the ridge regression as the SC example. For SVRG-SD and SAGA-SD, we set $\sigma\!=\!1/2$ on the three data sets. In addition, unlike SAGA~\cite{defazio:saga}, we fixed $m\!=\!n$ for each epoch of SAGA-SD. For SVRG-SD, Catalyst, Katyusha, SVRG and its proximal variant, we set the epoch size $m\!=\!2n$, as suggested in~\cite{zhu:Katyusha,johnson:svrg,xiao:prox-svrg}. Each of these methods had its step size parameter chosen so as to give the fastest convergence.

Figure~\ref{fig2} shows how the objective gap, i.e., $F(x^{s})\!-\!F(x^{*})$, of all these algorithms decreases for ridge regression problems with the regularization parameter $\lambda\!=\!10^{-4}$ (more results are given in the Supplementary Material). Note that the horizontal axis denotes the number of effective passes over the data. As seen in these figures, SVRG-SD and SAGA-SD achieve consistent speedups for all the data sets, and significantly outperform their counterparts, SVRG and SAGA, in all the settings. This confirms that our sufficient decrease technique is able to accelerate SVRG and SAGA. Impressively, SVRG-SD and SAGA-SD usually converge much faster than the well-known accelerated VR-SGD methods, Catalyst and Katyusha, which further justifies the effectiveness of our sufficient decrease stochastic optimization method.

\begin{figure}[t]
\centering
\subfigure[Lasso, $\lambda\!=\!10^{-4}$]{\includegraphics[width=0.326\columnwidth]{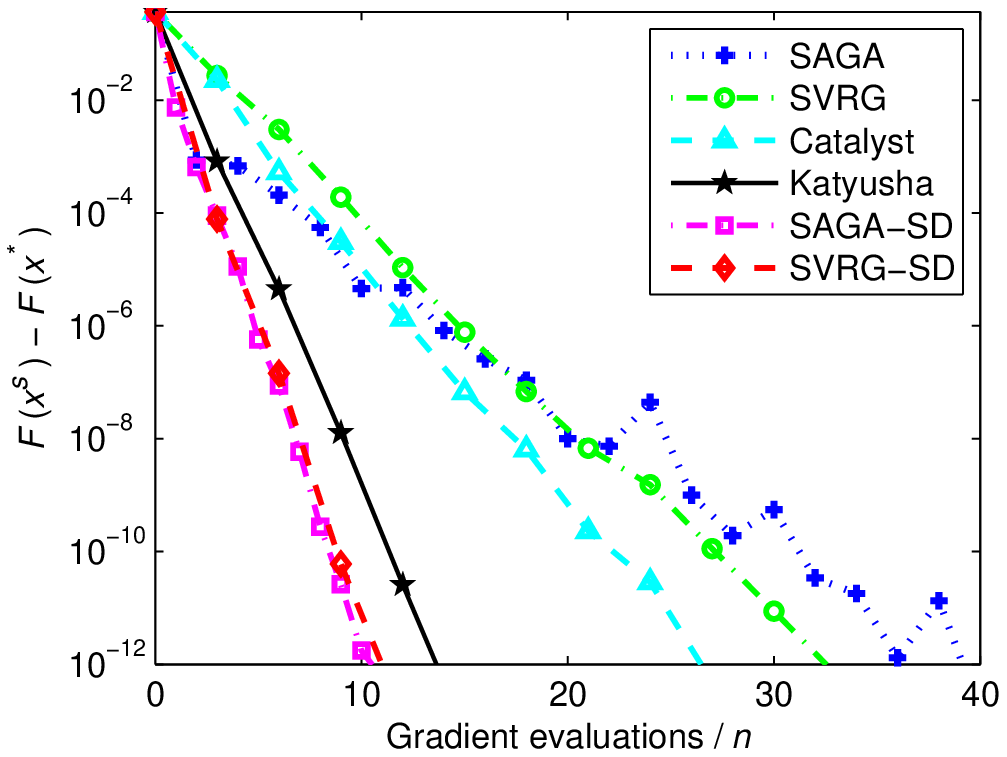}}\,
\subfigure[Elastic-net regularized Lasso, $\lambda_{1}\!=\!\lambda_{2}\!=\!10^{-5}$]{\includegraphics[width=0.326\columnwidth]{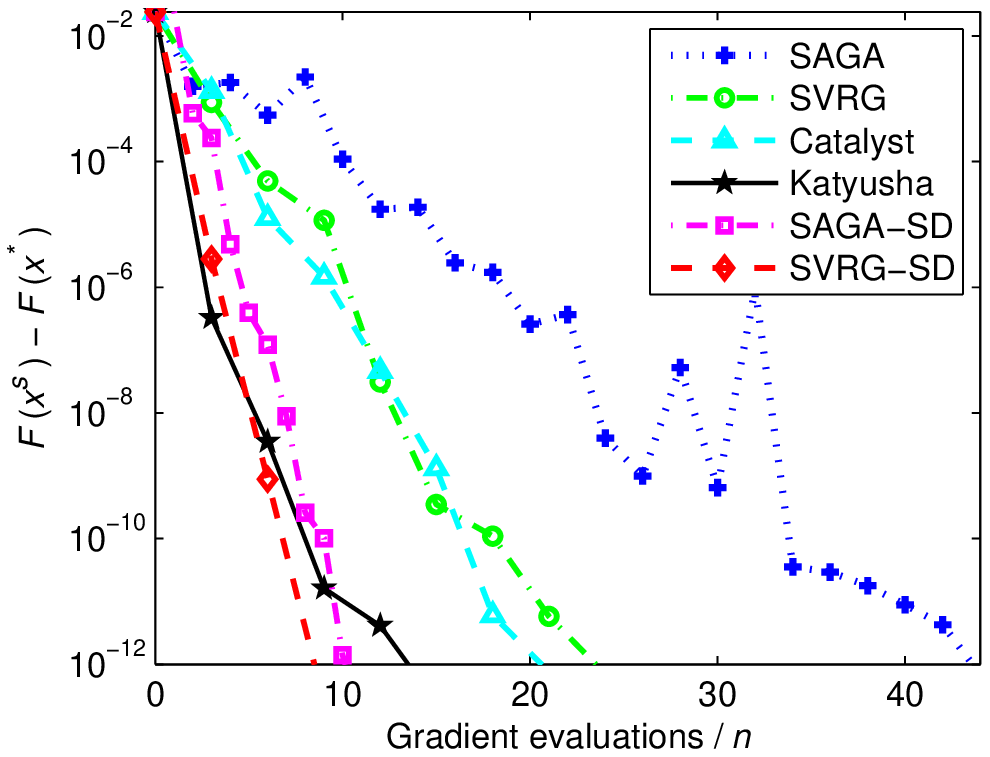}\,\includegraphics[width=0.326\columnwidth]{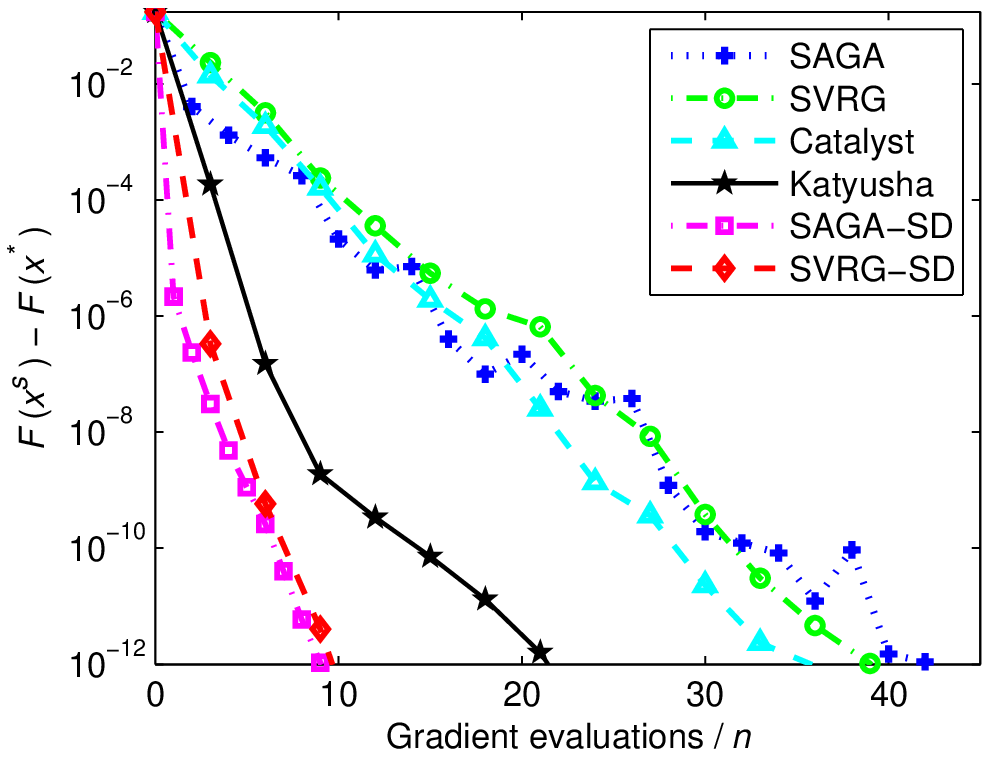}}
\vspace{-1.6mm}

\caption{Comparison of different VR-SGD methods for solving Lasso and elastic-net regularized Lasso problems on the three data sets: Ijcnn1 (left), Covtype (center), and SUSY (right).}
\label{fig3}
\end{figure}

\subsection{Lasso and Elastic-Net Regularized Lasso}
We also conducted experiments of the Lasso and elastic-net regularized (i.e., $\lambda_{1}\|\!\cdot\!\|_{1}\!+\!\lambda_{2}\|\!\cdot\!\|^2$) Lasso problems. We plot some representative results in Figure~\ref{fig3} (see Figures 3 and 4 in the Supplementary Material for more results), which show that SVRG-SD and SAGA-SD significantly outperform their counterparts (i.e., Prox-SVRG and SAGA) in all the settings, as well as Catalyst, and are considerably better than Katyusha in most cases. This empirically verifies that our sufficient decrease technique can accelerate SVRG and SAGA for solving both SC and NSC objectives.

\section{Conclusion \& Future Work}
To the best of our knowledge, this is the first work to design an efficient sufficient decrease technique for stochastic optimization. Moreover, we proposed two different schemes for Lasso and ridge regression to efficiently update the coefficient $\theta$, which takes the important decisions to shrink, expand or move in the opposite direction. This is very different from adaptive learning rate methods, e.g., \cite{kingma:sgd}, and line search methods, e.g., \cite{mahsereci:sgd}, all of which cannot address the issue in Section~\ref{sec31} whatever value the step size is. Unlike most VR-SGD methods~\cite{johnson:svrg,shalev-Shwartz:sdca,xiao:prox-svrg}, which only have convergence guarantees for SC problems, we provided the convergence guarantees of our algorithms for both SC and NSC cases. Experimental results verified the effectiveness of our sufficient decrease technique for stochastic optimization. Naturally, it can also be used to further speed up accelerated VR-SGD methods such as~\cite{zhu:Katyusha,zhu:univr,lin:vrsg}.

As each function $f_{i}(\cdot)$ can have different degrees of smoothness, to select the random index $i^{s}_{k}$ from a non-uniform distribution is a much better choice than simple uniform random sampling~\cite{zhao:prox-smd}, as well as without-replacement sampling vs.\ with-replacement sampling~\cite{shamir:sgd}. On the practical side, both our algorithms tackle the NSC and non-smooth problems directly, without using any quadratic regularizer as in~\cite{zhu:Katyusha,lin:vrsg}, as well as proximal settings. Note that some asynchronous parallel and distributed variants~\cite{lee:dsgd,reddi:sgd} of VR-SGD methods have also been proposed for such stochastic settings. We leave these variations out from our comparison and consider similar extensions to our stochastic sufficient decrease method as future work.


\newpage

\vspace{3mm}

\begin{center}
{\Large \textbf{Supplementary Materials for ``Guaranteed Sufficient Decrease for Variance Reduced Stochastic Gradient Descent''}}
\end{center}

In this supplementary material, we give the detailed proofs for some lemmas, theorems and corollaries stated in the main paper. Moreover, we also report more experimental results for both of our algorithms.

\section*{Notations}
Throughout this paper, $\|\!\cdot\!\|$ denotes the standard Euclidean norm, and $\|\!\cdot\!\|_{1}$ is the $\ell_{1}$-norm, i.e., $\|x\|_{1}\!=\!\sum^{d}_{i=1}\!|x_{i}|$. We denote by $\nabla\!f(x)$ the full gradient of $f(x)$ if it is differentiable, or $\partial f(x)$ the subdifferential of $f(\cdot)$ at $x$ if it is only Lipschitz continuous. Note that Assumption 2 is the general form for the two cases when $F(x)$ is smooth or non-smooth\footnote{Strictly speaking, when the function $F(\cdot)$ is non-smooth, $\vartheta\in \partial F(x)$; while $F(\cdot)$ is smooth, $\vartheta=\nabla F(x)$.}. That is, if $F(x)$ is smooth, the inequality in (12) in Assumption 2 becomes the following form:
\begin{displaymath}
F(y)\geq F(x)+\nabla F(x)(y-x)+\frac{\mu}{2}\|y-x\|^{2}.
\end{displaymath}

\section*{Appendix A: Proof of Theorem 1}
Although the proposed SVRG-SD is a variant of SVRG, it is non-trivial to analyze its convergence property, as well as that of SAGA-SD. Before proving Theorem 1, we first give the following lemma.

\begin{lemma}
\label{lemm1}
Let $x^{*}$ be the optimal solution of Problem (1), then the following inequality holds
\begin{equation*}
\begin{split}
&\mathbb{E}\!\left[\left\|\nabla\! f_{i^{s}_{k}}\!(x^{s}_{k-\!1})\!-\!\nabla\! f(x^{s}_{k-\!1})\!-\!\nabla\! f_{i^{s}_{k}}\!(\widetilde{x}^{s-\!1})\!+\!\nabla\! f(\widetilde{x}^{s-\!1})\right\|^{2}\right]\\
\leq&\, 4L\!\left[F(x^{s}_{k-\!1})-F(x^{*})+F(\widetilde{x}^{s-\!1})-F(x^{*})\right]\!.
\end{split}
\end{equation*}
\end{lemma}

Lemma~\ref{lemm1} provides the upper bound on the expected variance of the variance reduced gradient estimator in (9) (i.e., the SVRG estimator independently introduced in~\cite{johnson:svrg,zhang:svrg}), which satisfies $\mathbb{E}[\widetilde{\nabla}\! f_{i^{s}_{k}}(x^{s}_{k-\!1})]\!=\!\nabla\!f(x^{s}_{k-\!1})$. This lemma is essentially identical to Corollary 3.5 in~\cite{xiao:prox-svrg}. From Lemma~\ref{lemm1}, we immediately get the following result, which is useful in our convergence analysis.

\begin{corollary}\label{coro4}
For any $\alpha\geq\beta>0$, the following inequality holds
\begin{equation*}
\begin{split}
&\alpha\mathbb{E}\!\left[\left\|\nabla\! f_{i^{s}_{k}}\!(x^{s}_{k-\!1})-\!\nabla\! f(x^{s}_{k-\!1})\!-\!\nabla\! f_{i^{s}_{k}}\!(\widetilde{x}^{s-\!1})\!+\!\nabla\! f(\widetilde{x}^{s-\!1})\right\|^{2}\right]\!-\!\beta\mathbb{E}\!\left[\left\|\nabla\! f_{i^{s}_{k}}\!(x^{s}_{k-\!1})\!-\!\nabla\! f_{i^{s}_{k}}\!(\widetilde{x}^{s-\!1})\right\|^{2}\right]\\
\leq\, &4L(\alpha\!-\!\beta)\!\left[F(x^{s}_{k-1})-F(x^{*})+F(\widetilde{x}^{s-1})-F(x^{*})\right]\!.
\end{split}
\end{equation*}
\end{corollary}

\begin{proof}
\begin{equation*}
\begin{split}
&\alpha\mathbb{E}\!\left[\left\|\nabla\! f_{i^{s}_{k}}\!(x^{s}_{k-\!1})\!-\!\nabla\! f(x^{s}_{k-\!1})\!-\!\nabla\! f_{i^{s}_{k}}\!(\widetilde{x}^{s-\!1})\!+\!\nabla\! f(\widetilde{x}^{s-\!1})\right\|^{2}\right]-\beta\mathbb{E}\!\left[\left\|\nabla\! f_{i^{s}_{k}}\!(x^{s}_{k-\!1})\!-\!\nabla\! f_{i^{s}_{k}}\!(\widetilde{x}^{s-\!1})\right\|^{2}\right]\\
=\,&\alpha\mathbb{E}\!\!\left[\left\|[\nabla\! f_{i^{s}_{k}}\!(x^{s}_{k-\!1})\!-\!\nabla\! f_{i^{s}_{k}}\!(\widetilde{x}^{s-\!1})]\!-\![\nabla\! f(x_{k-1})\!-\!\nabla\! f(\widetilde{x}^{s-\!1})]\right\|^{2}\right]-\beta\mathbb{E}\!\!\left[\left\|\nabla\! f_{i^{s}_{k}}\!(x^{s}_{k-\!1})\!-\!\nabla\! f_{i^{s}_{k}}\!(\widetilde{x}^{s-\!1})\right\|^{2}\right]\\
=\,&\alpha\mathbb{E}\!\!\left[\left\|\nabla\! f_{i^{s}_{k}}\!(x^{s}_{k-\!1})\!-\!\nabla\! f_{i^{s}_{k}}\!(\widetilde{x}^{s-\!1})\right\|^{2}\right]\!-\!\alpha\!\left\|\nabla\! f(x^{s}_{k-\!1})\!-\!\nabla\! f(\widetilde{x}^{s-\!1})\right\|^{2}\!-\!\beta\mathbb{E}\!\!\left[\left\|\nabla\! f_{i^{s}_{k}}\!(x^{s}_{k-\!1})\!-\!\nabla\! f_{i^{s}_{k}}\!(\widetilde{x}^{s-\!1})\right\|^{2}\right]\\
\leq\,&\alpha\mathbb{E}\!\left[\left\|\nabla\! f_{i^{s}_{k}}\!(x^{s}_{k-1})-\nabla\! f_{i^{s}_{k}}\!(\widetilde{x}^{s-1})\right\|^{2}\right]\!-\!\beta\mathbb{E}\!\left[\left\|\nabla\! f_{i^{s}_{k}}\!(x^{s}_{k-\!1})\!-\!\nabla\! f_{i^{s}_{k}}\!(\widetilde{x}^{s-\!1})\right\|^{2}\right]\\
=\,&(\alpha\!-\!\beta)\mathbb{E}\!\left[\left\|\left[\nabla\! f_{i^{s}_{k}}\!(x^{s}_{k-1})-\nabla\! f_{i^{s}_{k}}\!(x^{*})\right]-[\nabla\! f_{i^{s}_{k}}\!(\widetilde{x}^{s-1})-\nabla\! f_{i^{s}_{k}}\!(x^{*})]\right\|^{2}\right]\\
\leq\,&2(\alpha\!-\!\beta)\left\{\mathbb{E}\!\left[\left\|\nabla\! f_{i^{s}_{k}}\!(x^{s}_{k-\!1})-\nabla\! f_{i^{s}_{k}}\!(x^{*})\right\|^{2}\right]+\mathbb{E}\!\left[\left\|\nabla\! f_{i^{s}_{k}}\!(\widetilde{x}^{s-\!1})-\nabla\! f_{i^{s}_{k}}\!(x^{*})\right\|^{2}\right]\right\}\\
\leq\, & 4L(\alpha\!-\!\beta)\!\left[F(x^{s}_{k-1})-F(x^{*})+F(\widetilde{x}^{s-1})-F(x^{*})\right]\!,
\end{split}
\end{equation*}
where the second equality holds due to the fact that $\mathbb{E}[\|x\!-\!\mathbb{E}x\|^{2}]\!=\!\mathbb{E}[\|x\|^{2}]\!-\!\|\mathbb{E}x\|^{2}$; the second inequality holds due to the fact that $\|a-b\|^{2}\leq2(\|a\|^{2}+\|b\|^{2})$; and the last inequality follows from Lemma 3.4 in~\cite{xiao:prox-svrg} (i.e., $\mathbb{E}[\left\|\nabla\! f_{i}(x)\!-\!\nabla\! f_{i}(x^{*})\right\|^{2}]\!\leq\! 2L\!\left[F(x)\!-\!F(x^{*})\right]$).
\end{proof}

Moreover, we also introduce the following lemmas~\cite{baldassarre:prox,lan:sgd}, which are useful in our convergence analysis.

\begin{lemma}\label{prop1}
Let $\widetilde{F}(x,y)$ be the linear approximation of $F(\cdot)$ at $y$ with respect to $f$, i.e.,
\begin{displaymath}
\widetilde{F}(x,y)=f(y)+\left\langle \nabla f(y),\, x-y\right\rangle+ r(x).
\end{displaymath}
Then
\begin{displaymath}
F(x)\leq \widetilde{F}(x,y)+\frac{L}{2}\|x-y\|^{2}\leq F(x)+\frac{L}{2}\|x-y\|^{2}.
\end{displaymath}
\end{lemma}

\begin{lemma}\label{prop2}
Assume that $\hat{x}$ is an optimal solution of the following problem,
\begin{displaymath}
\min_{x\in\mathbb{R}^{d}}\frac{\tau}{2}\|x-y\|^{2}+g(x),
\end{displaymath}
where $g(x)$ is a convex function (but possibly non-differentiable). Then the following inequality holds for all $x\!\in\!\mathbb{R}^{d}$:
\begin{displaymath}
g(\hat{x})+\frac{\tau}{2}\|\hat{x}-y\|^{2}+\frac{\tau}{2}\|x-\hat{x}\|^{2}\leq g(x)+\frac{\tau}{2}\|x-y\|^{2}.
\end{displaymath}
\end{lemma}
\vspace{1mm}

\textbf{Proof of Theorem 1:}
\begin{proof}
Let $\eta=\frac{1}{L\alpha}$ and $p_{i^{s}_{k}}\!=\!\widetilde{\nabla}\! f_{i^{s}_{k}}\!(x^{s}_{k-\!1})\!=\!\nabla\! f_{i^{s}_{k}}\!(x^{s}_{k-1})-\nabla\! f_{i^{s}_{k}}\!(\widetilde{x}^{s-1})+\nabla\! f(\widetilde{x}^{s-1})$. Using Lemma~\ref{prop1}, we have
\begin{equation}\label{equ71}
\begin{split}
F(y^{s}_{k})\leq\,& f(x^{s}_{k-1})+\left\langle\nabla\! f(x^{s}_{k-\!1}),\,y^{s}_{k}\!-\!x^{s}_{k-\!1}\right\rangle+\frac{L\alpha}{2}\!\left\|y^{s}_{k}\!-\!x^{s}_{k-\!1}\right\|^{2}\!-\!\frac{L(\alpha\!-\!1)}{2}\!\left\|y^{s}_{k}\!-\!x^{s}_{k-\!1}\right\|^{2}+r(y^{s}_{k})\\
=\,& f_{i^{s}_{k}}\!(x^{s}_{k-1})+\left\langle p_{i^{s}_{k}},\,y^{s}_{k}-x^{s}_{k-1}\right\rangle+r(y^{s}_{k})+\frac{L\alpha}{2}\!\|y^{s}_{k}-x^{s}_{k-1}\|^2\\
&+\left\langle\nabla\! f(x^{s}_{k-1})-p_{i^{s}_{k}},\,y^{s}_{k}-x^{s}_{k-1}\right\rangle-\frac{L(\alpha\!-\!1)}{2}\|y^{s}_{k}-x^{s}_{k-1}\|^{2}+f(x^{s}_{k-1})-f_{i^{s}_{k}}(x^{s}_{k-1}).
\end{split}
\end{equation}
Then
\begin{equation}\label{equ72}
\begin{split}
&\left\langle\nabla\! f(x^{s}_{k-1})-p_{i^{s}_{k}},\,y^{s}_{k}-x^{s}_{k-1}\right\rangle-\frac{L(\alpha\!-\!1)}{2}\|y^{s}_{k}-x^{s}_{k-1}\|^{2}\\
\leq\,& \frac{1}{2L(\alpha\!-\!1)}\|\nabla\!f(x^{s}_{k-1})-p_{i^{s}_{k}}\|^{2}+\frac{L(\alpha\!-\!1)}{2}\|y^{s}_{k}\!-\!x^{s}_{k-1}\|^{2}-\frac{L(\alpha\!-\!1)}{2}\|y^{s}_{k}\!-\!x^{s}_{k-1}\|^{2}\\
=\,&\frac{1}{2L(\alpha\!-\!1)}\|\nabla\!f(x^{s}_{k-1})-p_{i^{s}_{k}}\|^{2},
\end{split}
\end{equation}
where the inequality follows from the Young's inequality, i.e., $a^{T}b\leq\|a\|^{2}/(2\rho)+\rho\|b\|^{2}/2$ for any $\rho\!>\!0$. Substituting the inequality \eqref{equ72} into the inequality \eqref{equ71}, we have
\begin{equation}\label{equ73}
\begin{split}
F(y^{s}_{k})&\leq f_{i^{s}_{k}}\!(x^{s}_{k-1})+\left\langle p_{i^{s}_{k}},\,y^{s}_{k}-x^{s}_{k-1}\right\rangle+r(y^{s}_{k})+\frac{L\alpha}{2}\|y^{s}_{k}-x^{s}_{k-1}\|^2\\
&\quad+\frac{1}{2L(\alpha\!-\!1)}\|\nabla\!f(x^{s}_{k-1})-p_{i^{s}_{k}}\|^{2}+f(x^{s}_{k-1})-f_{i^{s}_{k}}(x^{s}_{k-1})\\
&\leq f_{i^{s}_{k}}\!(x^{s}_{k-\!1})+r(\widehat{w}^{s}_{k-\!1})+\frac{L\alpha}{2}\!\left(\|\widehat{w}^{s}_{k-\!1}\!-\!x^{s}_{k-1}\|^{2}\!-\!\|\widehat{w}^{s}_{k-\!1}\!-\!y^{s}_{k}\|^{2}\right)+\langle p_{i^{s}_{k}},\,\widehat{w}^{s}_{k-\!1}\!-\!x^{s}_{k-\!1}\rangle\\
&\quad+\frac{1}{2L(\alpha\!-\!1)}\|\nabla\!f(x^{s}_{k-1})-p_{i^{s}_{k}}\|^{2}+f(x^{s}_{k-\!1})-f_{i^{s}_{k}}(x^{s}_{k-\!1})\\
&\leq F_{i^{s}_{k}}\!(\widehat{w}^{s}_{k-1})+\frac{L\alpha}{2}\left(\|\widehat{w}^{s}_{k-1}-x^{s}_{k-1}\|^{2}-\|\widehat{w}^{s}_{k-1}-y^{s}_{k}\|^{2}\right)+f(x^{s}_{k-1})-f_{i^{s}_{k}}(x^{s}_{k-1})\\
&\quad+\frac{1}{2L(\alpha\!-\!1)}\|\nabla\!f(x^{s}_{k-1})-p_{i^{s}_{k}}\|^{2}+\left\langle-\nabla f_{i^{s}_{k}}(\widetilde{x}^{s-1})+\nabla f(\widetilde{x}^{s-1}),\,\widehat{w}^{s}_{k-1}-x^{s}_{k-1}\right\rangle\\
&\leq \sigma F_{i^{s}_{k}}\!(x^{*})+(1-\sigma)F_{i^{s}_{k}}\!(\widehat{x}^{s}_{k-1})+\frac{L\alpha\sigma^{2}}{2}\left(\|x^{*}-z^{s}_{k-1}\|^{2}-\|x^{*}-z^{s}_{k}\|^{2}\right)\\
&\quad+\frac{1}{2L(\alpha\!-\!1)}\|\nabla\!f(x^{s}_{k-1})-p_{i^{s}_{k}}\|^{2}+f(x^{s}_{k-1})-f_{i^{s}_{k}}(x^{s}_{k-1})\\
&\quad+\!\left\langle\nabla f(\widetilde{x}^{s-1})\!-\!\nabla f_{i^{s}_{k}}(\widetilde{x}^{s-1}),\,\widehat{w}^{s}_{k-1}\!-\!x^{s}_{k-1}\right\rangle,
\end{split}
\end{equation}
where $\widehat{w}^{s}_{k-1}\!=\!\sigma x^{*}+(1\!-\!\sigma)\widehat{x}^{s}_{k-1}$, and $\widehat{x}_{k-\!1}\!=\!\theta_{k-\!1}x_{k-2}$. The second inequality follows from Lemma \ref{prop2} with $g(x)\!:=\!\left\langle p_{i^{s}_{k}},\,x\!-\!x^{s}_{k-1}\right\rangle\!+\!r(x)$, $\tau\!=\!L\alpha$, $\hat{x}\!=\!y^{s}_{k}$, $x\!=\!\widehat{w}^{s}_{k-1}$ and $y\!=\!x^{s}_{k-1}$; the third inequality holds due to the convexity of the component function $f_{i^{s}_{k}}(x)$ (i.e., $f_{i^{s}_{k}}\!(x^{s}_{k-\!1})\!+\!\langle\nabla\! f_{i^{s}_{k}}\!(x^{s}_{k-\!1}),\widehat{w}^{s}_{k-\!1}\!-\!x^{s}_{k-\!1}\rangle\!\leq\! f_{i^{s}_{k}}\!(\widehat{w}^{s}_{k-\!1})$); and the last inequality holds due to the convexity of the function $F_{i^{s}_{k}}\!(x)\!:=\!f_{i^{s}_{k}}\!(x)\!+\!r(x)$, and
\begin{displaymath}
z^{s}_{k-1}=[x^{s}_{k-1}-(1\!-\!\sigma)\widehat{x}^{s}_{k-1}]/{\sigma},\;\,z^{s}_{k}=[y^{s}_{k}-(1\!-\!\sigma)\widehat{x}^{s}_{k-1}]/{\sigma},
\end{displaymath}
which mean that $\widehat{w}^{s}_{k-\!1}\!-x^{s}_{k-\!1}=\sigma(x^{*}-z^{s}_{k-\!1})$ and $\widehat{w}^{s}_{k-\!1}\!-y^{s}_{k}=\sigma(x^{*}-z^{s}_{k})$.

Using Property 1 with $\zeta=\frac{\delta\eta}{1-L\eta}$ and $\eta=1/L\alpha,$\footnote{Note that our fast versions of SVRG-SD and SAGA-SD (i.e., SVRG-SD and SAGA-SD with randomly partial sufficient decrease) have the similar convergence properties as SVRG-SD and SAGA-SD because Property 1 still holds when $\theta_{k}\!=\!1$. That is, the main difference between their convergence properties is the different values of $\beta_{k}$, as shown below.} we obtain
\begin{equation}\label{equ74}
\begin{split}
F(\theta_{k}{x}^{s}_{k-1})=F(\widehat{x}^{s}_{k})&\leq F(x^{s}_{k-1})-\frac{(\theta_{k}\!-\!1)^{2}}{2L(\alpha-1)}\|\nabla\!f_{i^{s}_{k}}\!(x^{s}_{k-\!1})-\!\nabla\! f_{i^{s}_{k}}\!(\widetilde{x}^{s-1})\|^{2}\\
&\leq F(x^{s}_{k-1})-\frac{\beta_{k}}{2L(\alpha\!-\!1)}\|\nabla\!f_{i^{s}_{k}}\!(x^{s}_{k-\!1})-\!\nabla\! f_{i^{s}_{k}}\!(\widetilde{x}^{s-1})\|^{2},
\end{split}
\end{equation}
where $\beta_{k}=\min\!\left[1/\alpha_{k},\,(\theta_{k}\!-\!1)^{2}\right]$, and $\alpha_{k}$ is defined below. Then there exists $\overline{\beta}_{k}$ such that
\begin{equation}\label{equ83}
\mathbb{E}\!\left[\frac{\beta_{k}}{2L(\alpha\!-\!1)}\|\nabla\!f_{i^{s}_{k}}\!(x^{s}_{k-\!1})-\!\nabla\! f_{i^{s}_{k}}\!(\widetilde{x}^{s-1})\|^{2}\right]=\frac{\overline{\beta}_{k}}{2L(\alpha\!-\!1)}\mathbb{E}\!\left[\|\nabla\!f_{i^{s}_{k}}\!(x^{s}_{k-\!1})-\!\nabla\! f_{i^{s}_{k}}\!(\widetilde{x}^{s-1})\|^{2}\right],
\end{equation}
where $\overline{\beta}_{k}=\mathbb{E}[\beta_{k}\|\nabla\!f_{i^{s}_{k}}\!(x^{s}_{k-\!1})-\!\nabla\! f_{i^{s}_{k}}\!(\widetilde{x}^{s-1})\|^{2}]/\mathbb{E}[\|\nabla\!f_{i^{s}_{k}}\!(x^{s}_{k-\!1})-\!\nabla\! f_{i^{s}_{k}}\!(\widetilde{x}^{s-1})\|^{2}]$, and $\overline{\beta}_{k}<(\alpha\!-\!1)/{2}$. Using the inequality~\eqref{equ74}, then we have
\begin{equation}\label{equ101}
\begin{split}
\mathbb{E}\!\left[F(\widehat{x}^{s}_{k})-F(x^{*})\right]&\leq  \mathbb{E}\!\left[F(x^{s}_{k-1})-F(x^{*})-\frac{\beta_{k}}{2L(\alpha\!-\!1)}\|\nabla\!f_{i^{s}_{k}}\!(x^{s}_{k-\!1})-\!\nabla\! f_{i_{k}}\!(\widetilde{x}^{s-1})\|^{2}\right]\\
&= \mathbb{E}\!\left[F(x^{s}_{k-1})-F(x^{*})\right]-\frac{\overline{\beta}_{k}}{2L(\alpha\!-\!1)}\mathbb{E}\!\left[\|\nabla\!f_{i^{s}_{k}}\!(x^{s}_{k-\!1})-\!\nabla\! f_{i^{s}_{k}}\!(\widetilde{x}^{s-1})\|^{2}\right].
\end{split}
\end{equation}

There must exist a constant $\alpha_{k}\!>\!0$ such that $F(y^{s}_{k})\!-\!F(x^{*})\!=\!\alpha_{k}[F(x^{s}_{k-\!1})\!-\!F(x^{*})]$. Since
$\mathbb{E}\!\left[f(x^{s}_{k-\!1})\!-\!f_{i^{s}_{k}}\!(x^{s}_{k-\!1})\right]\!=\!0$, $\mathbb{E}\!\left[\nabla\! f_{i^{s}_{k}}(\widetilde{x}^{s-1})\right]\!=\!\nabla\! f(\widetilde{x}^{s-1})$, $\mathbb{E}\!\left[F_{i^{s}_{k}}\!(x^{*})\right]\!=\!F(x^{*})$, and $\mathbb{E}\!\left[F_{i^{s}_{k}}\!(x^{s}_{k-1})\right]\!=\!F(x^{s}_{k-1})$, and taking the expectation of both sides of (\ref{equ73}), we have
\begin{equation}\label{equ75}
\begin{split}
&\alpha_{k}\mathbb{E}\!\left[F(x^{s}_{k-1})-F(x^{*})\right]-\frac{c_{k}\overline{\beta}_{k}}{2L(\alpha\!-\!1)}\mathbb{E}\!\left[\|\nabla\!f_{i^{s}_{k}}\!(x^{s}_{k-\!1})-\!\nabla\! f_{i^{s}_{k}}\!(\widetilde{x}^{s-1})\|^{2}\right]\\
\leq\,&(1-\sigma)\mathbb{E}\!\left[F(\widehat{x}^{s}_{k-1})-F(x^{*})\right]+\frac{L\alpha\sigma^{2}}{2}\mathbb{E}\!\left[\|x^{*}-z^{s}_{k-1}\|^{2}-\|x^{*}-z^{s}_{k}\|^{2}\right]\\
&+\frac{1}{2L(\alpha\!-\!1)}\mathbb{E}\|\nabla\!f(x^{s}_{k-1})-p_{i^{s}_{k}}\|^{2}-\frac{c_{k}\overline{\beta}_{k}}{2L(\alpha\!-\!1)}\mathbb{E}\!\left[\|\nabla\!f_{i^{s}_{k}}\!(x^{s}_{k-\!1})-\!\nabla\! f_{i^{s}_{k}}\!(\widetilde{x}^{s-1})\|^{2}\right]\\
\leq&(1-\sigma)\mathbb{E}\!\left[F(\widehat{x}^{s}_{k-1})-F(x^{*})\right]+\frac{L\alpha\sigma^{2}}{2}\mathbb{E}\!\left[\|x^{*}-z^{s}_{k-1}\|^{2}-\|x^{*}-z^{s}_{k}\|^{2}\right]\\
&+\frac{2(1-c_{k}\overline{\beta}_{k})}{\alpha\!-\!1}\left[F(x^{s}_{k-1})-F(x^{*})+F(\widetilde{x}^{s-1})-F(x^{*})\right],
\end{split}
\end{equation}
where the second inequality follows from Lemma~\ref{lemm1} and Corollary~\ref{coro4}. Here, $c_{k}=\alpha_{k}-[{2(1\!-\!c_{k}\overline{\beta}_{k})}]/({\alpha\!-\!1})$, i.e.,
\begin{displaymath}
c_{k}= \frac{\alpha_{k}(\alpha-1)-2}{\alpha-1-2\overline{\beta}_{k}}.
\end{displaymath}
Since $\frac{2}{\alpha-1}<\sigma$ with the suitable choices of $\alpha$ and $\sigma$, we have $c_{k}>\alpha_{k}-\frac{2}{\alpha-1}>1-\sigma$. Thus, (\ref{equ75}) is rewritten as follows:
\begin{equation}\label{equ76}
\begin{split}
&c_{k}\mathbb{E}\!\left[F(x^{s}_{k-1})-F(x^{*})\right]-\frac{c_{k}\overline{\beta}_{k}}{2L(\alpha-1)}\mathbb{E}\!\left[\|p_{i^{s}_{k}}-\!\nabla\! f_{i^{s}_{k}}\!(\widetilde{x}^{s-1})\|^{2}\right]\\
\leq&\,(1-\sigma)\mathbb{E}[F(\widehat{x}^{s}_{k-1})-F(x^{*})]+\frac{L\alpha\sigma^{2}}{2}\mathbb{E}\!\left[\|x^{*}-z^{s}_{k-1}\|^{2}-\|x^{*}-z^{s}_{k}\|^{2}\right]\\
&\,+\frac{2(1-c_{k}\overline{\beta}_{k})}{\alpha-1}\mathbb{E}\!\left[F(\widetilde{x}^{s-1})-F(x^{*})\right].
\end{split}
\end{equation}

Combining the above two inequalities (\ref{equ101}) and (\ref{equ76}), we have
\begin{equation}
\begin{split}
&c_{k}\mathbb{E}\!\left[F(\widehat{x}^{s}_{k})-F(x^{*})\right]\\
\leq\,&(1-\sigma)\mathbb{E}\!\left[F(\widehat{x}^{s}_{k-1})-F(x^{*})\right]+\frac{L\alpha\sigma^{2}}{2}\mathbb{E}\!\left[\|x^{*}-z^{s}_{k-1}\|^{2}-\|x^{*}-z^{s}_{k}\|^{2}\right]\\
&+\frac{2(1-c_{k}\overline{\beta}_{k})}{\alpha-1}\mathbb{E}\!\left[F(\widetilde{x}^{s-1})-F(x^{*})\right].
\end{split}
\end{equation}

Taking the expectation over the random choice of $i^{s}_{1},i^{s}_{2},\ldots,i^{s}_{m}$, summing up the above inequality over $k=1,\ldots,m$, and $\widehat{x}^{s}_{0}=\widetilde{x}^{s-1}$, we have
\begin{equation}\label{equ102}
\begin{split}
&\mathbb{E}\!\left[\sum^{m}_{k=1}\!\left[c_{k}-(1-\sigma)\right][F(\widehat{x}^{s}_{k})-F(x^{*})]\right]\\
\leq\,&(1-\sigma)\mathbb{E}\!\left[F(\widetilde{x}^{s-1})-F(x^{*})\right]+\frac{L\alpha\sigma^{2}}{2}\mathbb{E}\!\left[\|x^{*}-z^{s}_{0}\|^{2}-\|x^{*}-z^{s}_{m}\|^{2}\right]\\
&+\mathbb{E}\!\left[\sum^{m}_{k=1}\frac{2(1-c_{k}\overline{\beta}_{k})}{\alpha-1}[F(\widetilde{x}^{s-1})-F(x^{*})]\right].
\end{split}
\end{equation}
In addition, there exists $\widehat{\beta}^{s}$ for the $s$-th epoch such that
\begin{equation}\label{equ103}
\begin{split}
&\;\mathbb{E}\!\left[\sum^{m}_{k=1}\left[c_{k}-(1-\sigma)\right][F(\widehat{x}^{s}_{k})-F(x^{*})]\right]\\
=&\;\mathbb{E}\!\left[\sum^{m}_{k=1}\left(\sigma-\frac{2}{\alpha-1}+\frac{2c_{k}\overline{\beta}_{k}}{\alpha-1}\right)[F(\widehat{x}^{s}_{k})-F(x^{*})]\right]\\
=&\;\left(\sigma-\frac{2}{\alpha-1}+\widehat{\beta}^{s}\right)\mathbb{E}\!\left[\sum^{m}_{k=1}[F(\widehat{x}^{s}_{k})-F(x^{*})]\right],
\end{split}
\end{equation}
where
\begin{displaymath}
\widehat{\beta}^{s}=\frac{\mathbb{E}\!\left[\sum^{m}_{k=1}\frac{2c_{k}\beta_{k}}{\alpha-1}[F(\widehat{x}^{s}_{k})-F(x^{*})]\right]}{\mathbb{E}\!\left[\sum^{m}_{k=1}[F(\widehat{x}^{s}_{k})-F(x^{*})]\right]}.
\end{displaymath}
Let $\widehat{\beta}=\min_{s=1,\ldots,S}\widehat{\beta}^{s}$. Using
\begin{displaymath}
\widetilde{x}^{s}=\frac{1}{m}\sum^{m}_{k=1}\widehat{x}^{s}_{k},\;\,F(\widetilde{x}^{s})\leq\frac{1}{m}\sum^{m}_{k=1}F(\widehat{x}^{s}_{k}),
\end{displaymath}
(\ref{equ102}) and (\ref{equ103}), we have
\begin{equation*}
\begin{split}
&\left(\sigma-\frac{2}{\alpha-1}+\widehat{\beta}\right)m\mathbb{E}\!\left[F(\widetilde{x}^{s})-F(x^{*})\right]\\
\leq\,&\left(1-\sigma+\frac{2m}{\alpha\!-\!1}\right)\mathbb{E}\!\left[F(\widetilde{x}^{s-1})-F(x^{*})\right]\\
&+\frac{L\alpha\sigma^{2}}{2}\mathbb{E}\!\left[\|x^{*}-z^{s}_{0}\|^{2}-\|x^{*}-z^{s}_{m}\|^{2}\right].
\end{split}
\end{equation*}
Therefore,
\begin{equation*}
\begin{split}
&\mathbb{E}\!\left[F(\widetilde{x}^{s})-F(x^{*})\right]\\
\leq\,&\left(\frac{1-\sigma}{\left(\sigma-\frac{2}{\alpha-1}+\widehat{\beta}\right)m}+\frac{2}{(\alpha\!-\!1)\left(\sigma-\frac{2}{\alpha-1}+\widehat{\beta}\right)}\right)\mathbb{E}\!\left[F(\widetilde{x}^{s-1})-F(x^{*})\right]\\
&+\frac{L\alpha\sigma^{2}}{2m\left(\sigma-\frac{2}{\alpha-1}+\widehat{\beta}\right)}\mathbb{E}\!\left[\|x^{*}-z^{s}_{0}\|^{2}-\|x^{*}-z^{s}_{m}\|^{2}\right].
\end{split}
\end{equation*}
This completes the proof.
\end{proof}
\vspace{1mm}

\section*{Appendix B: Proofs of Corollaries 1 and 2}

\textbf{Proof of Corollary 1:}
\begin{proof}
For $\mu$-strongly convex problems, and let $x^{s}_{0}=\widehat{x}^{s}_{0}=\widetilde{x}^{s-1}$ and
\begin{equation*}
z^{s}_{0}=\frac{x^{s}_{0}-(1-\sigma)\widehat{x}^{s}_{0}}{\sigma}=\widetilde{x}^{s-1}.
\end{equation*}
Due to the strong convexity of $F(\cdot)$, we have
\begin{equation}
\frac{\mu}{2}\|x^{*}-z^{s}_{0}\|^{2}=\frac{\mu}{2}\|x^{*}-\widetilde{x}^{s-1}\|^{2}\leq F(\widetilde{x}^{s-1})-F(x^{*}).
\end{equation}
Using Theorem 1, we obtain
\begin{equation*}
\begin{split}
&\mathbb{E}\!\left[F(\widetilde{x}^{s})-F(x^{*})\right]\\
\leq&\left(\frac{1-\sigma}{m(\sigma\!-\!\frac{2}{\alpha-1}\!+\!\widehat{\beta})}+\frac{2}{(\alpha\!-\!1)\left(\sigma\!-\!\frac{2}{\alpha-1}\!+\!\widehat{\beta}\right)}+\frac{L\alpha\sigma^{2}}{m\mu\left(\sigma\!-\!\frac{2}{\alpha-1}\!+\!\widehat{\beta}\right)}\right)\mathbb{E}\!\left[F(\widetilde{x}^{s-1})-F(x^{*})\right].
\end{split}
\end{equation*}

Replacing $\alpha$ and $\sigma$ in the above inequality with $19$ and $1/2$, respectively, we have
\begin{equation*}
\begin{split}
&\mathbb{E}\!\left[F(\widetilde{x}^{s})-F(x^{*})\right]\\
\leq&\left(\frac{9}{(7+18\widehat{\beta})m}+\frac{2}{7+18\widehat{\beta}}+\frac{171L}{(14+36\widehat{\beta})m\mu}\right)\mathbb{E}\!\left[F(\widetilde{x}^{s-1})-F(x^{*})\right]\!.
\end{split}
\end{equation*}
This completes the proof.
\end{proof}
\vspace{1mm}

\textbf{Proof of Corollary 2:}
\begin{proof}
For non-strongly convex problems, and using Theorem 1 with $\alpha=19$ and $\sigma=1/2$, we have
\begin{equation}\label{equ105}
\begin{split}
\mathbb{E}[F(\widetilde{x}^{s})-F(x^{*})]\leq&\;\frac{171L}{(28+72\widehat{\beta})m}\mathbb{E}\!\left[\left\|x^{*}-z^{s}_{0}\right\|^{2}-\left\|x^{*}-z^{s}_{m}\right\|^{2}\right]\\
&\;+\left(\frac{9}{(7+18\widehat{\beta})m}+\frac{2}{7+18\widehat{\beta}}\right)\left[F(\widetilde{x}^{s-1})-F(x^{*})\right]\!.
\end{split}
\end{equation}

According to the settings of Algorithm 1 for the non-strongly convex case, and let
\begin{equation*}
x^{s}_{0}=\widehat{x}^{s}_{0}=[x^{s-1}_{m}-(1-\sigma)\widehat{x}^{s-1}_{m}]/\sigma,
\end{equation*}
then we have
\begin{equation*}
z^{s}_{0}=\frac{x^{s}_{0}-(1-\sigma)\widehat{x}^{s}_{0}}{\sigma}=\frac{x^{s-1}_{m}-(1-\sigma)\widehat{x}^{s-1}_{m}}{\sigma},
\end{equation*}
and
\begin{equation*}
z^{s-1}_{m}=\frac{x^{s-1}_{m}-(1-\sigma)\widehat{x}^{s-1}_{m}}{\sigma}.
\end{equation*}
Therefore, $z^{s}_{0}=z^{s-1}_{m}$.

Using $z^{0}_{0}=\widetilde{x}^{0}$, and summing up the inequality (\ref{equ105}) over all $s=1,\ldots,S$, then
\begin{equation*}
\begin{split}
\mathbb{E}\!\left[F\!\left(\frac{1}{S}\sum^{S}_{s=1}\widetilde{x}^{s}\right)-F(x^{*})\right]\leq&\;\frac{171L}{(16+40\widehat{\beta})mS}\left\|x^{*}-\widetilde{x}^{0}\right\|^{2}\\
&\;+\left(\frac{9}{(4+8\widehat{\beta})mS}+\frac{1}{(2+4\widehat{\beta})S}\right)\left[F(\widetilde{x}^{0})-F(x^{*})\right]\!.
\end{split}
\end{equation*}

Due to the settings of Algorithm 1 for the non-strongly convex case, we have
\begin{equation*}\label{equ106}
\begin{split}
\mathbb{E}\!\left[F(\overline{x})-F(x^{*})\right]\leq&\;\frac{171L}{(16+40\widehat{\beta})mS}\left\|x^{*}-\widetilde{x}^{0}\right\|^{2}\\
&\;+\left(\frac{9}{(4+8\widehat{\beta})mS}+\frac{1}{(2+4\widehat{\beta})S}\right)\left[F(\widetilde{x}^{0})-F(x^{*})\right]\!.
\end{split}
\end{equation*}

This completes the proof.
 \end{proof}
\vspace{3mm}

\section*{Appendix C: Proof of Lemma 1}
Lemma 1 provides the upper bound on the expected variance of the variance reduced gradient estimator in (9) (i.e., the SAGA estimator introduced in~\cite{defazio:saga}). Before giving the proof of Lemma 1, we first present the following lemmas.

\begin{lemma}[\cite{defazio:saga}]
\label{lemm12}
Let $x^{*}$ be the optimal solution of Problem (1), then the following inequality holds for all $\phi_{j}$:
\begin{displaymath}
\frac{1}{n}\sum^{n}_{j=1}\left\|\nabla\!f_{j}(\phi_{j})-\nabla\!f_{j}(x^{*})\right\|^{2}\leq 2L\! \left[\frac{1}{n}\sum^{n}_{j=1}f_{j}(\phi_{j})-f(x^{*})-\frac{1}{n}\sum^{n}_{j=1}\left\langle \nabla\!f_{j}(x^{*}),\,\phi_{j}-x^{*}\right\rangle\right]\!.
\end{displaymath}
\end{lemma}
\vspace{1mm}

\begin{lemma}[\cite{defazio:saga}]
\label{lemm13}
\begin{displaymath}
\begin{split}
\mathbb{E}\!\!\left[\frac{1}{n}\!\sum^{n}_{j=1}\!\left\langle \partial F_{j}(x^{*}),\:\phi^{k}_{j}\!-\!x^{*}\right\rangle\right]\!=\!\frac{1}{n}\!\left\langle \partial F(x^{*}),\,x^{s}_{k-\!1}\!-\!x^{*}\right\rangle+(1\!-\!\frac{1}{n})\frac{1}{n}\!\sum^{n}_{j=1}\!\left\langle \partial F_{j}(x^{*}),\,\phi^{k-\!1}_{j}\!-\!x^{*}\right\rangle\!,
\end{split}
\end{displaymath}
where $F_{i}(\cdot)\!=\!f_{i}(\cdot)+r(\cdot)$, and $\partial F_{i}(x^{*})$ denotes a sub-gradient of $F_{i}(\cdot)$ at $x^{*}$.
\end{lemma}

\textbf{Proof of Lemma 1:}
\begin{proof}
Using Lemma~\ref{lemm13}, we have
\begin{equation}\label{equ107}
\begin{split}
&\,\mathbb{E}\!\left[\frac{1}{n}\sum^n_{j=1}\left\langle \partial F_{j}(x^{*}),\phi^{k-1}_{j}-x^{*}\right\rangle\right]\\
=\,&\mathbb{E}\!\left[\frac{1}{n}\left\langle \partial F(x^{*}),\:x^{s}_{k-2}-x^{*}\right\rangle+(1-\frac{1}{n})\frac{1}{n}\sum^n_{j=1}\left\langle \partial F_{j}(x^{*}),\:\phi^{k-2}_{j}-x^{*}\right\rangle\right]\\
=\,&(1-\frac{1}{n})\mathbb{E}\!\left[\frac{1}{n}\sum^n_{j=1}\left\langle \partial F_{j}(x^{*}),\:\phi^{k-2}_{j}-x^{*}\right\rangle\right]\\
=\,&(1-\frac{1}{n})^{k-1}\mathbb{E}[\frac{1}{n}\sum^n_{j=1}\left\langle \partial F_{j}(x^{*}),\:\phi^{0}_{j}-x^{*}\right\rangle]\\
=\,&(1-\frac{1}{n})^{k-1}\mathbb{E}\!\left[\left\langle \partial F(x^{*}),\:{x}^{s}_{0}-x^{*}\right\rangle\right]\\
=\,&0,
\end{split}
\end{equation}
where the second and last equalities hold from the optimality of $x^{*}$, the third equality holds due to Lemma~\ref{lemm13}, and the fourth equality is due to $\phi^{0}_{j}\!=\!{x}^{s}_{0}$ for all $j=1,\ldots,n$.

Since $\mathbb{E}[\nabla\! f_{i^{s}_{k}}\!(x^{s}_{k-\!1})]\!=\!\nabla\!f(x^{s}_{k-\!1})$ and $\mathbb{E}[\nabla\! f_{i^{s}_{k}}\!(\phi^{k-\!1}_{i^{s}_{k}})]\!=\!\frac{1}{n}\!\sum^{n}_{i=1}\!\nabla\!f_{i}(\phi^{k-\!1}_{i})$, then for any $i^{s}_{k}\!\in\![n]$,
\begin{displaymath}
\begin{split}
&\mathbb{E}\!\left[\left\|\nabla\! f_{i^{s}_{k}}\!(x^{s}_{k-\!1})-\!\nabla\!f(x^{s}_{k-\!1})-\nabla\! f_{i^{s}_{k}}\!(\phi^{k-1}_{i^{s}_{k}})+\frac{1}{n}\!\sum^{n}_{j=1}\nabla\! f_{j}(\phi^{k-1}_{j})\right\|^{2}\right]\\
=\,&\mathbb{E}\!\left[\left\|\nabla\! f_{i^{s}_{k}}\!(x^{s}_{k-\!1})-\nabla\! f_{i^{s}_{k}}\!(\phi^{k-1}_{i_{k}})\right\|^{2}\right]-\|\nabla\!f(x_{k-\!1})-\frac{1}{n}\!\sum^{n}_{j=1}\nabla\! f_{j}(\phi^{k-1}_{j})\|^{2}\\
\leq\,&\mathbb{E}\!\left[\left\|\nabla\! f_{i^{s}_{k}}\!(x^{s}_{k-\!1})-\nabla\! f_{i^{s}_{k}}\!(\phi^{k-1}_{i^{s}_{k}})\right\|^{2}\right]\\
=\,&\mathbb{E}\!\left[\left\|[\nabla\! f_{i^{s}_{k}}\!(x^{s}_{k-\!1})-\nabla\! f_{i^{s}_{k}}\!(x^{*})]-[\nabla\! f_{i^{s}_{k}}\!(\phi^{k-1}_{i^{s}_{k}})-\nabla\! f_{i^{s}_{k}}\!(x^{*})]\right\|^{2}\right]\\
\leq\,&2\mathbb{E}\!\left[\left\|\nabla\! f_{i^{s}_{k}}\!(\phi^{k-1}_{i^{s}_{k}})-\nabla\! f_{i^{s}_{k}}\!(x^{*})\right\|^{2}\right]+2\mathbb{E}\!\left[\left\|\nabla\! f_{i^{s}_{k}}\!(x^{s}_{k-\!1})-\nabla\! f_{i^{s}_{k}}\!(x^{*})\right\|^{2}\right]\\
\leq\,& 4L\!\left[\frac{1}{n}\!\sum^{n}_{j=1}\!f_{j}(\phi^{k-\!1}_{j})\!-\!f(x^{*})\!+\!\frac{1}{n}\!\sum^n_{j=1}\!\langle\xi^{*}\!,\phi^{k-\!1}_{j}\!-\!x^{*}\rangle\!-\!\frac{1}{n}\!\sum^n_{j=1}\!\left\langle \nabla\! f_{j}(x^{*})\!+\!\xi^{*}\!,\,\phi^{k-\!1}_{j}\!-\!x^{*}\right\rangle\right]\\
&+4L\!\left[F(x^{s}_{k-\!1})\!-\!F(x^{*})\right]\\
\leq\,& 4L\!\left[\frac{1}{n}\!\sum^{n}_{j=1}\!f_{j}(\phi^{k-\!1}_{j})\!-\!f(x^{*})\!+\!\frac{1}{n}\!\sum^n_{j=1}\!r(\phi^{k-\!1}_{j})\!-\!r(x^{*})\right]+4L\!\left[F(x^{s}_{k-\!1})\!-\!F(x^{*})\right]\\
=\,&4L\!\left[\frac{1}{n}\!\sum^{n}_{j=1}F_{j}(\phi^{k-\!1}_{j})-F(x^{*})+F(x^{s}_{k-\!1})-F(x^{*})\right]\!,
\end{split}
\end{displaymath}
where $\xi^{*}\!=\!\partial r(x^{*})$, if $r(\cdot)$ is non-smooth, and $\xi^{*}\!=\!\nabla r(x^{*})$ otherwise. The first equality holds due to the fact that $\mathbb{E}[\|x\!-\!\mathbb{E}x\|^{2}]\!=\!\mathbb{E}[\|x\|^{2}]\!-\!\|\mathbb{E}x\|^{2}$; the second inequality holds due to the fact that $\|a-b\|^{2}\leq2(\|a\|^{2}+\|b\|^{2})$; and the third inequality follows from Lemma \ref{lemm12} and Lemma 3.4 in~\cite{xiao:prox-svrg}; and the last inequality holds due to the equality in (\ref{equ107}) and the convexity of $r(\cdot)$.
\end{proof}

\section*{Appendix D: Proofs of Theorem 2 and Corollary 3}
From Lemma 1, we immediately have the following result, which is useful in our convergence analysis below.

\begin{corollary}\label{coro5}
For any $\alpha\geq\beta>0$, we have
\begin{equation*}
\begin{split}
&\alpha\;\!\mathbb{E}\!\!\left[\left\|\nabla\! f_{i^{s}_{k}}\!(x^{s}_{k-\!1})\!-\!\nabla\! f(x^{s}_{k-\!1})\!-\!\nabla\! f_{i^{s}_{k}}\!(\phi^{k-\!1}_{i^{s}_{k}})\!+\!\frac{1}{n}\!\sum^{n}_{j=1}\!\nabla\! f_{j}(\phi^{k-\!1}_{j})\right\|^{2}\right]\!-\!\beta\;\!\mathbb{E}\!\!\left[\left\|\nabla\! f_{i^{s}_{k}}\!(x^{s}_{k-\!1})\!-\!\nabla\! f_{i^{s}_{k}}\!(\phi^{k-\!1}_{i^{s}_{k}})\right\|^{2}\right]\\
&\leq 4L(\alpha\!-\!\beta)\!\left[F(x^{s}_{k-\!1})-F(x^{*})+\frac{1}{n}\sum^{n}_{j=1}F_{j}(\phi^{k-\!1}_{j})-F(x^{*})\right]\!.
\end{split}
\end{equation*}
\end{corollary}

\vspace{3mm}
\textbf{Proof of Theorem 2:}
\begin{proof}
Let $p_{i^{s}_{k}}=\nabla\! f_{i^{s}_{k}}\!(x^{s}_{k-\!1})-\nabla\! f_{i^{s}_{k}}\!(\phi^{k-\!1}_{i^{s}_{k}})+\frac{1}{n}\!\sum^{n}_{j=1}\!\nabla\! f_{j}(\phi^{k-\!1}_{j})$, and $\widehat{x}^{s}_{k}=\theta_{k}x^{s}_{k-1}$. By the similar derivation for \eqref{equ75} only replacing Lemma 2 and Corollary~\ref{coro4} with Lemma 1 and Corollary~\ref{coro5}, then the following inequality holds:
\begin{equation}\label{equ78}
\begin{split}
&\alpha_{k}\mathbb{E}\!\left[F(x^{s}_{k-1})-F(x^{*})\right]-\frac{c_{k}\overline{\beta}_{k}}{2L(\alpha\!-\!1)}\mathbb{E}\!\left[\|p_{i^{s}_{k}}-\nabla f(x^{s}_{k-\!1})\|^{2}\right]\\
\leq\,&(1-\sigma)\mathbb{E}\!\left[F(\widehat{x}^{s}_{k-1})-F(x^{*})\right]+\frac{L\alpha\sigma^{2}}{2}\mathbb{E}\!\left[\|x^{*}\!-z^{s}_{k-1}\|^{2}-\|x^{*}\!-z^{s}_{k}\|^{2}\right]\\
&\!\!+\!\frac{2(1\!-\!c_{k}\overline{\beta}_{k})}{\alpha-1}\!\left[F(x^{s}_{k-\!1})\!-\!F(x^{*})\!+\!\frac{1}{n}\!\sum^{n}_{j=1}F_{j}(\phi^{k-\!1}_{j})\!-\!F(x^{*})\right]\!.
\end{split}
\end{equation}

Given $q>0$, and using the result in the proof of Theorem 1 in \cite{defazio:saga}, we obtain
\begin{equation}\label{equ79}
\begin{split}
\frac{q}{n}\!\left[F(x^{s}_{k-\!1})-F(x^{*})\right]=q\mathbb{E}\!\left[\frac{1}{n}\!\sum^{n}_{j=1}\!F_{j}(\phi^{k}_{j})-F(x^{*})\right]-q(1\!-\!\frac{1}{n})\!\left(\frac{1}{n}\!\sum^{n}_{j=1}\!F_{j}(\phi^{k-\!1}_{j})-F(x^{*})\right)\!.
\end{split}
\end{equation}

Using (\ref{equ79}) and Lemma 1, then (\ref{equ78}) is rewritten as follows:
\begin{equation}\label{equ80}
\begin{split}
&\left(\alpha_{k}\!-\!\frac{2(1\!-\!c_{k}\overline{\beta}_{k})}{\alpha-1}\!-\!\frac{q}{n}\right)\mathbb{E}\!\left[F(x^{s}_{k-\!1})\!-\!F(x^{*})\right]\!-\!\frac{c_{k}\overline{\beta}_{k}}{2L(\alpha\!-\!1)}\mathbb{E}\!\left[\|\nabla\! f_{i^{s}_{k}}(x^{s}_{k-\!1})\!-\!\nabla\! f_{i^{s}_{k}}\!(\phi^{k-\!1}_{i^{s}_{k}})\|^{2}\right]\\
\leq\,&(1-\sigma)\mathbb{E}\!\left[F(\widehat{x}^{s}_{k-1})-F(x^{*})\right]+\frac{L\alpha\sigma^{2}}{2}\mathbb{E}\!\left[\|x^{*}-z^{s}_{k-1}\|^{2}-\|x^{*}-z^{s}_{k}\|^{2}\right]\\
&+\left(\frac{2(1\!-\!c_{k}\overline{\beta}_{k})}{\alpha-1}\!+\!q(1\!-\!\frac{1}{n})\right)\!\left[\frac{1}{n}\sum^{n}_{j=1}F_{j}(\phi^{k-\!1}_{j})\!-\!F(x^{*})\right]\!-q\mathbb{E}\!\left[\frac{1}{n}\!\sum^{n}_{j=1}F_{j}(\phi^{k}_{j})\!-\!F(x^{*})\right]\!.
\end{split}
\end{equation}
Let
\begin{equation}\label{equ81}
\frac{q}{n}=\frac{2}{\alpha-1}\quad \textup{and}\quad c_{k}=\alpha_{k}-\frac{q}{n}-\frac{2(1-c_{k}\overline{\beta}_{k})}{\alpha-1}.
\end{equation}
Therefore,
\begin{equation*}
\begin{split}
c_{k}=\frac{\alpha_{k}(\alpha-1)-4}{\alpha-1-2\overline{\beta}_{k}}>0.
\end{split}
\end{equation*}

Using (\ref{equ80}) and (\ref{equ81}), we have
\begin{equation*}
\begin{split}
&c_{k}\mathbb{E}\!\left[F(\widehat{x}^{s}_{k})-F(x^{*})\right]\\
\leq\,& c_{k}\mathbb{E}\!\left[F(x^{s}_{k-1})-F(x^{*})\right]-\frac{c_{k}\overline{\beta}_{k}}{2L(\alpha-1)}\mathbb{E}\!\left[\left\|\nabla\! f_{i^{s}_{k}}(x^{s}_{k-1})-\nabla\! f_{i^{s}_{k}}\!(\phi^{k-\!1}_{i^{s}_{k}})\right\|^{2}\right]\\
\leq\,&(1-\sigma)\left[F(\widehat{x}^{s}_{k-1})-F(x^{*})\right]+\frac{L\alpha\sigma^{2}}{2}\mathbb{E}\!\left[\left\|x^{*}-z^{s}_{k-1}\|^{2}-\|x^{*}-z^{s}_{k}\right\|^{2}\right]\\
&+\!\left(\frac{2(1\!-\!c_{k}\overline{\beta}_{k})}{\alpha-1}\!+\!q(1\!-\!\frac{1}{n})\right)\!\left(\frac{1}{n}\!\sum^{n}_{j=1}\!F_{j}(\phi^{k-\!1}_{j})\!-\!F(x^{*})\right)\!-\!q\mathbb{E}\!\!\left[\frac{1}{n}\!\sum^{n}_{j=1}\!F_{j}(\phi^{k}_{j})\!-\!F(x^{*})\right]\\
\leq\,&(1-\sigma)[F(\widehat{x}^{s}_{k-1})-F(x^{*})]+\frac{L\alpha\sigma^{2}}{2}\mathbb{E}\!\left[\left\|x^{*}-z^{s}_{k-1}\right\|^{2}-\left\|x^{*}-z^{s}_{k}\right\|^{2}\right]\\
&+q\left(\frac{1}{n}\sum^{n}_{j=1}F_{j}(\phi^{k-1}_{j})-F(x^{*})\right)-q\mathbb{E}\!\left[\frac{1}{n}\sum^{n}_{j=1}F_{j}(\phi^{k}_{j})-F(x^{*})\right]\!.
\end{split}
\end{equation*}

Taking the expectation over the random choice of the history of $i^{s}_{1},\ldots,i^{s}_{m}$, using Lemma 1, and summing up the above inequality over $k=1,\ldots,m$, then \begin{equation}\label{equ82}
\begin{split}
&\mathbb{E}\!\left[\sum^{m}_{k=1}(c_{k}-(1-\sigma))\left[F(\widehat{x}^{s}_{k})-F(x^{*})\right]\right]\\
\leq\,&(1-\sigma)[F(\widehat{x}^{s}_{0})-F(x^{*})]+\frac{L\alpha\sigma^{2}}{2}\mathbb{E}\!\left[\left\|x^{*}-z^{s}_{0}\right\|^{2}-\left\|x^{*}-z^{s}_{m}\right\|^{2}\right]\\
&+q\mathbb{E}\!\left[\frac{1}{n}\sum^{n}_{j=1}F_{j}(\phi^{0}_{j})-F(x^{*})\right]-q\mathbb{E}\!\left[\frac{1}{n}\sum^{n}_{j=1}F_{j}(\phi^{m}_{j})-F(x^{*})\right].
\end{split}
\end{equation}
$c_{k}$ and ${q}/{n}$ are defined in (\ref{equ81}) with $\sigma=1/2$ and $\alpha=19$, then there exists $\widehat{\beta}^{s}\geq 0$ for the $s$-th epoch such that
\begin{equation}\label{equ83}
\begin{split}
\mathbb{E}\!\left[\sum^{m}_{k=1}(c_{k}-(1\!-\!\sigma))\left[F(\widehat{x}^{s}_{k})-F(x^{*})\right]\right]&=\mathbb{E}\!\left[\sum^{m}_{k=1}\frac{7+2c_{k}\overline{\beta}_{k}}{9}\left[F(\widehat{x}^{s}_{k})-F(x^{*})\right]\right]\\
&=\left(\frac{7}{9}\!+\!\widehat{\beta}^{s}\right)\mathbb{E}\!\left[\sum^{m}_{k=1}[F(\widehat{x}^{s}_{k})-F(x^{*})]\right],
\end{split}
\end{equation}
where $\widehat{\beta}^{s}=\mathbb{E}\!\left[\frac{2}{9}\sum^{m}_{k=1}c_{k}\overline{\beta}_{k}\!\left(F(\widehat{x}^{s}_{k})-F(x^{*})\right)\right]/\mathbb{E}[\sum^{m}_{k=1}(F(\widehat{x}^{s}_{k})-F(x^{*}))]$.

Let $\widehat{\beta}\!=\!\min_{s=1,\ldots,S}\widehat{\beta}^{s}$ as in the proof of Theorem 1. Using $\widetilde{x}^{s}\!=\!\frac{1}{m}\sum^{m}_{k=1}\widehat{x}^{s}_{k}$, $F(\widetilde{x}^{s})\!\leq\! \frac{1}{m}\sum^{m}_{k=1}F(\widehat{x}^{s}_{k})$, and (\ref{equ83}), then (\ref{equ82}) is rewritten as follows:
\begin{equation*}
\begin{split}
&m\left(\frac{7}{9}+\widehat{\beta}\right)\mathbb{E}\!\left[F(\widetilde{x}^{s})-F(x^{*})\right]\\
\leq\,&\frac{1}{2}\!\left[F(\widetilde{x}^{s-1})-F(x^{*})\right]+\frac{19L}{8}\mathbb{E}\!\left[\left\|x^{*}-\widetilde{x}^{s-1}\right\|^{2}\right]\\
&+q\left(\frac{1}{n}\sum^{n}_{j=1}F_{j}(\phi^{0}_{j})-F(x^{*})\right)-q\:\!\mathbb{E}\!\left[\frac{1}{n}\sum^{n}_{j=1}F_{j}(\phi^{m}_{j})-F(x^{*})\right]\\
\leq\,&\frac{19L}{8}\mathbb{E}\!\left[\|x^{*}-\widetilde{x}^{s-1}\|^{2}\right]+\left(\frac{1}{2}+q\right)\left[F(\widetilde{x}^{s-1})-F(x^{*})\right],
\end{split}
\end{equation*}
where the first and second inequalities hold due to the facts that $\widehat{x}^{s}_{0}=\widetilde{x}^{s-1}$ and $\phi^{0}_{j}=\widetilde{x}^{s-1}$.

Setting $\sigma=1/2$, $\alpha=19$, $\frac{2}{\alpha-1}=\frac{q}{n}$, and using the $\mu$-strongly convex property, we have
\begin{equation*}
m\left(\frac{7}{9}+\widehat{\beta}\right)\mathbb{E}\!\left[F(\widetilde{x}^{s})-F(x^{*})\right]\leq\left(\frac{1}{2}+\frac{n}{9}+\frac{19L}{4\mu}\right)\left[F(\widetilde{x}^{s-1})-F(x^{*})\right].
\end{equation*}

Therefore,
\begin{equation*}
\mathbb{E}\!\left[F(\widetilde{x}^{s})-F(x^{*})\right]\leq\left(\frac{n}{(7+9\widehat{\beta})m}+\frac{9}{(14+18\widehat{\beta})m}+\frac{171L}{(28+36\widehat{\beta})\mu m}\right)\mathbb{E}\!\left[F(\widetilde{x}^{s-1})-F(x^{*})\right]\!.
\end{equation*}
This completes the proof.
\end{proof}

\vspace{3mm}
\textbf{Proof of Corollary 3:}
\begin{proof}
Using the similar derivation in the proof of Theorem 2 for the strongly convex case, and with the same parameter settings (i.e., $\sigma\!=\!1/2$, $\alpha\!=\!19$, and $\frac{2}{\alpha-1}\!=\!\frac{q}{n}$), we have
\begin{equation}\label{equ108}
\begin{split}
&\left(\frac{7}{9}+\widehat{\beta}\right)\mathbb{E}\!\left[F(\widetilde{x}^{s})-F(x^{*})\right]\\
\leq&\;\frac{19L}{8m}\mathbb{E}\!\left[\left\|x^{*}-z^{s}_{0}\right\|^{2}-\left\|x^{*}-z^{s}_{m}\right\|^{2}\right]+\left(\frac{1}{2m}+\frac{n}{9m}\right)\left[F(\widetilde{x}^{s-1})-F(x^{*})\right]\!.
\end{split}
\end{equation}

According to the settings of Algorithm~\ref{alg12} for the non-strongly convex case as in Algorithm 1, we have
\begin{equation*}
z^{s}_{0}=z^{s-1}_{m},\;\,z^{0}_{0}=\widetilde{x}^{0}.
\end{equation*}
Summing up the above inequality (\ref{equ108}) over $s=1,\ldots,S$, and setting $m=n$, then
\begin{equation*}
\begin{split}
&\,\mathbb{E}\!\left[F(\overline{x})-F(x^{*})\right]\leq\mathbb{E}\!\left[F\left(\frac{1}{S}\sum^{S}_{s=1}\widetilde{x}^{s}\right)-F(x^{*})\right]\\
\leq&\,\frac{171L}{(49\!+\!56\widehat{\beta})nS}\|x^{*}\!-\!\widetilde{x}^{0}\|^{2}\!+\!\left(\frac{81}{(98\!+\!126\widehat{\beta})nS}\!+\!\frac{9}{(49\!+\!63\widehat{\beta})S}\right)\!\left[F(\widetilde{x}^{0})\!-\!F(x^{*})\right]\!.
\end{split}
\end{equation*}
This completes the proof.
\end{proof}

\section*{Appendix E: Codes and Data Sets}
In this section, we first present the detailed descriptions for the three popular data sets: Covtype, SUSY and Ijcnn1, which were obtained from the LIBSVM Data website{\footnote{\url{https://www.csie.ntu.edu.tw/~cjlin/libsvm/}}}, as shown in Table~\ref{tab_sim1}. The C++ code of SVRG~\cite{johnson:svrg} was downloaded from~\url{http://riejohnson.com/svrg_download.html}. For fair comparison, we implemented the proposed SVRG-SD and SAGA-SD (see Algorithm~\ref{alg12}) algorithms, SAGA~\cite{defazio:saga}, Prox-SVRG~\cite{xiao:prox-svrg}, Catalyst~\cite{lin:vrsg} (which is based on SVRG and has three important parameters: $\alpha_{k}$, $\kappa$, and the learning rate, $\eta$), and Katyusha~\cite{zhu:Katyusha} in C++ with a Matlab interface\footnote{The codes of all those algorithms can be downloaded by the following link:\\ \centerline{\url{https://www.dropbox.com/s/5sg7h49qctr9ahi/Code_VD_SGD.zip?dl=0.}}}, and performed all the experiments on a PC with an Intel i5-2400 CPU and 16GB RAM.

\begin{table}[!th]
\centering
\caption{Data sets and their regularization parameters.}
\label{tab_sim1}
\setlength{\tabcolsep}{6.9pt}
\linespread{1.36}
\begin{tabular}{lccc}
\hline
\ Data sets   & Sizes $n$    & Dimensions $d$  & Sparsity \\
\hline
\ Ijcnn1      & 49,990         & 22             & 59.09\% \\
\ Covtype     & 581,012        & 54             & 22.12\% \\
\ SUSY        & 5,000,000      & 18             & 98.82\% \\
\ Sido0       & 12,678         & 4,932          & 9.84\% \\
\hline
\end{tabular}
\end{table}

\begin{algorithm}[t]
\caption{SAGA-SD}
\label{alg12}
\renewcommand{\algorithmicrequire}{\textbf{Input:}}
\renewcommand{\algorithmicensure}{\textbf{Initialize:}}
\renewcommand{\algorithmicoutput}{\textbf{Output:}}
\begin{algorithmic}[1]
\REQUIRE the number of epochs $S$, the number of iterations $m$ per epoch, and step size $\eta$.\\
\ENSURE $\widetilde{x}^{0}$.
\FOR{$s=1,2,\ldots S$}
\STATE {$\,x^{s}_{0}\!=\widehat{x}^{s}_{0}\!=\widetilde{x}^{s-\!1}$;}
\FOR{$k=1,\ldots,m$}
\STATE {Pick $i^{s}_{k}$ uniformly at random from $[n]$;}
\STATE {Take $\phi^{k}_{i^{s}_{k}}\!=x^{s}_{k-1}$ and store $\nabla\! f_{i^{s}_{k}}\!(\phi^{k}_{i^{s}_{k}})$ in the table;}
\STATE {$\widetilde{\nabla}\!f_{i^{s}_{k}}\!(x^{s}_{k-\!1})=\nabla\! f_{i^{s}_{k}}\!(x^{s}_{k-\!1})\!-\!\nabla\! f_{i^{s}_{k}}\!(\phi^{k-\!1}_{i^{s}_{k}})\!+\!\frac{1}{n}\!\sum^{n}_{j=1}\!\!\nabla\! f_{j}(\phi^{k-\!1}_{j})$;}
\STATE {$y^{s}_{k}=\textrm{prox}^{r}_{\eta}\!\left(x^{s}_{k-1}-\eta\widetilde{\nabla} f_{i^{s}_{k}}(x^{s}_{k-1})\right)$;}
\STATE {$\theta_{k}= \arg\min_{\theta\in\mathbb{R}}F(\theta x^{s}_{k-\!1})\!+\!\frac{\zeta(1\!-\!\theta)^2}{2}\!\|\nabla\!f_{i^{s}_{k}}\!(x^{s}_{k-\!1})-\!\nabla\! f_{i^{s}_{k}}\!(\phi^{k-\!1}_{i^{s}_{k}})\|^{2}$;}
\STATE {$x^{s}_{k}=y^{s}_{k}+(1\!-\!\sigma)(\widehat{x}^{s}_{k}-\widehat{x}^{s}_{k-1})$\: and\: $\widehat{x}^{s}_{k}=\theta_{k}x^{s}_{k-\!1}$;}
\ENDFOR
\STATE {$\widetilde{x}^{s}=\frac{1}{m}\!\sum^{m}_{k=1}\!\widehat{x}^{s}_{k}$;}
\ENDFOR
\OUTPUT $\overline{x}\!=\!\widetilde{x}^{S}$
\end{algorithmic}
\end{algorithm}

\begin{figure}[t]
\centering
\subfigure[Ijcnn1, $\lambda\!=\!10^{-4}$]{\includegraphics[width=0.326\columnwidth]{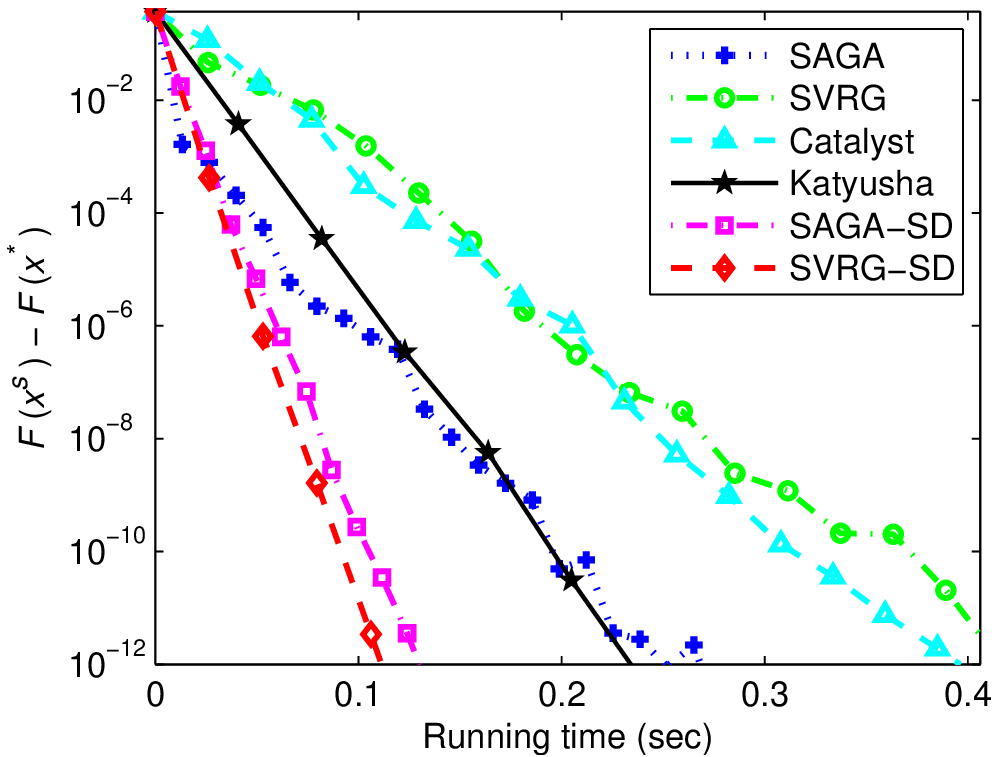}}\,
\subfigure[Covtype, $\lambda\!=\!10^{-4}$]{\includegraphics[width=0.326\columnwidth]{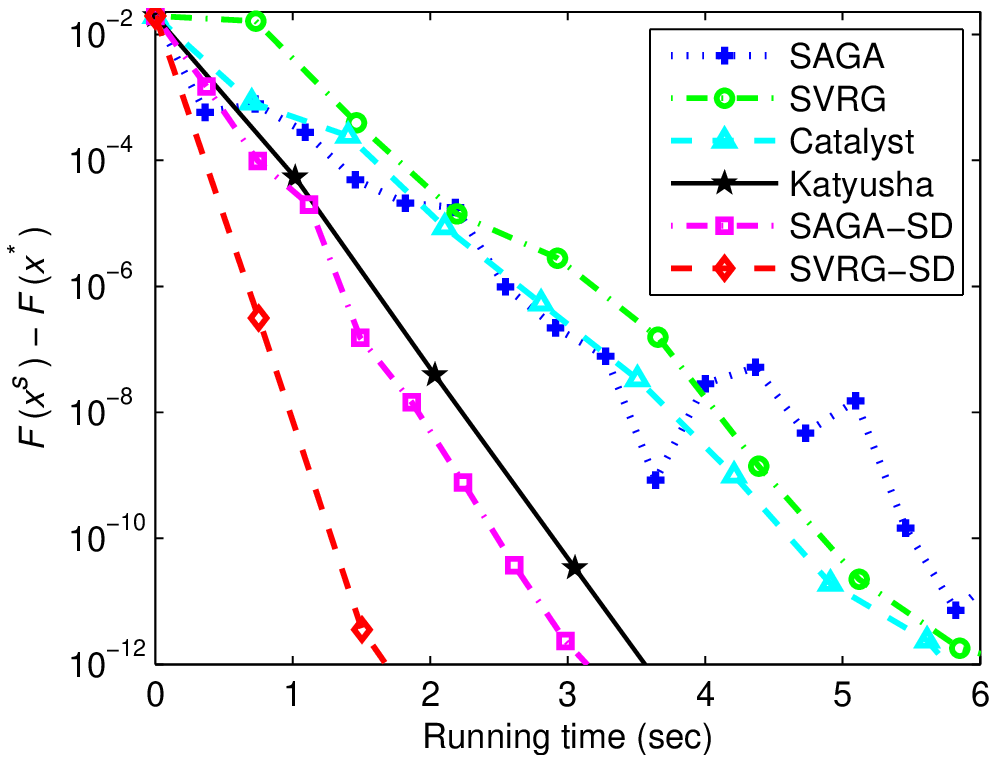}}\,
\subfigure[SUSY, $\lambda\!=\!10^{-4}$]{\includegraphics[width=0.326\columnwidth]{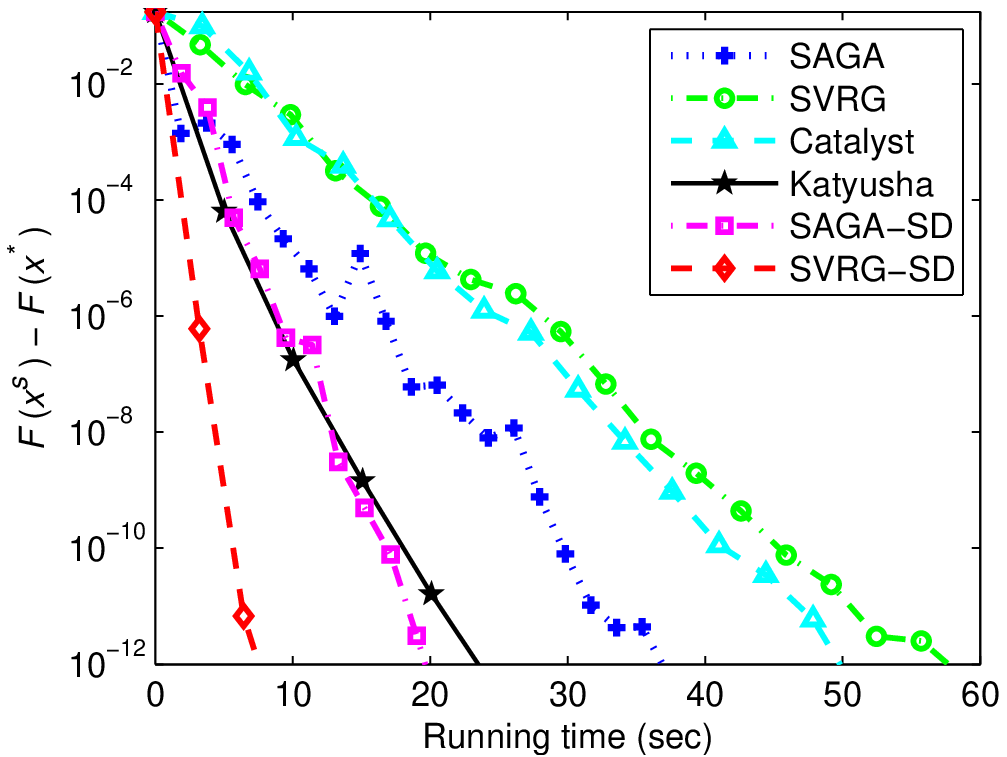}}
\vspace{1mm}

\includegraphics[width=0.326\columnwidth]{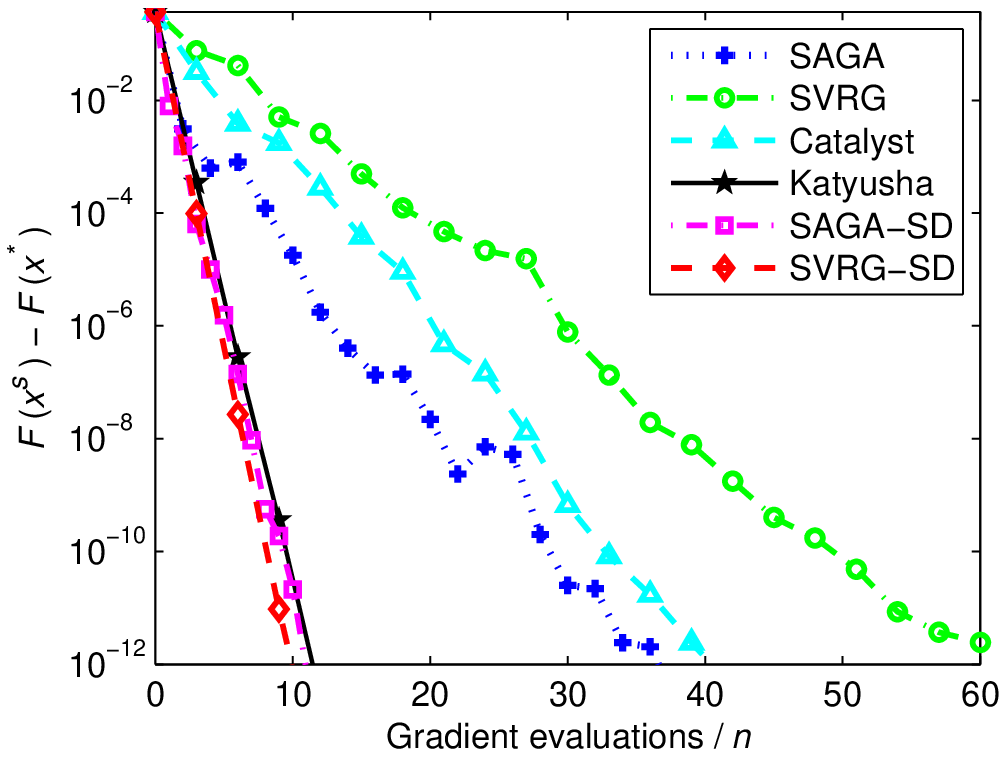}\,
\includegraphics[width=0.326\columnwidth]{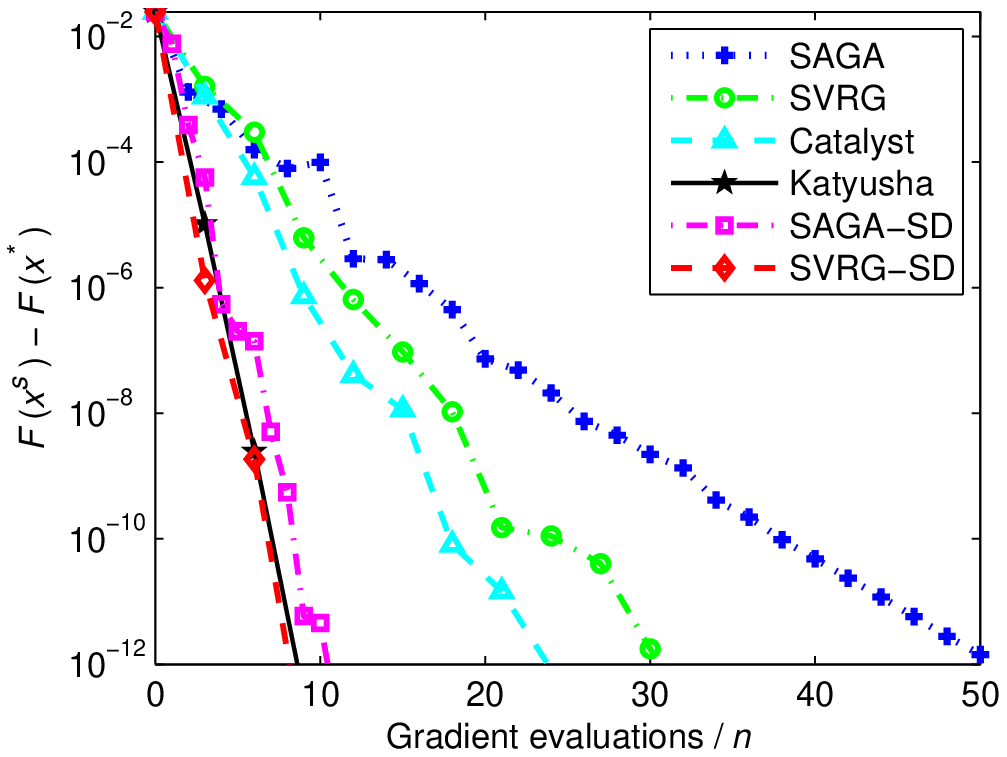}\,
\includegraphics[width=0.326\columnwidth]{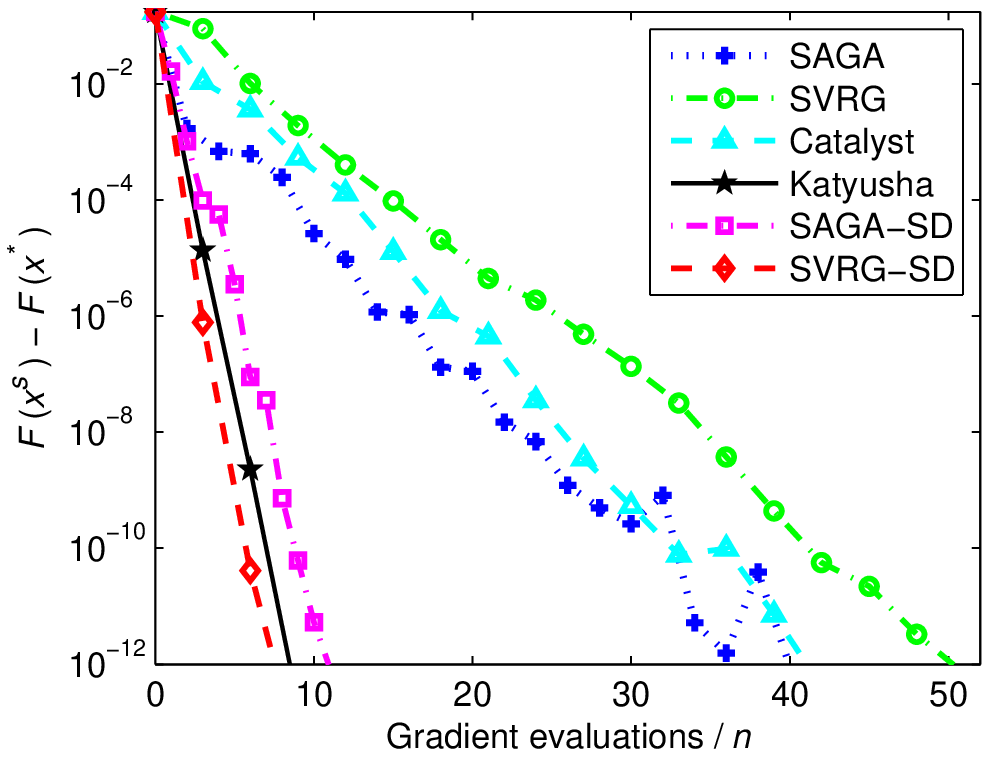}

\subfigure[Ijcnn1, $\lambda\!=\!10^{-5}$]{\includegraphics[width=0.326\columnwidth]{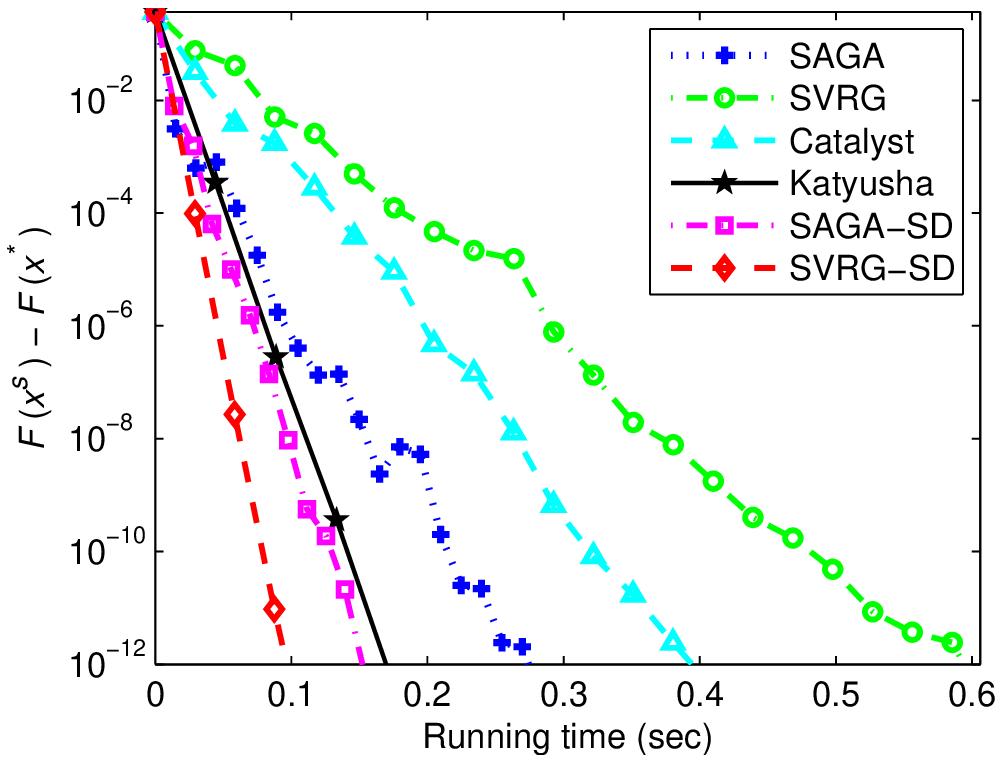}}\,
\subfigure[Covtype, $\lambda\!=\!10^{-5}$]{\includegraphics[width=0.326\columnwidth]{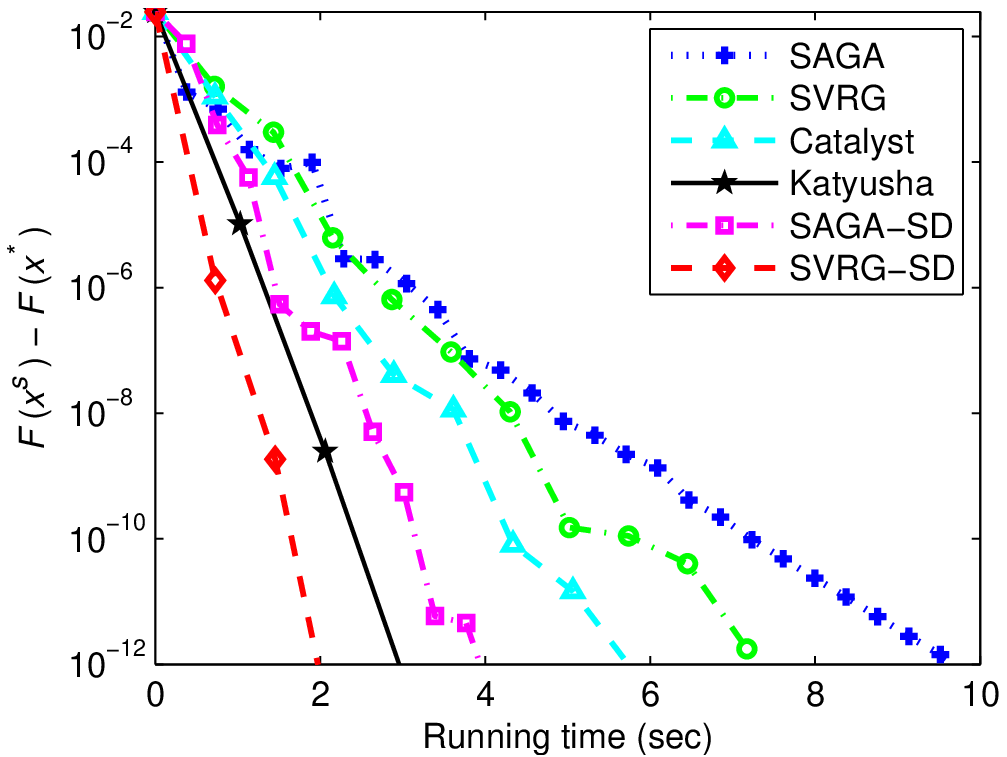}}\,
\subfigure[SUSY, $\lambda\!=\!10^{-5}$]{\includegraphics[width=0.326\columnwidth]{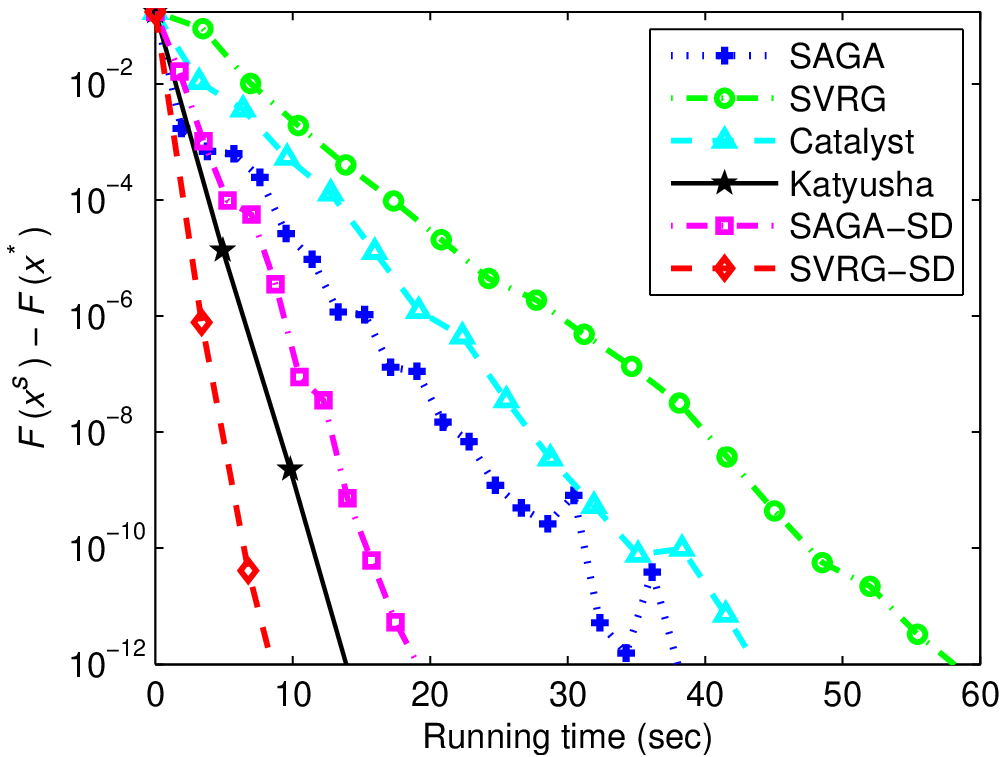}}
\caption{Comparison of different variance reduced SGD methods for solving strongly convex ridge regression problems. The vertical axis is the objective value minus the minimum, and the horizontal axis denotes the running time (seconds) or the number of effective passes over the data.}
\label{fig_sim1}
\end{figure}

\begin{figure}[th]
\centering
\includegraphics[width=0.326\columnwidth]{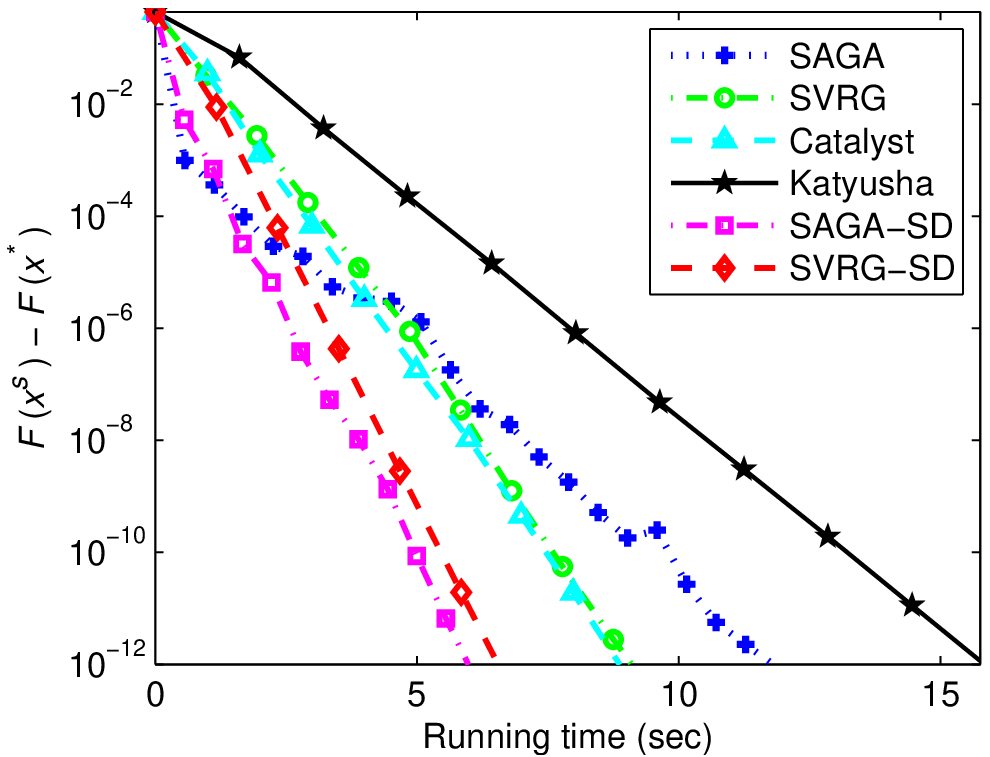}\,
\includegraphics[width=0.326\columnwidth]{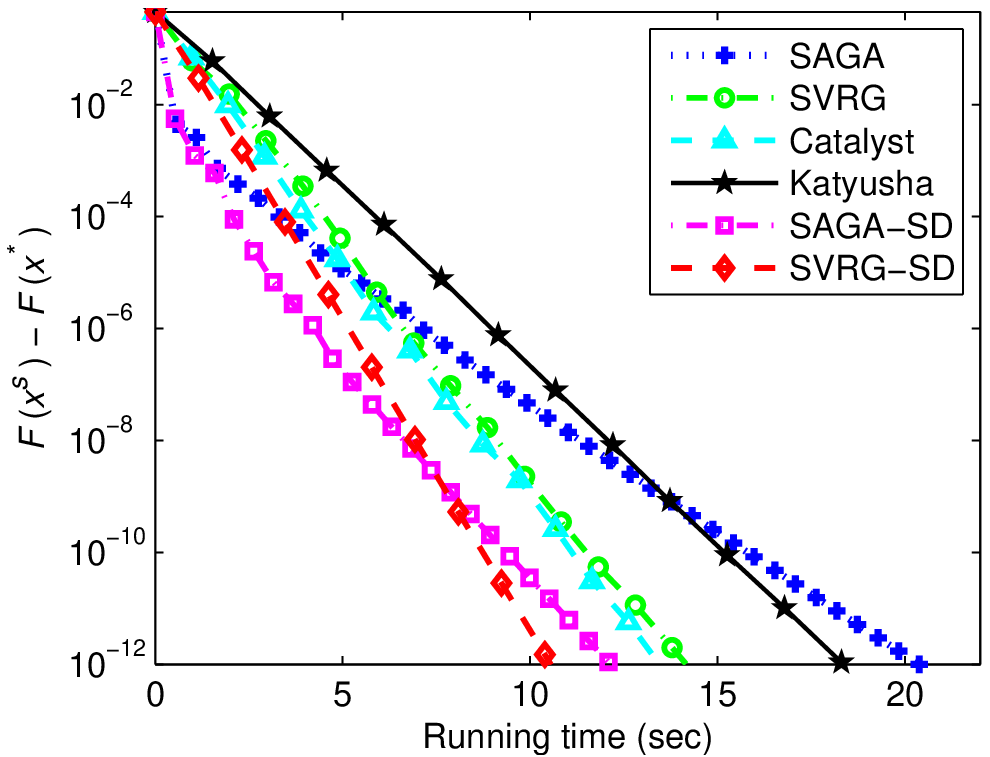}\,
\includegraphics[width=0.326\columnwidth]{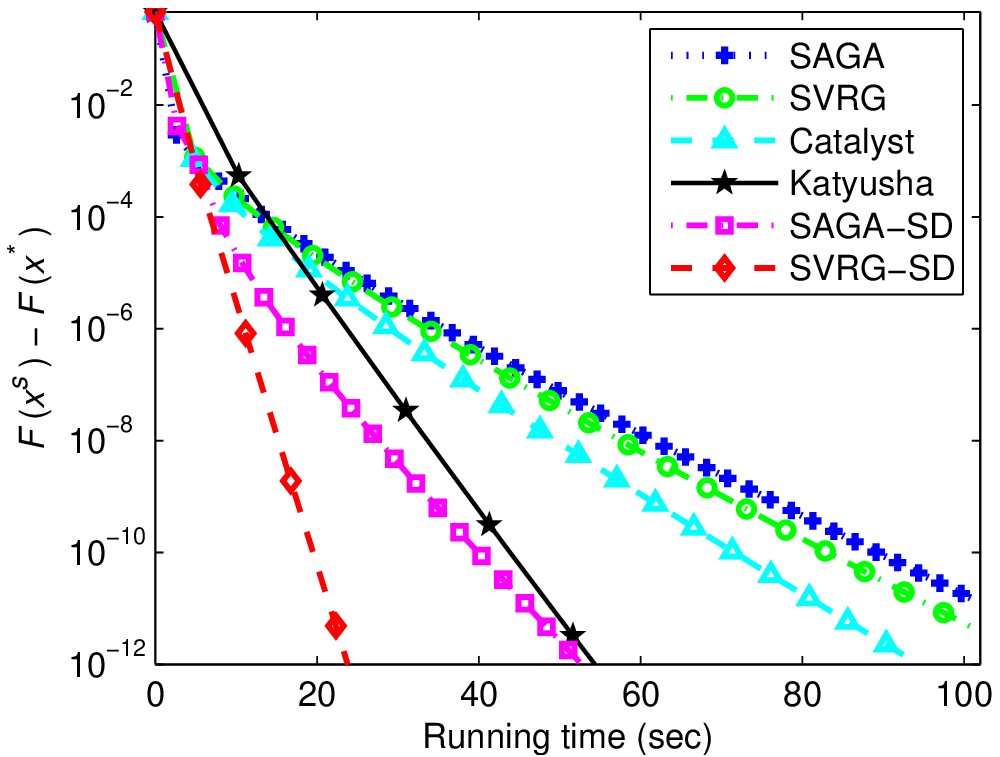}

\subfigure[$\lambda\!=\!10^{-3}$]{\includegraphics[width=0.326\columnwidth]{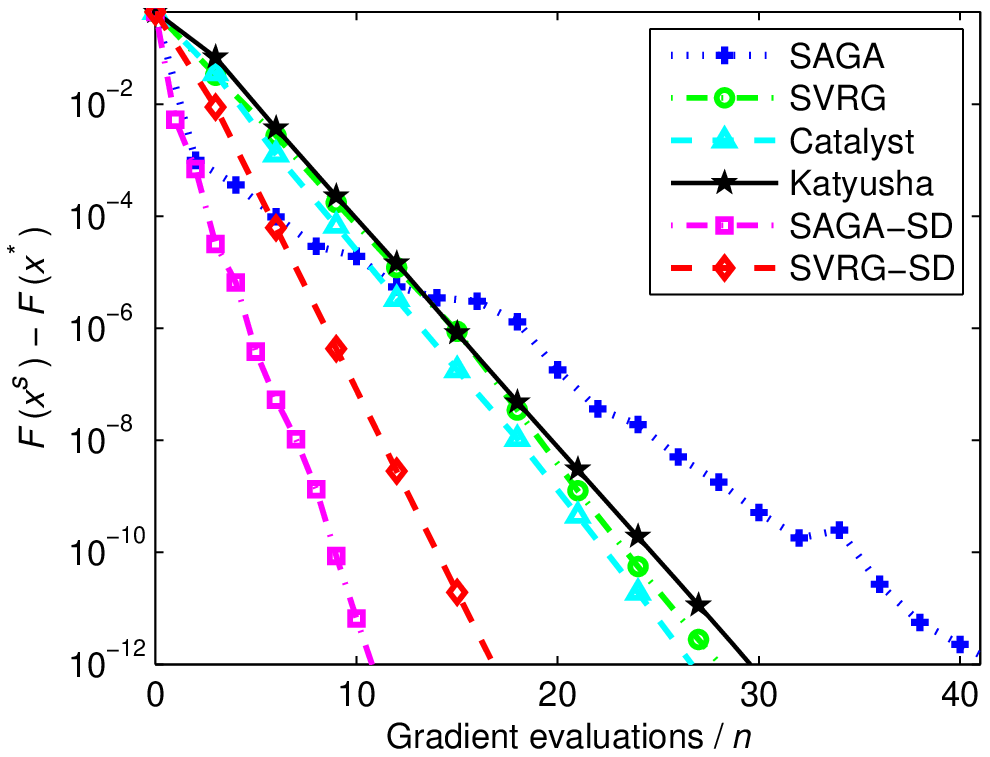}}\,
\subfigure[$\lambda\!=\!10^{-4}$]{\includegraphics[width=0.326\columnwidth]{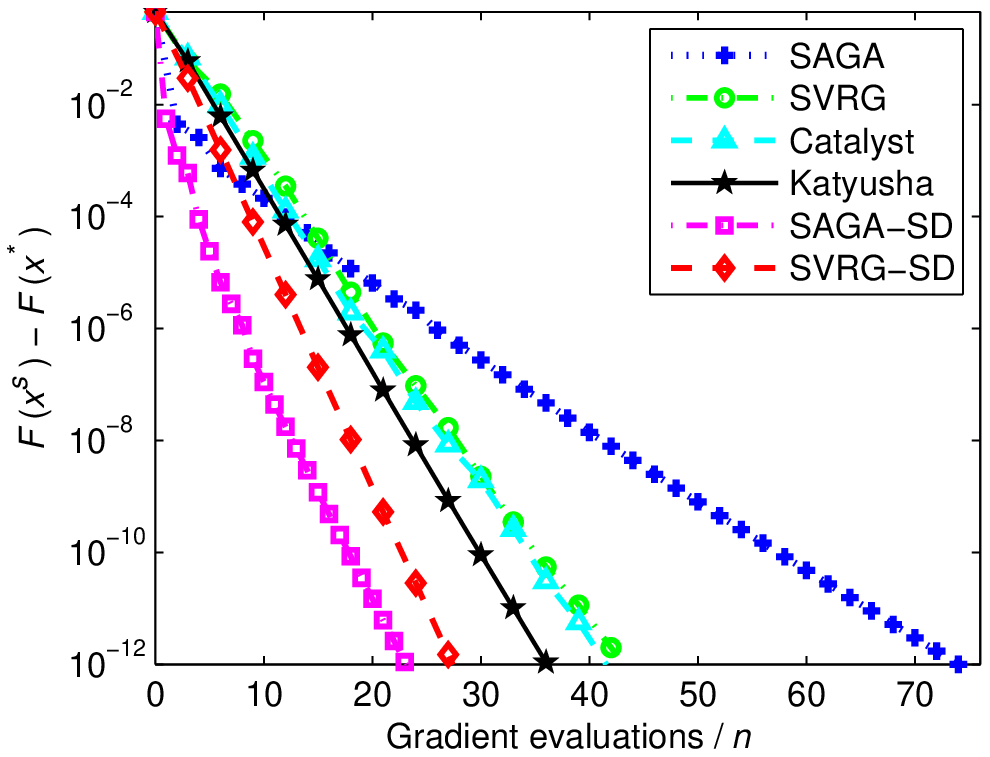}}\,
\subfigure[$\lambda\!=\!10^{-5}$]{\includegraphics[width=0.326\columnwidth]{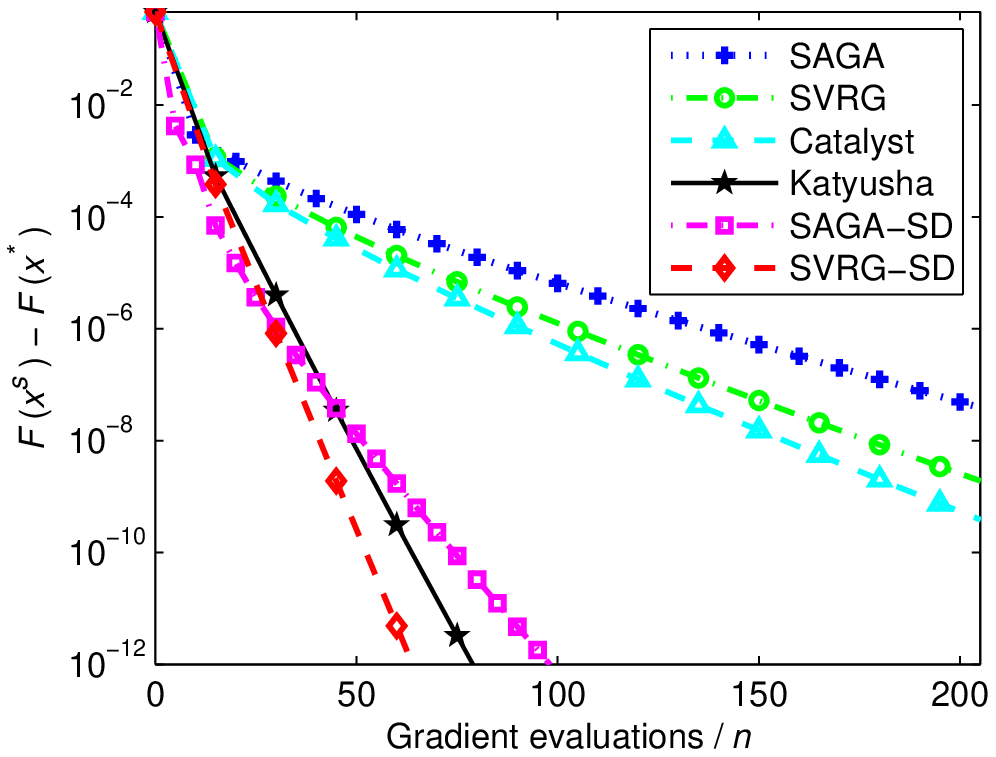}}
\caption{Comparison of different variance reduced SGD methods for solving strongly convex ridge regression problems with different regularization parameters on the Sido0 data set. The vertical axis represents the objective value minus the minimum, and the horizontal axis denotes the running time (top) or the number of effective passes (bottom).}
\label{fig_sim2}
\end{figure}

\section*{Appendix F: More Experimental Results}
In this section, we report more experimental results of SVRG~\cite{johnson:svrg}, SAGA~\cite{defazio:saga}, Catalyst~\cite{lin:vrsg}, Katyusha~\cite{zhu:Katyusha}, SVRG-SD and SAGA-SD for solving strongly convex ridge regression problems with regularization parameters $\lambda\!=\!10^{-4}$ and $\lambda\!=\!10^{-5}$ in Figure~\ref{fig_sim1}, where the horizontal axis denotes the number of effective passes over the data set (evaluating $n$ component gradients, or computing a single full gradient is considered as one effective pass) or the running time (seconds). Figure~\ref{fig_sim2} shows the performance of all these methods for solving ridge regression problems with different regularization parameters on a sparse data set, Sido0, which can be downloaded from the Causality Workbench website{\footnote{\url{http://www.causality.inf.ethz.ch/home.php}}}. From all the results, we can observe that SVRG-SD and SAGA-SD significantly outperform their counterparts: SVRG and SAGA in terms of both number of effective passes and running time. The accelerated method, Catalyst, usually outperforms the non-accelerated methods, SVRG and SAGA. Moreover, SVRG-SD and SAGA-SD achieve at least comparable performance with the best known stochastic method, Katyusha~\cite{zhu:Katyusha}, in terms of number of effective passes. Since SVRG-SD and SAGA-SD have much lower per-iteration complexities than Katyusha, they have more obvious advantage over Katyusha in terms of running time.

\begin{figure}[th]
\centering
\includegraphics[width=0.326\columnwidth]{Fig51}\,
\includegraphics[width=0.326\columnwidth]{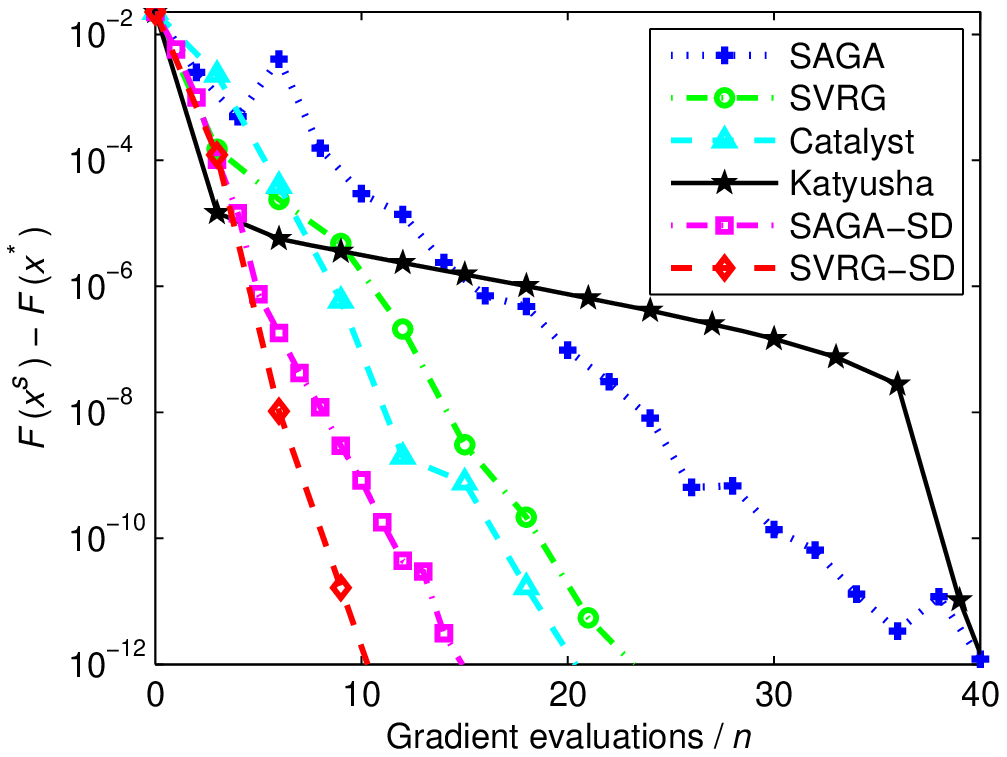}\,
\includegraphics[width=0.326\columnwidth]{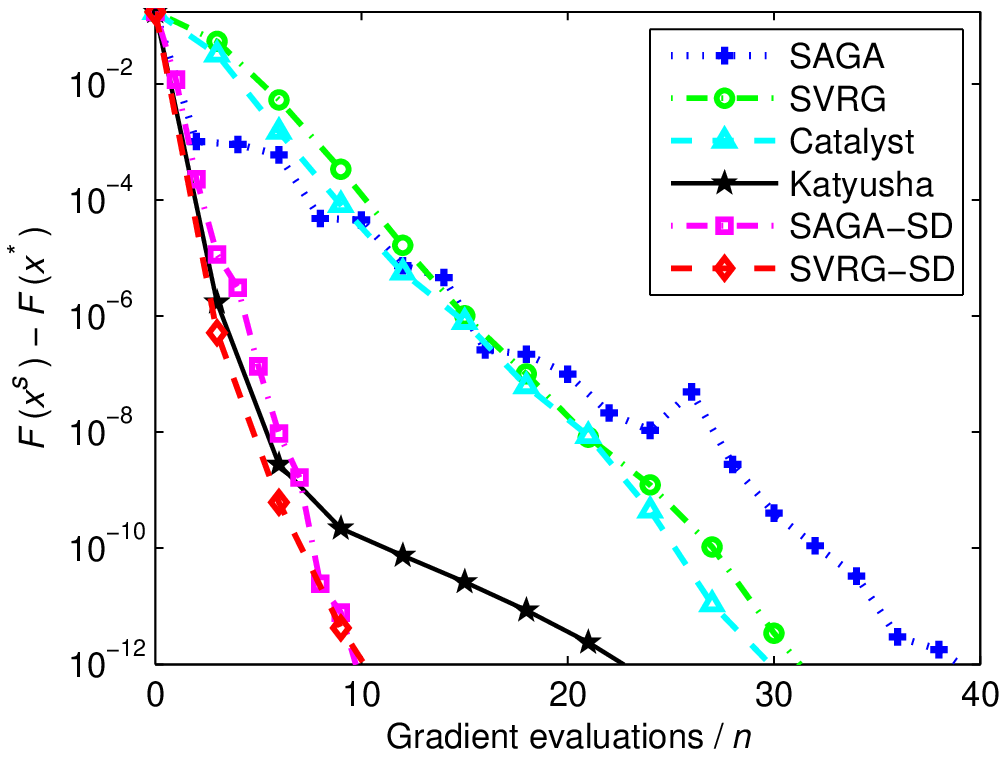}

\subfigure[Ijcnn1, $\lambda\!=\!10^{-4}$]{\includegraphics[width=0.326\columnwidth]{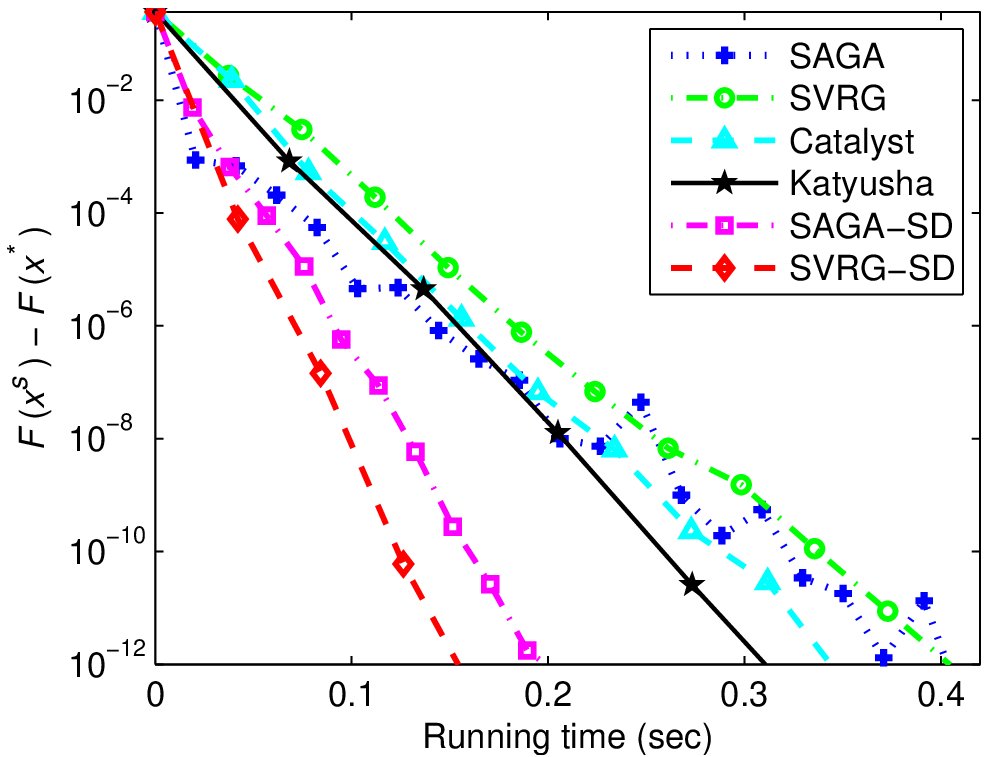}}\,
\subfigure[Covtype, $\lambda\!=\!10^{-4}$]{\includegraphics[width=0.326\columnwidth]{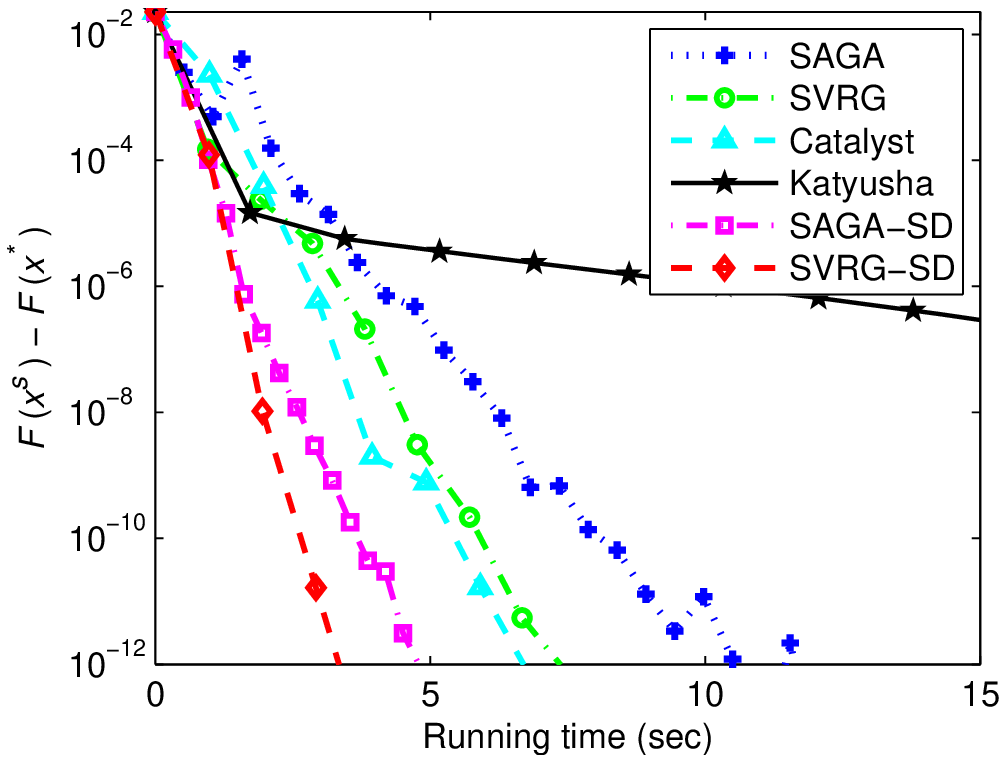}}\,
\subfigure[SUSY, $\lambda\!=\!10^{-4}$]{\includegraphics[width=0.326\columnwidth]{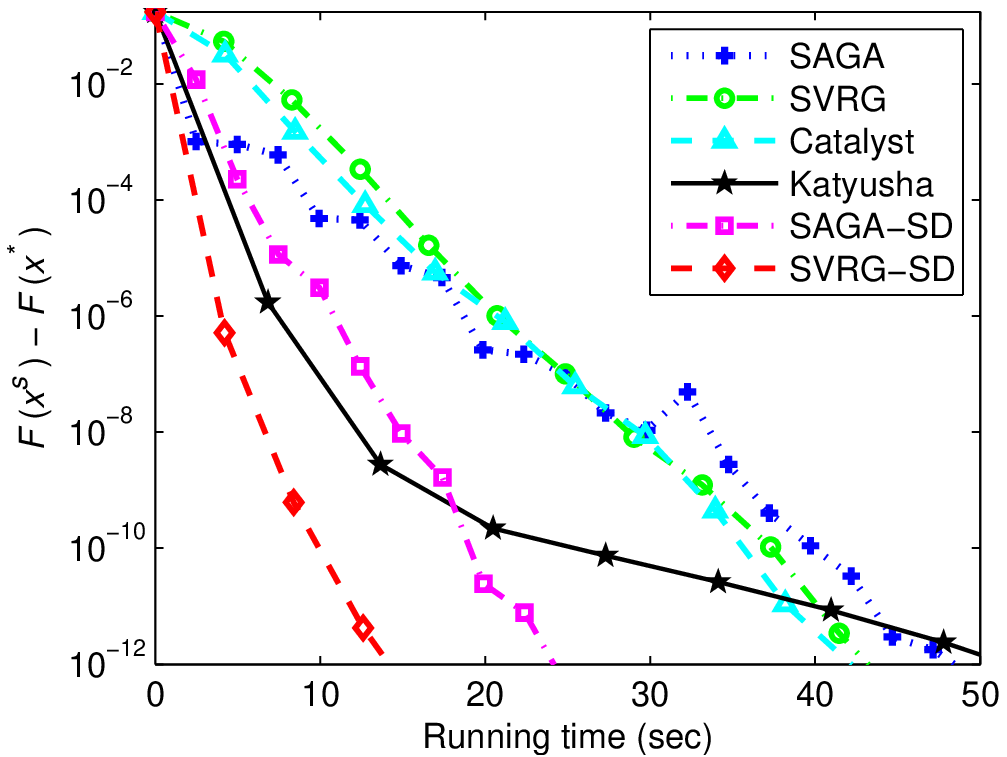}}
\vspace{1mm}

\includegraphics[width=0.326\columnwidth]{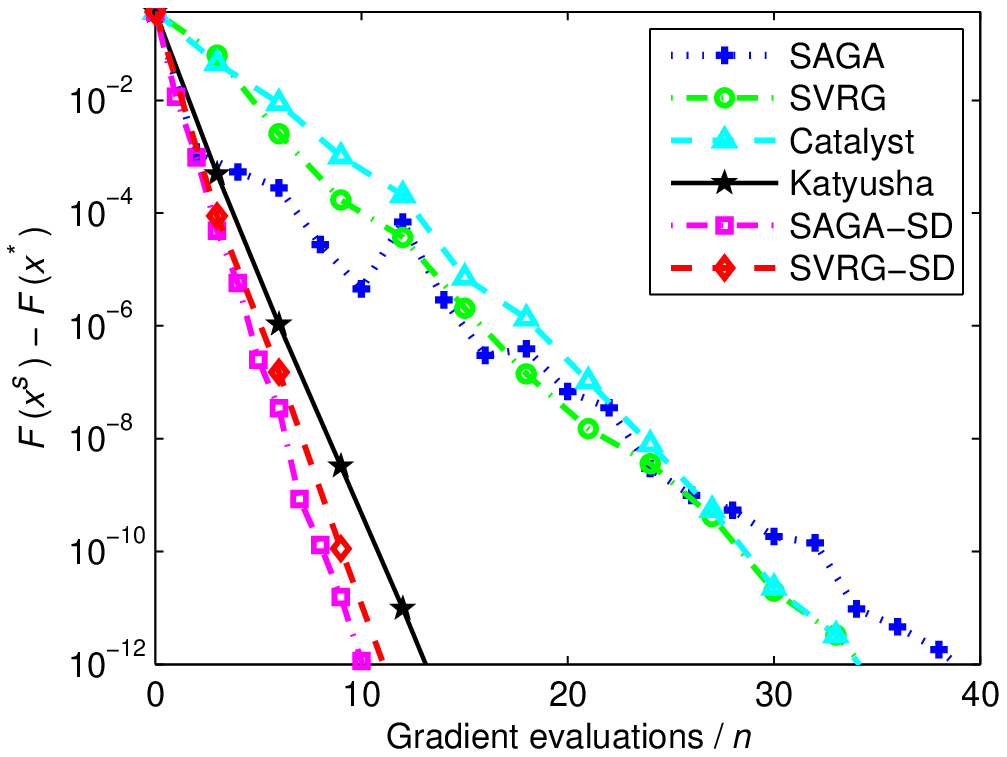}\,
\includegraphics[width=0.326\columnwidth]{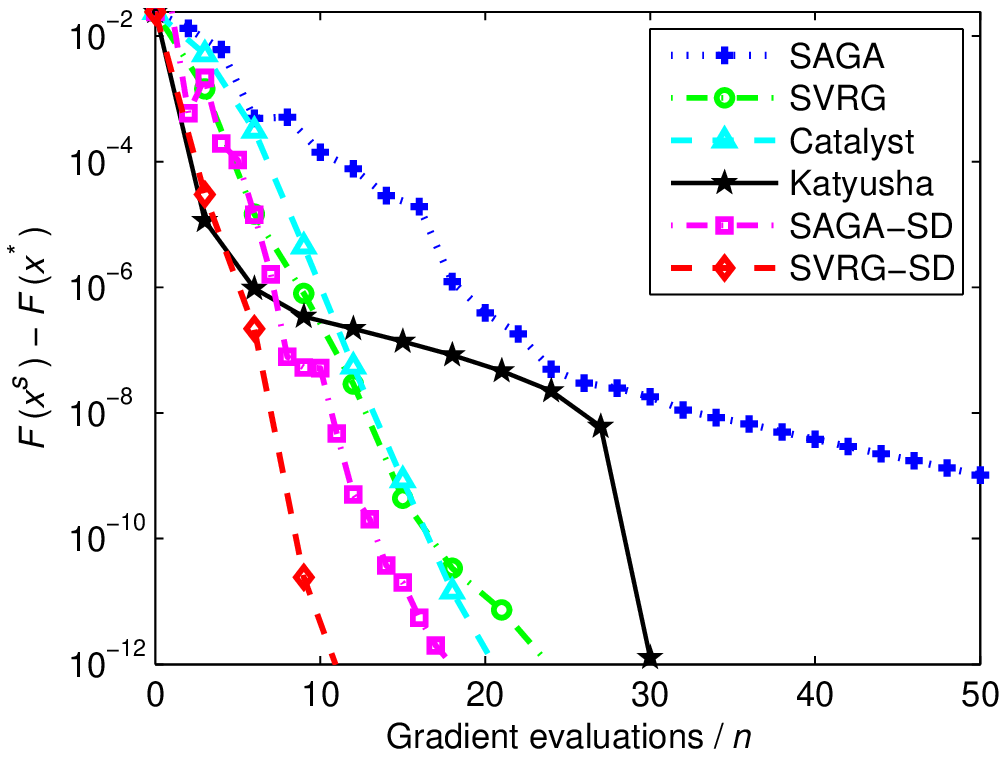}\,
\includegraphics[width=0.326\columnwidth]{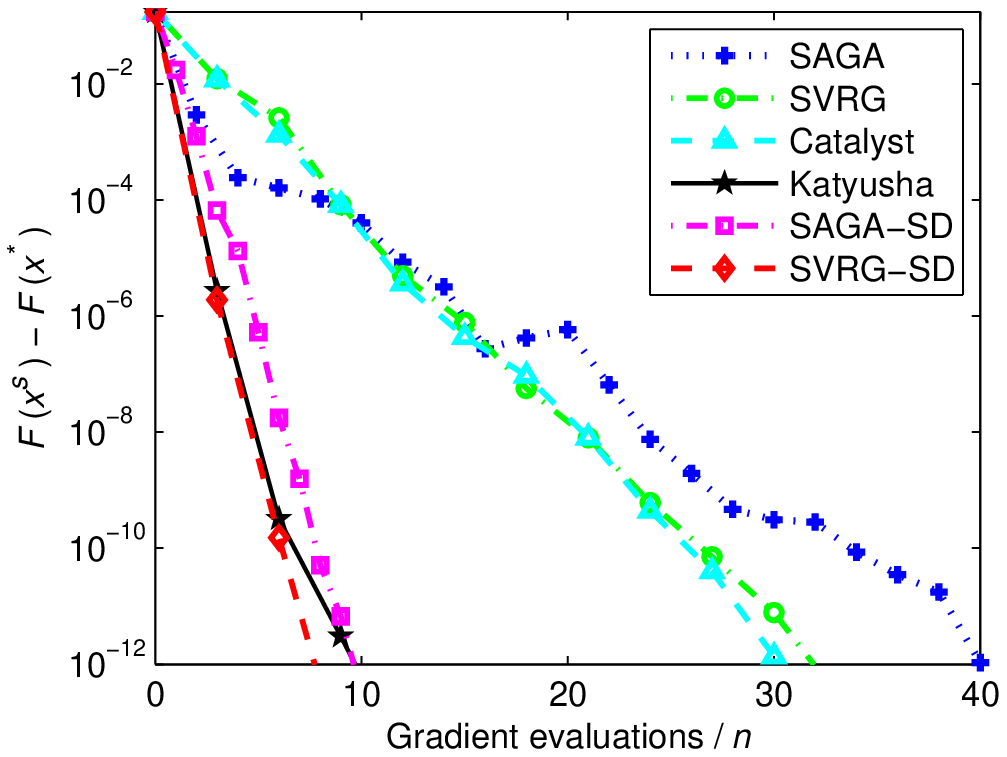}

\subfigure[Ijcnn1, $\lambda\!=\!10^{-5}$]{\includegraphics[width=0.326\columnwidth]{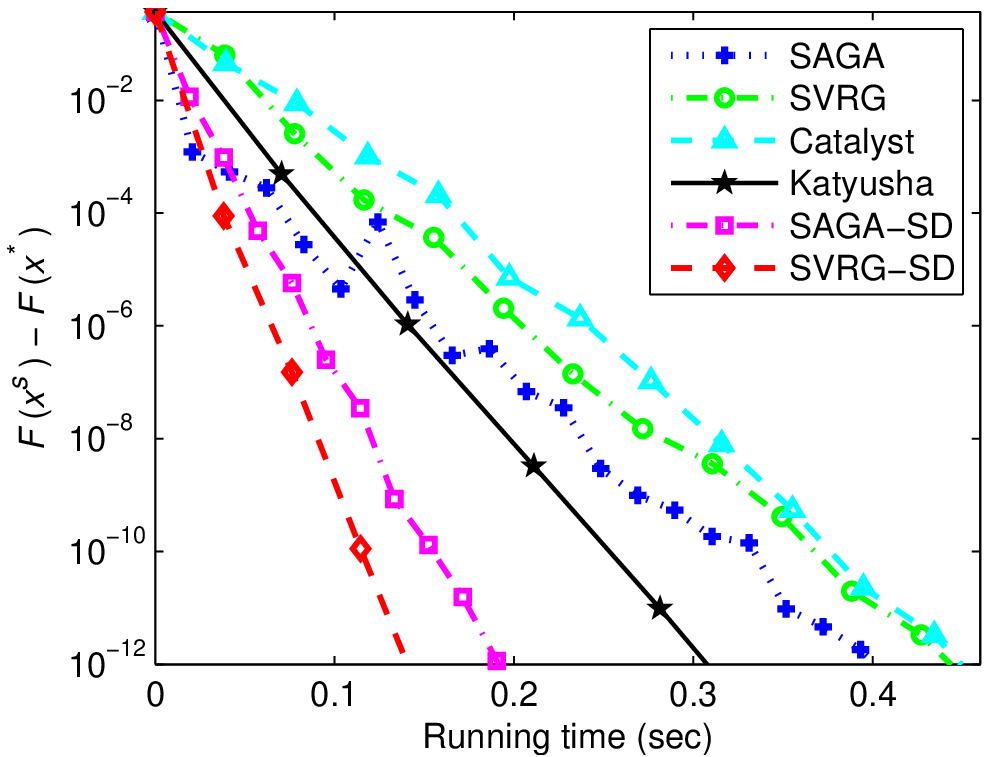}}\,
\subfigure[Covtype, $\lambda\!=\!10^{-5}$]{\includegraphics[width=0.326\columnwidth]{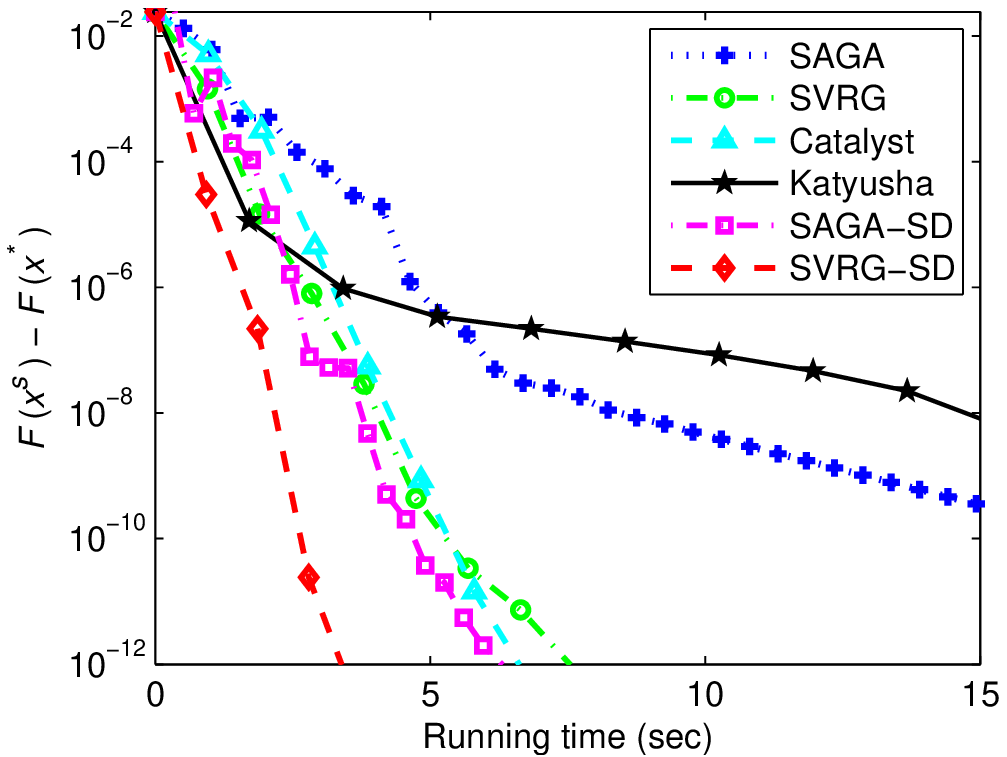}}\,
\subfigure[SUSY, $\lambda\!=\!10^{-5}$]{\includegraphics[width=0.326\columnwidth]{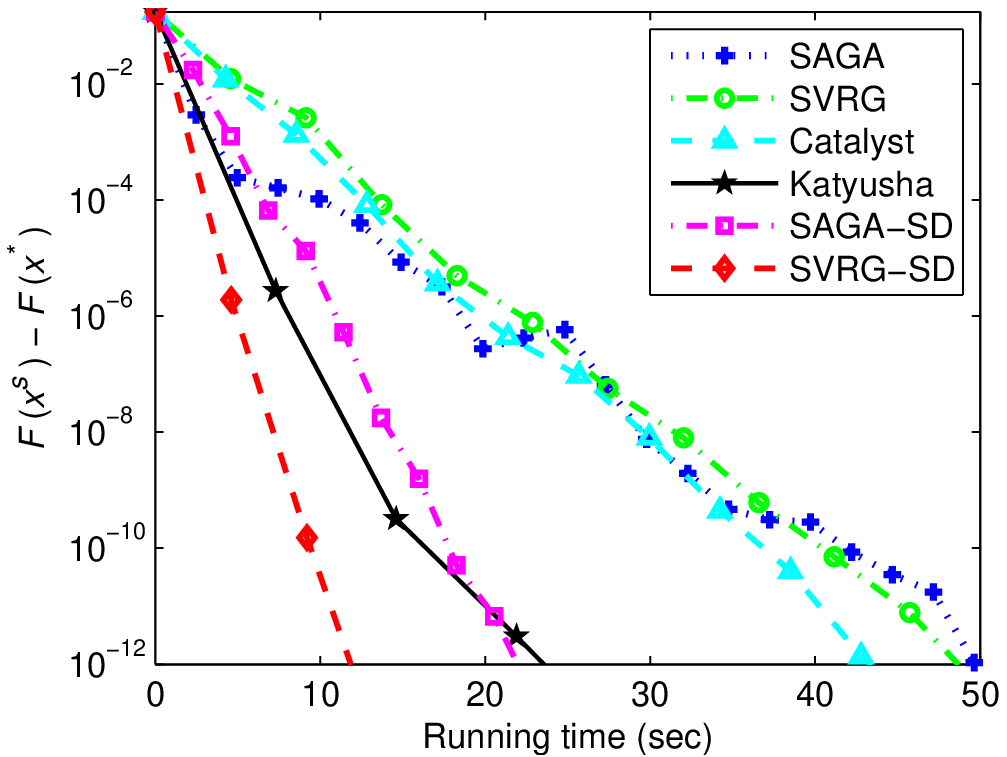}}
\caption{Comparison of different variance reduced SGD methods for solving non-strongly convex Lasso problems. The vertical axis is the objective value minus the minimum, and the horizontal axis denotes the number of effective passes over the data (top) or the running time (seconds, bottom).}
\label{fig_sim3}
\end{figure}

\begin{figure}[th]
\centering
\includegraphics[width=0.326\columnwidth]{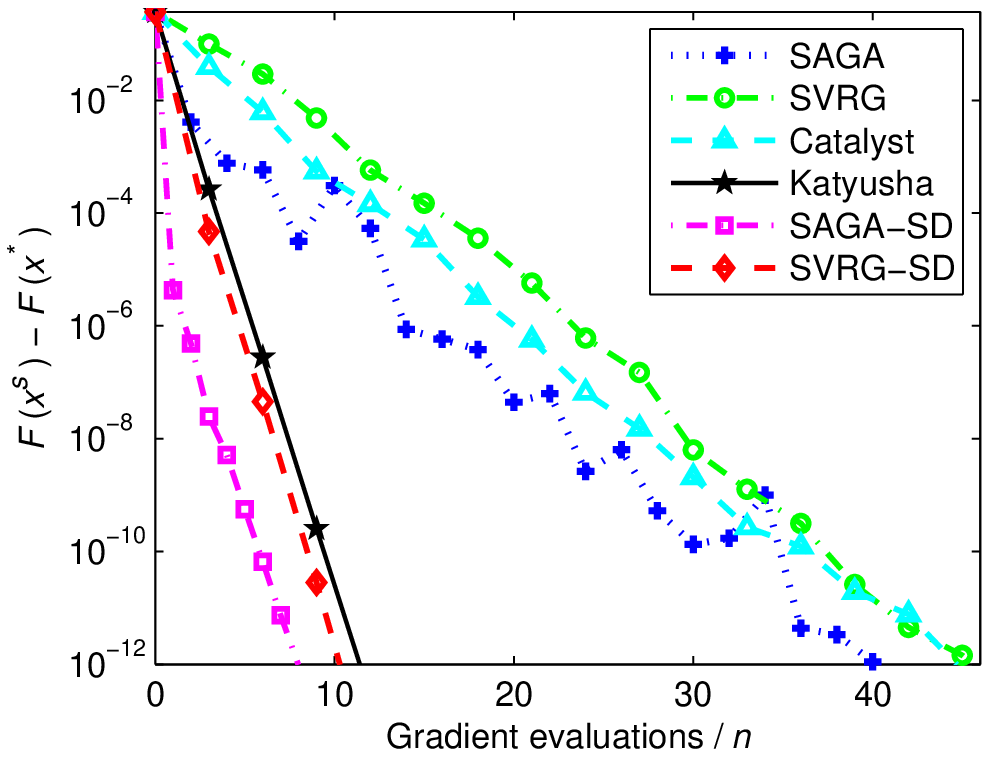}\,
\includegraphics[width=0.326\columnwidth]{Fig76}\,
\includegraphics[width=0.326\columnwidth]{Fig81}

\subfigure[$\lambda_{1}=10^{-5}$ \;and\; $\lambda_{2}=10^{-5}$]{\includegraphics[width=0.326\columnwidth]{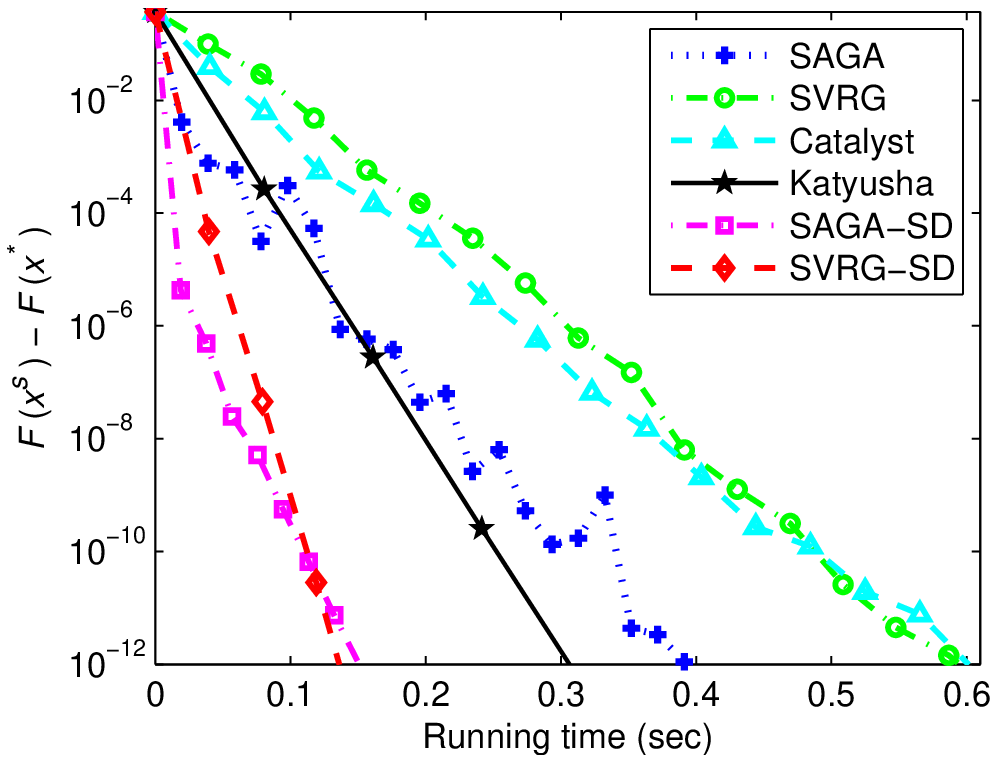}\,\includegraphics[width=0.326\columnwidth]{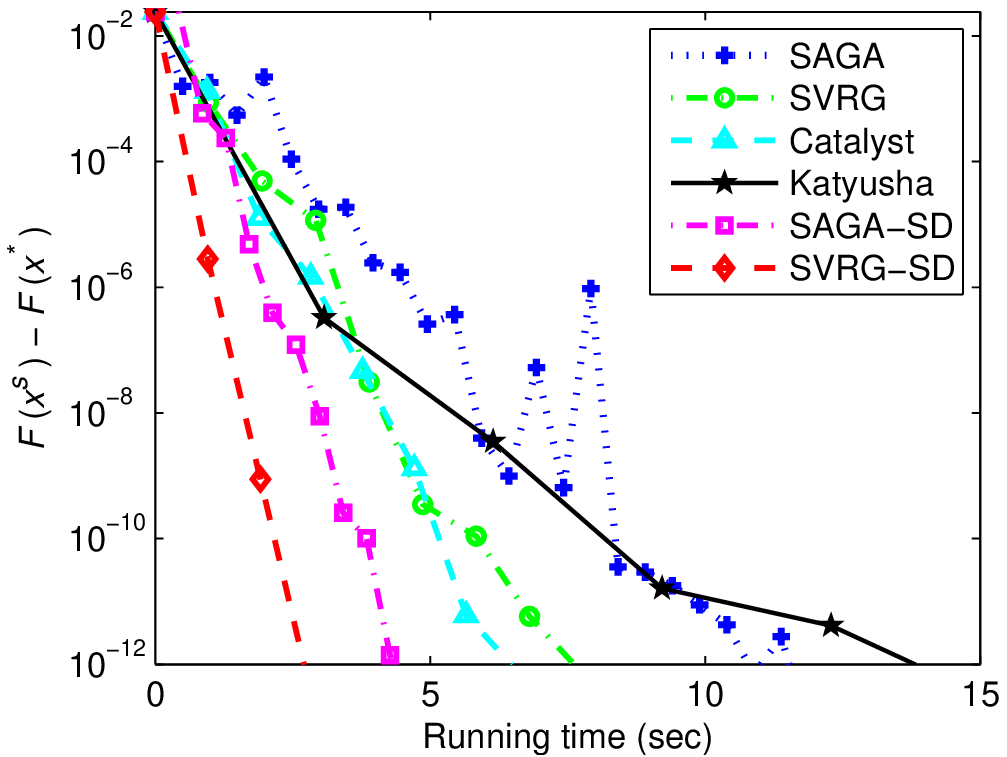}\,\includegraphics[width=0.326\columnwidth]{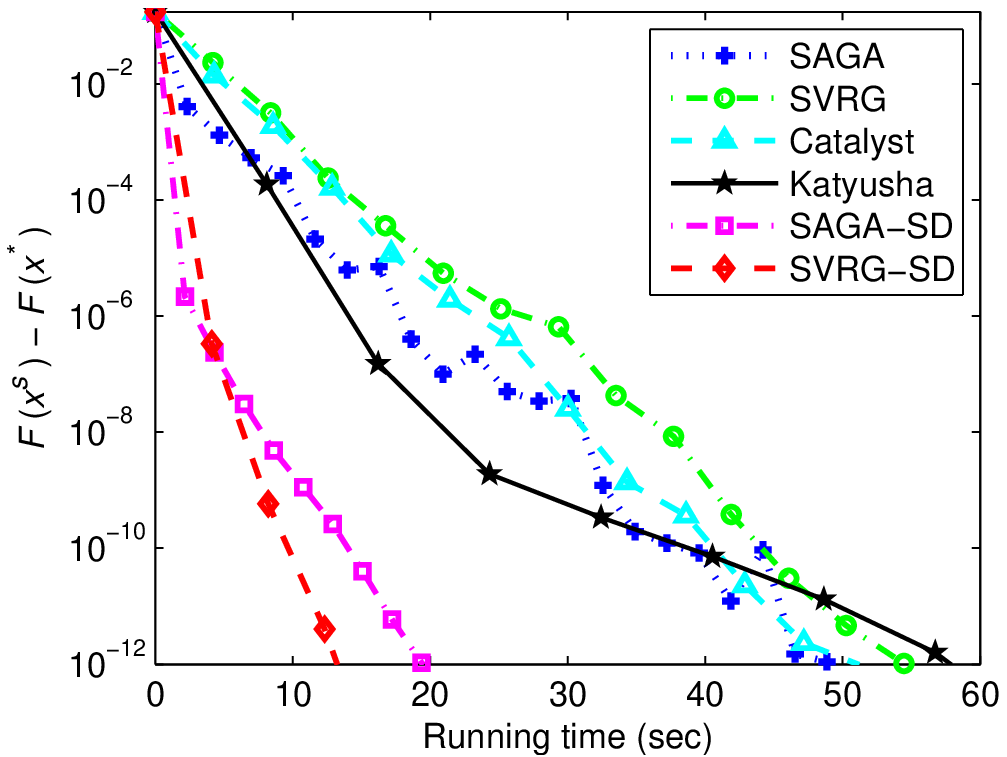}}
\vspace{1mm}

\includegraphics[width=0.326\columnwidth]{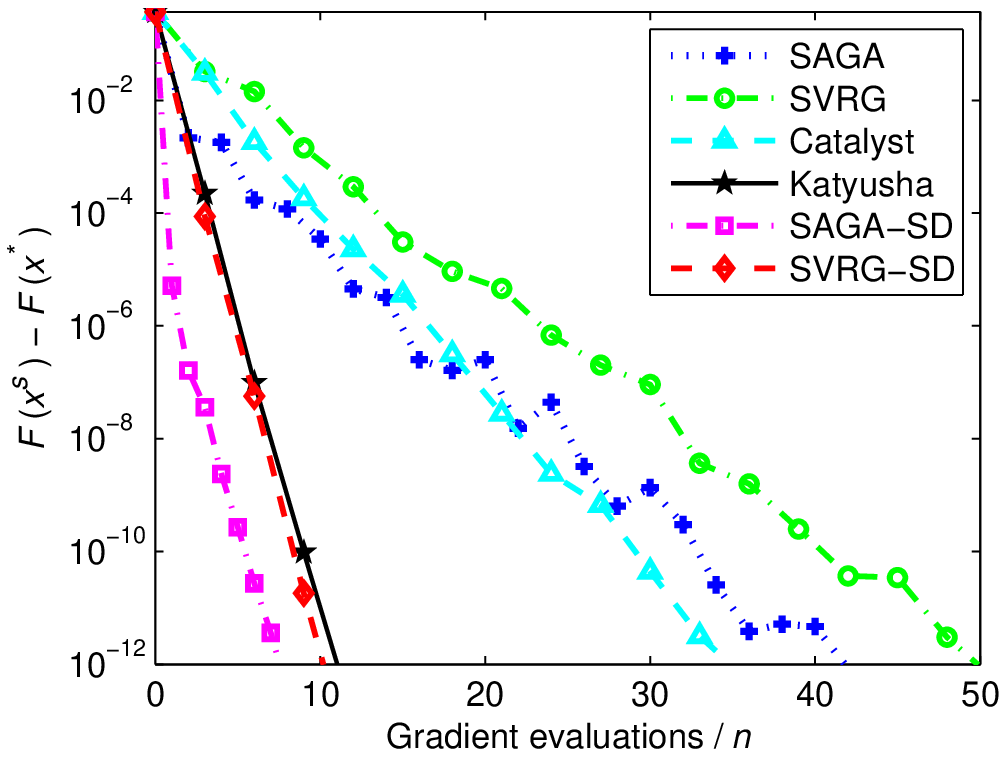}\,
\includegraphics[width=0.326\columnwidth]{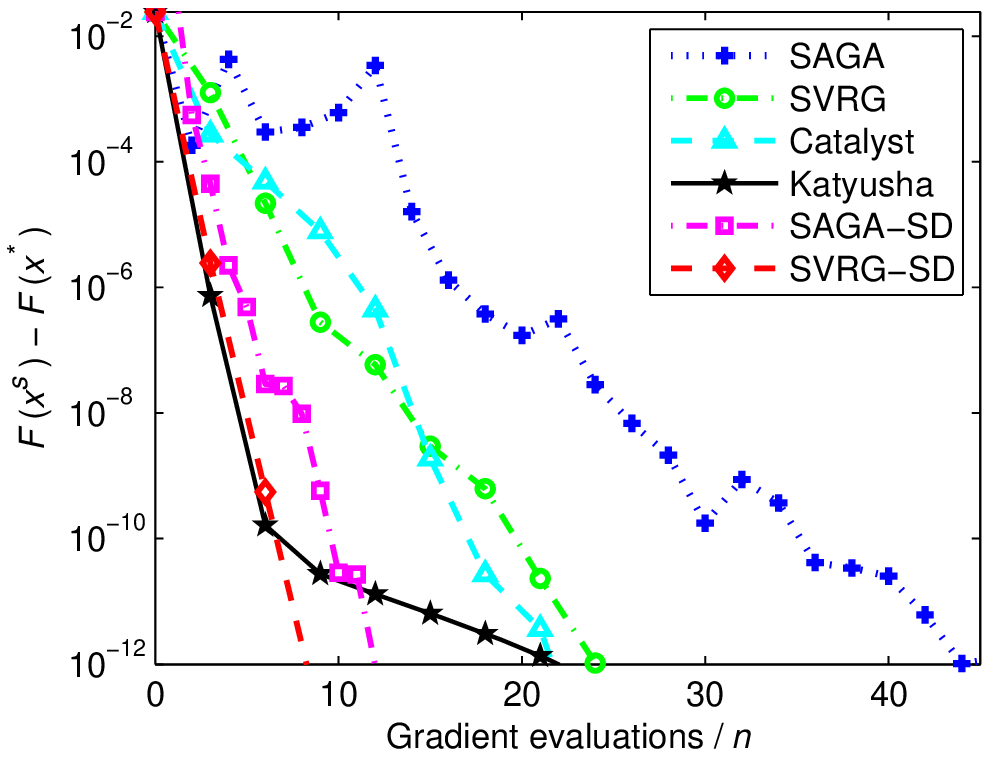}\,
\includegraphics[width=0.326\columnwidth]{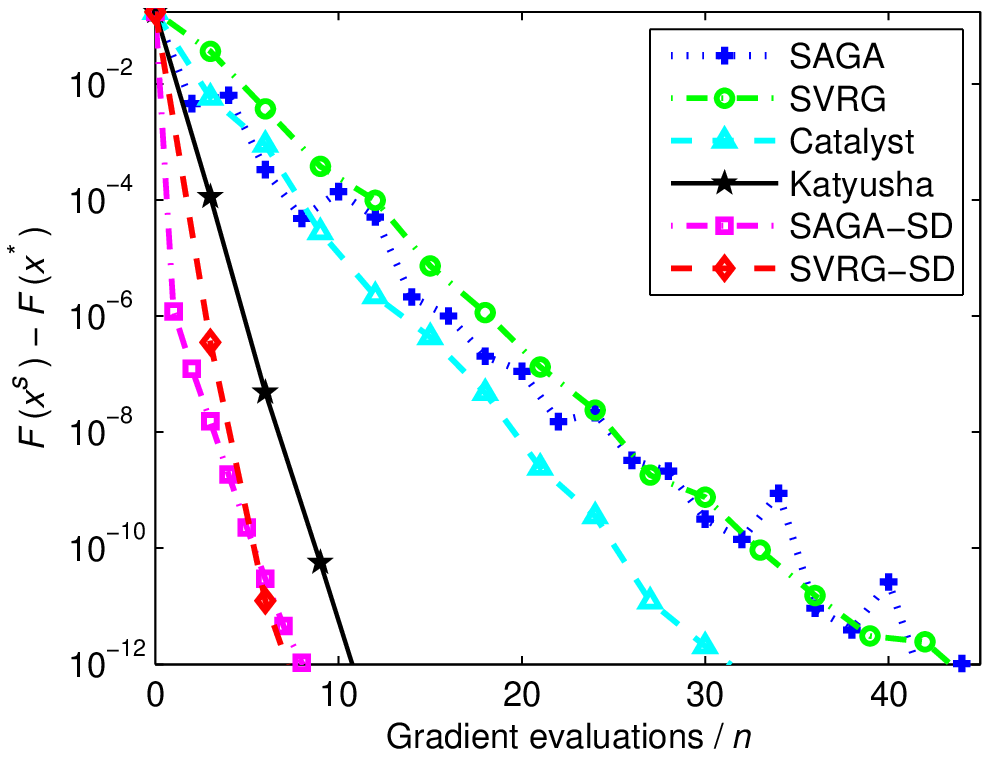}

\subfigure[$\lambda_{1}=10^{-6}$ \;and\; $\lambda_{2}=10^{-5}$]{\includegraphics[width=0.326\columnwidth]{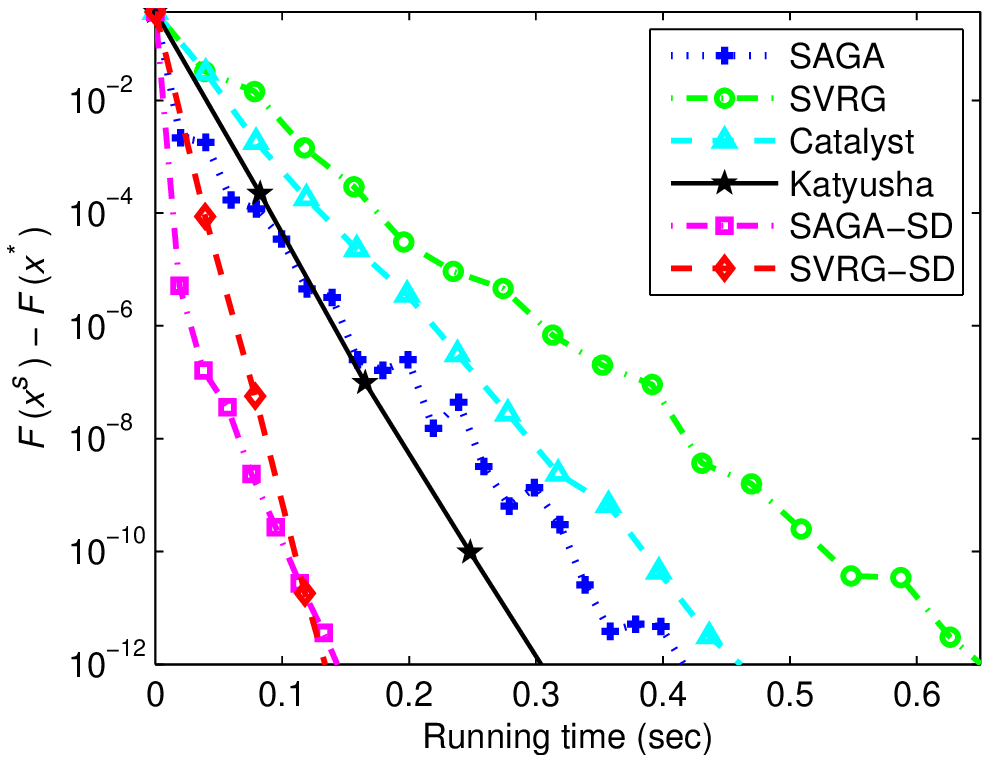}\,\includegraphics[width=0.326\columnwidth]{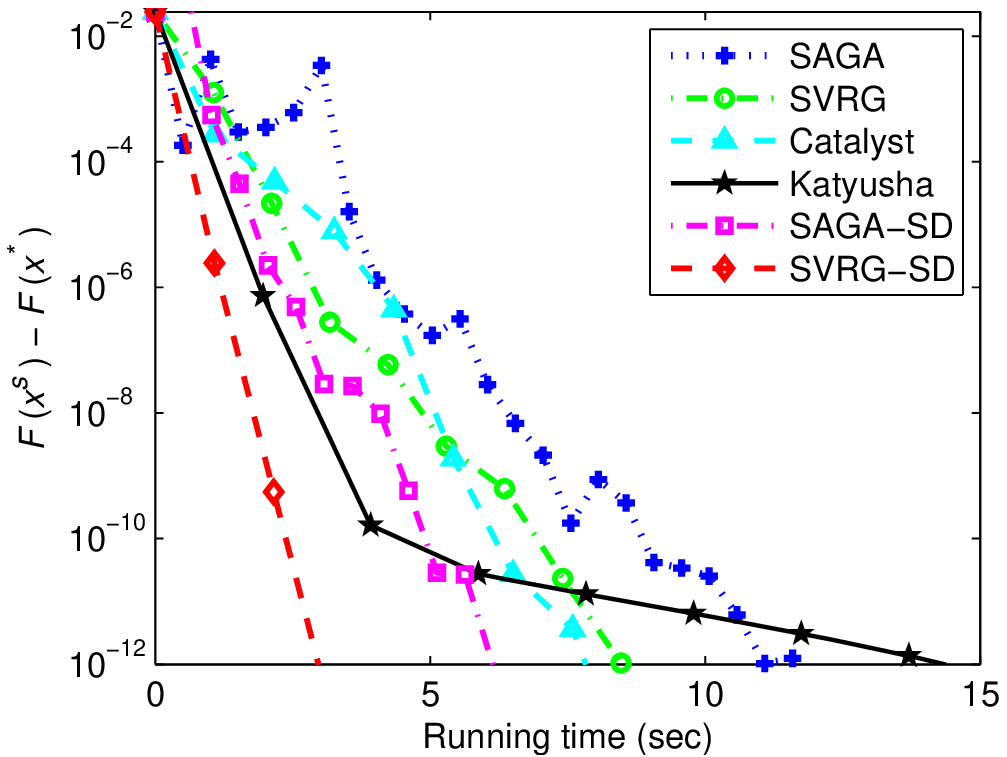}\,\includegraphics[width=0.326\columnwidth]{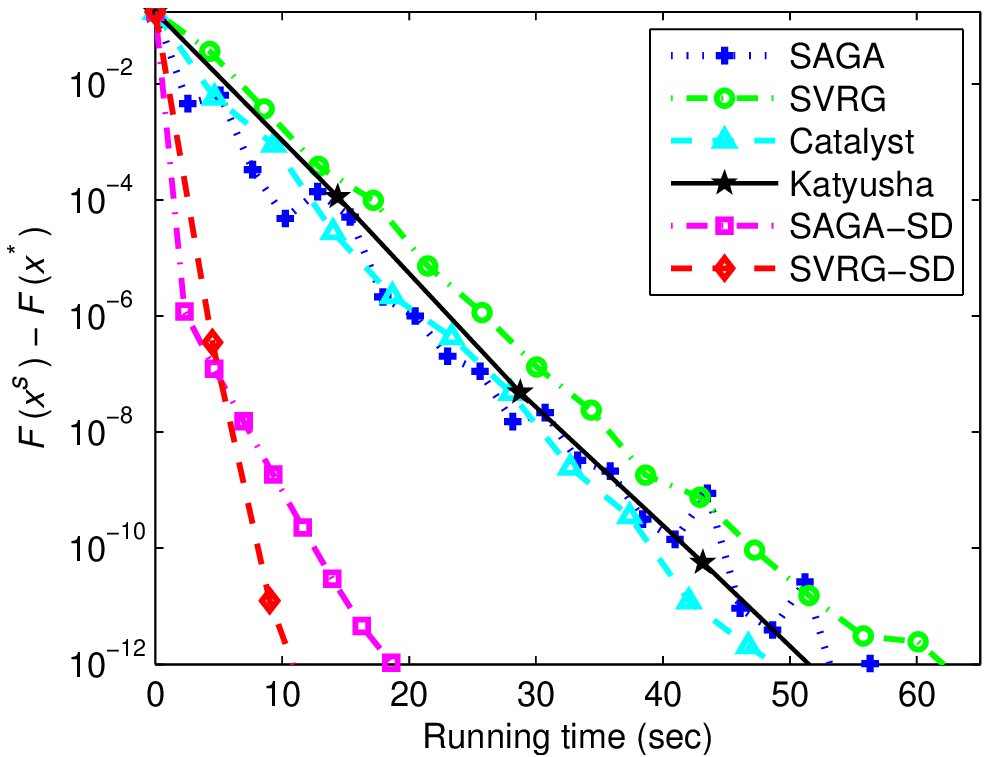}}
\caption{Comparison of different variance reduced SGD methods for solving elastic-net regularized (i.e., $\lambda_{1}\|\!\cdot\!\|_{1}\!+\!\lambda_{2}\|\!\cdot\!\|^2$) Lasso problems on Ijcnn1 (the first column), Covtype (the second column), and SUSY (the last column). The vertical axis is the objective value minus the minimum, and the horizontal axis denotes the number of effective passes over the data (top) or the running time (seconds, bottom).}
\label{fig_sim4}
\end{figure}

Moreover, we report the performance of Prox-SVRG~\cite{xiao:prox-svrg}, SAGA~\cite{defazio:saga}, Catalyst~\cite{lin:vrsg}, Katyusha~\cite{zhu:Katyusha}, SVRG-SD and SAGA-SD for solving Lasso and elastic-net regularized Lasso problems with different regularization parameters in Figures~\ref{fig_sim3} and \ref{fig_sim4}, respectively, from which we can see that SVRG-SD and SAGA-SD also achieve much faster convergence speed than their counterparts: Prox-SVRG and SAGA, respectively. In particular, they also have comparable or better performance than the accelerated VR-SGD methods, Catalyst and Katyusha.

\begin{algorithm}[t]
\caption{SVRG-I and SVRG-II}
\label{alg3}
\renewcommand{\algorithmicrequire}{\textbf{Input:}}
\renewcommand{\algorithmicensure}{\textbf{Initialize:}}
\renewcommand{\algorithmicoutput}{\textbf{Output:}}
\begin{algorithmic}[1]
\REQUIRE The number of epochs $S$, the number of iterations $m$ per epoch, and step size $\eta$.\\
\ENSURE $\widetilde{x}^{0}$.\\
\FOR{$s=1,2,\ldots,S$}
\STATE {$\widetilde{\mu}^{s-1}=\frac{1}{n}\!\sum^{n}_{i=1}\!\nabla\!f_{i}(\widetilde{x}^{s-1})$,\; $x_{0}=\widetilde{x}^{s-1}$;}
\FOR{$k=1,2,\ldots,m$}
\STATE {Pick $i_{k}$ uniformly at random from $[n]$;}
\STATE {$\widetilde{\nabla}\! f_{i_{k}}(x_{k-1})=\nabla\! f_{i_{k}}(x_{k-1})-\nabla\! f_{i_{k}}(\widetilde{x}^{s-1})+\widetilde{\mu}^{s-1}$;}
\STATE {$x_{k}=x_{k-1}-\eta\left[\widetilde{\nabla}\!f_{i_{k}}(x_{k-1})+\nabla r(x_{k-1})\right]$;}
\ENDFOR
\STATE {Option I:\,$\widetilde{x}^{s}=x_{m}$;}
\STATE {Option II:\,$\widetilde{x}^{s}=\frac{1}{m}\!\sum^{m}_{k=1}x_{k}$;}
\ENDFOR
\OUTPUT {$\widetilde{x}^{S}$}
\end{algorithmic}
\end{algorithm}

\begin{algorithm}[t]
\caption{SVRG-SDI}
\label{alg4}
\renewcommand{\algorithmicrequire}{\textbf{Input:}}
\renewcommand{\algorithmicensure}{\textbf{Initialize:}}
\renewcommand{\algorithmicoutput}{\textbf{Output:}}
\begin{algorithmic}[1]
\REQUIRE The number of epochs $S$, the number of iterations $m$ per epoch, and step size $\eta$.\\
\ENSURE $x_{0}=\widetilde{x}^{0}$, $\sigma=0.618-0.382/[1+\exp(-\log(6\lambda)-12)]$ for Case of SC, \,or\, $\sigma=1/(S\!+\!3)$ for Case of NSC.\\
\FOR{$s=1,2,\ldots,S$}
\STATE {$\widetilde{\mu}^{s-1}=\frac{1}{n}\!\sum^{n}_{i=1}\!\nabla\!f_{i}(\widetilde{x}^{s-1})$;}
\FOR{$k=1,2,\ldots,m$}
\STATE {Pick $i_{k}$ uniformly at random from $[n]$;}
\STATE {$\widetilde{\nabla}\! f_{i_{k}}(x_{k-1})=\nabla\! f_{i_{k}}(x_{k-1})-\nabla\! f_{i_{k}}(\widetilde{x}^{s-1})+\widetilde{\mu}^{s-1}$;}
\STATE {$x_{k}=x_{k-1}-\eta\left[\widetilde{\nabla}\!f_{i_{k}}(x_{k-1})+\nabla r(x_{k-1})\right]$\; or\; $x_{k}=\textrm{prox}^{\,r}_{\,\eta}\left(x_{k-1}-\eta\widetilde{\nabla}\!f_{i_{k}}(x_{k-1})\right)$;}
\ENDFOR
\STATE {$\widetilde{x}^{s}=\frac{1}{m}\!\sum^{m}_{k=1}\left[x_{k}+(1\!-\!\sigma)(x_{k}-x_{k-1})\right]$;}
\STATE {$x_{0}=x_{m}$;}
\ENDFOR
\OUTPUT {$\widetilde{x}^{S}$}
\end{algorithmic}
\end{algorithm}

\begin{algorithm}[t]
\caption{SAGA-SDI}
\label{alg5}
\renewcommand{\algorithmicrequire}{\textbf{Input:}}
\renewcommand{\algorithmicensure}{\textbf{Initialize:}}
\renewcommand{\algorithmicoutput}{\textbf{Output:}}
\begin{algorithmic}[1]
\REQUIRE The number of iterations $K\!=\!S\!\ast\!m$, and step size $\eta$.\\
\ENSURE $x_{0}$, $\sigma=0.5-0.5/[1+\exp(-\log\lambda-12)]$ for the cases of SC and NSC.
\FOR{$k=1,\ldots,K$}
\STATE {Pick $i_{k}$ uniformly at random from $[n]$;}
\STATE {Take $\phi^{k}_{i_{k}}\!\!=\!x_{k-1}$, and store $\nabla\! f_{i_{k}}\!(\phi^{k}_{i_{k}})$ in the table and all other entries in the table remain unchanged;}
\STATE {$\widetilde{\nabla}\! f_{i_{k}}(x_{k-1})=\nabla\! f_{i_{k}}(\phi^{k}_{i_{k}})-\nabla \! f_{i_{k}}(\phi^{k-1}_{i_{k}})+\frac{1}{n}\sum^{n}_{j=1}\!\nabla\! f_{j}(\phi^{k-1}_{j})$;}
\STATE {$x_{k}=x_{k-1}-\eta\left[\widetilde{\nabla}\!f_{i_{k}}(x_{k-1})+\nabla r(x_{k-1})\right]$\; or\; $x_{k}=\textrm{prox}^{\,r}_{\,\eta}\left(x_{k-1}-\eta\widetilde{\nabla}\!f_{i_{k}}(x_{k-1})\right)$;}
\STATE {$x_{k}=x_{k}+(1\!-\!\sigma)(x_{k}-x_{k-1})$;}
\ENDFOR
\OUTPUT $x_{K}$
\end{algorithmic}
\end{algorithm}

\section*{Appendix G: Pseudo-Codes of More Algorithms}
Recall that the main difference of the original SVRG~\cite{johnson:svrg} (denoted by SVRG-I, see Algorithm~\ref{alg3} with Option I for completeness) and its variant, SVRG-II (see Algorithm~\ref{alg3} with Option II) in~\cite{johnson:svrg} is that the former uses the last iterate of the previous epoch as the snapshot point $\widetilde{x}^{s}$, while in the latter, $\widetilde{x}^{s}$ is the average point of the previous epoch, which has been successfully used in \cite{zhu:Katyusha,zhu:univr,xiao:prox-svrg}. In this paper, we observed the following interesting phenomena: When the regularization parameter $\lambda$ is relatively large, e.g., $\lambda\!=\!10^{-3}$, SVRG-II converges significantly faster than SVRG-I, as shown in Figure \ref{fig_sim6} and also suggested in \cite{zhu:Katyusha,zhu:univr,xiao:prox-svrg}; whereas SVRG-I significantly outperforms SVRG-II when $\lambda$ is relatively small, e.g., $\lambda\!=\!10^{-7}$.

As a by-product of SVRG-SD and motivated by the above observations, we propose a momentum acceleration variant of the original SVRG, as outlined in Algorithm \ref{alg4}, which is called as SVRG-SDI, and can be viewed as a special case of SVRG-SD when $\theta_{k}\!\equiv\!1$ (i.e., without the proposed sufficient decrease technique). Note that SVRG-SDI, as well as SVRG-SD, can use much larger learning rates than SVRG-I and SVRG-II, e.g., $\eta\!=\!1.0$ for SVRG-SD and SVRG-SDI vs.\ $\eta\!=\!0.4$ for SVRG-I and SVRG-II. From the results in Figure~\ref{fig_sim6}, where we vary the values of the regularization parameter $\lambda$ from $10^{-3}$ to $10^{-8}$, it is clear that our SVRG-SDI method achieves significantly better performance than both SVRG-I and SVRG-II in most cases, and is comparable to the best known method, Katyusha~\cite{zhu:Katyusha}. Note that Katyusha fails to converge when the regularization parameter is greater than or equal to $10^{-3}$. Moreover, our SVRG-SD method often outperforms the other methods (including Katyusha and SVRG-SDI). Especially in the cases when $\lambda$ is relatively small, e.g., $\lambda\!=\!10^{-6}$ and $\lambda\!=\!10^{-7}$, SVRG-SD converges significantly faster than Katyusha and SVRG-SDI, which further verifies the effectiveness of our sufficient decrease technique for stochastic optimization. Our sufficient decrease technique in SVRG-SD and SVRG-SDI naturally generalizes to other stochastic gradient estimators as in~\cite{defazio:saga,nguyen:srg,roux:sag} as well, e.g., Algorithm~\ref{alg5} (called SAGA-SDI) for the SAGA estimator~\cite{defazio:saga} in (9).

\begin{figure}[th]
\centering
\subfigure[$\lambda\!=\!10^{-3}$]{\includegraphics[width=0.324\columnwidth]{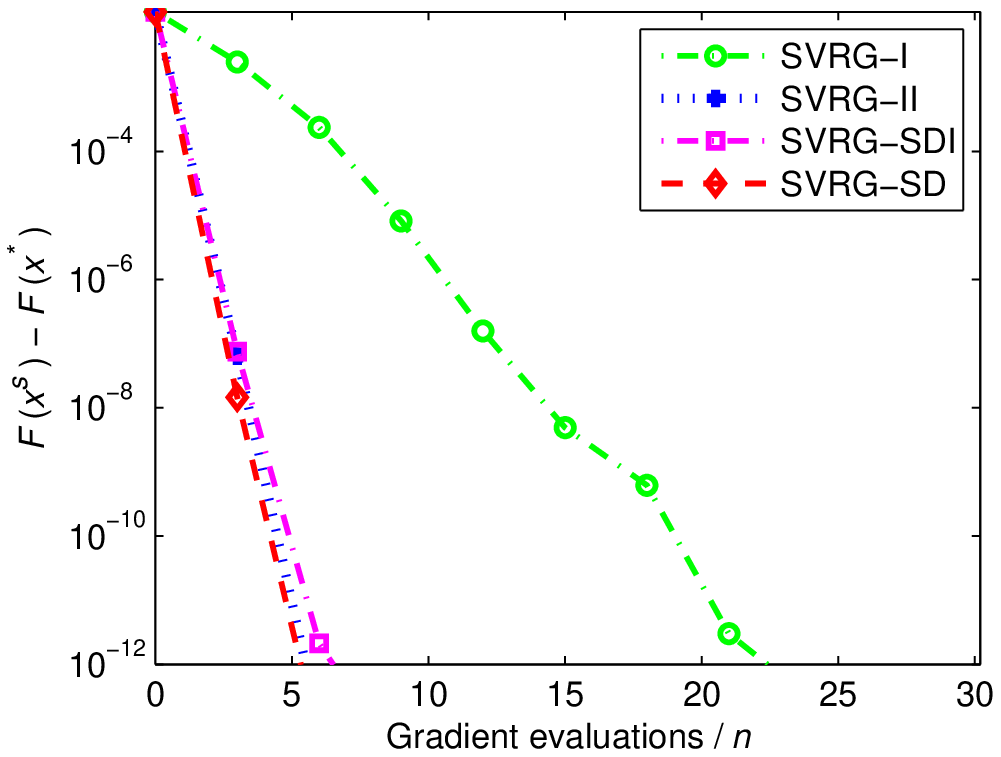}}\,
\subfigure[$\lambda\!=\!10^{-4}$]{\includegraphics[width=0.324\columnwidth]{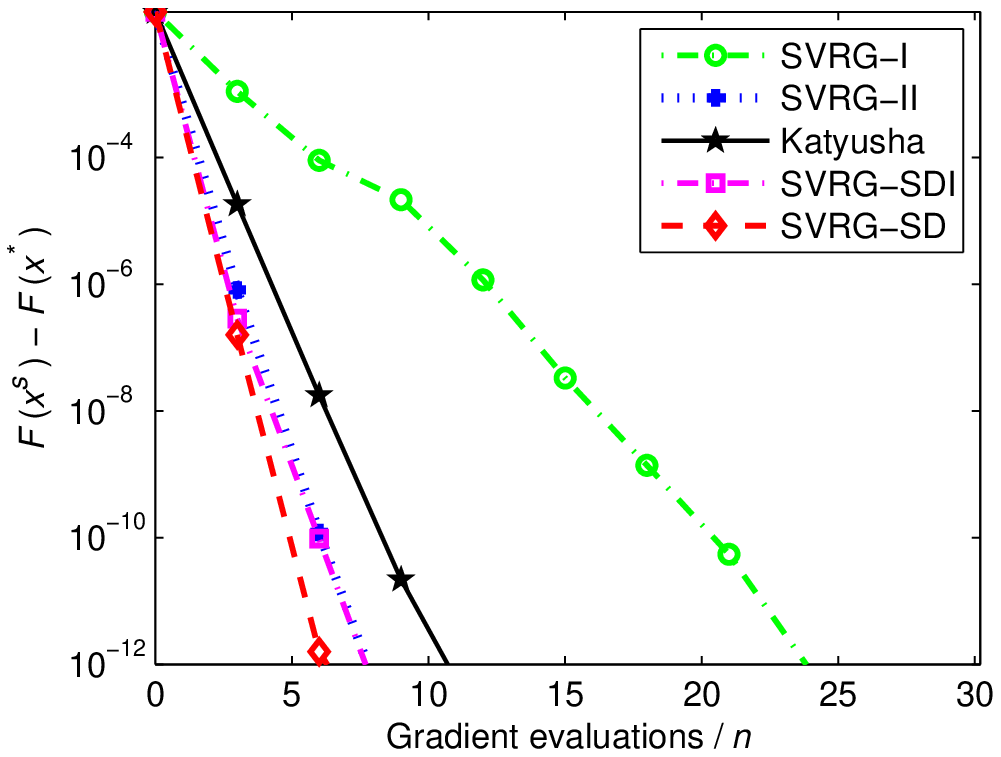}}\,
\subfigure[$\lambda\!=\!10^{-5}$]{\includegraphics[width=0.324\columnwidth]{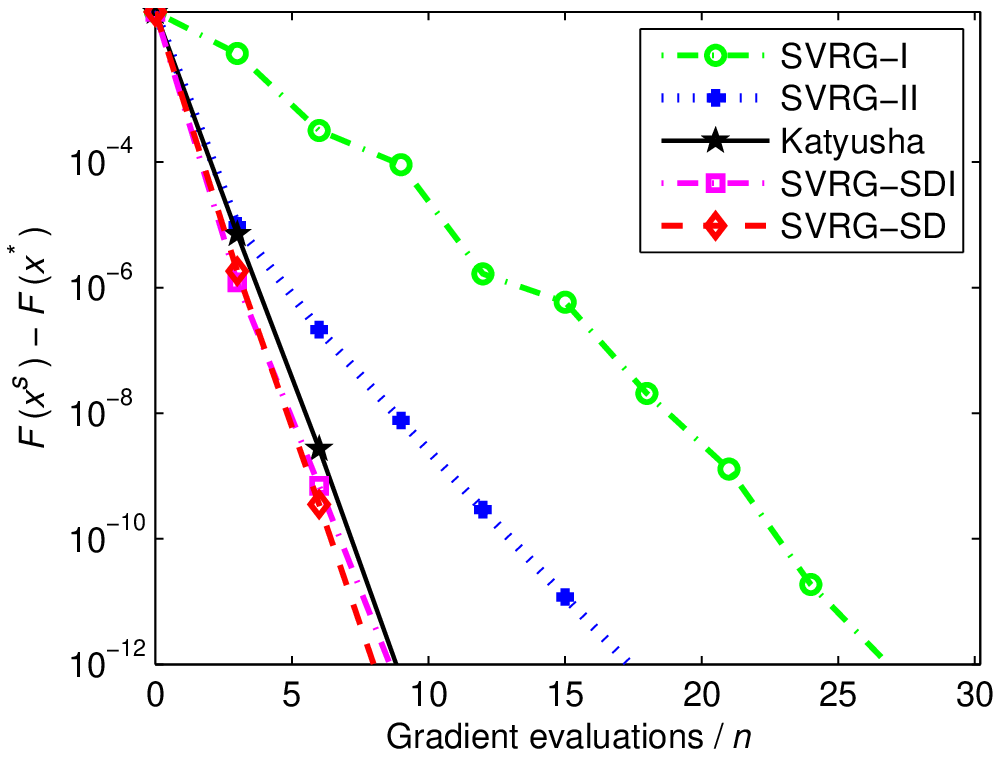}}
\subfigure[$\lambda\!=\!10^{-6}$]{\includegraphics[width=0.324\columnwidth]{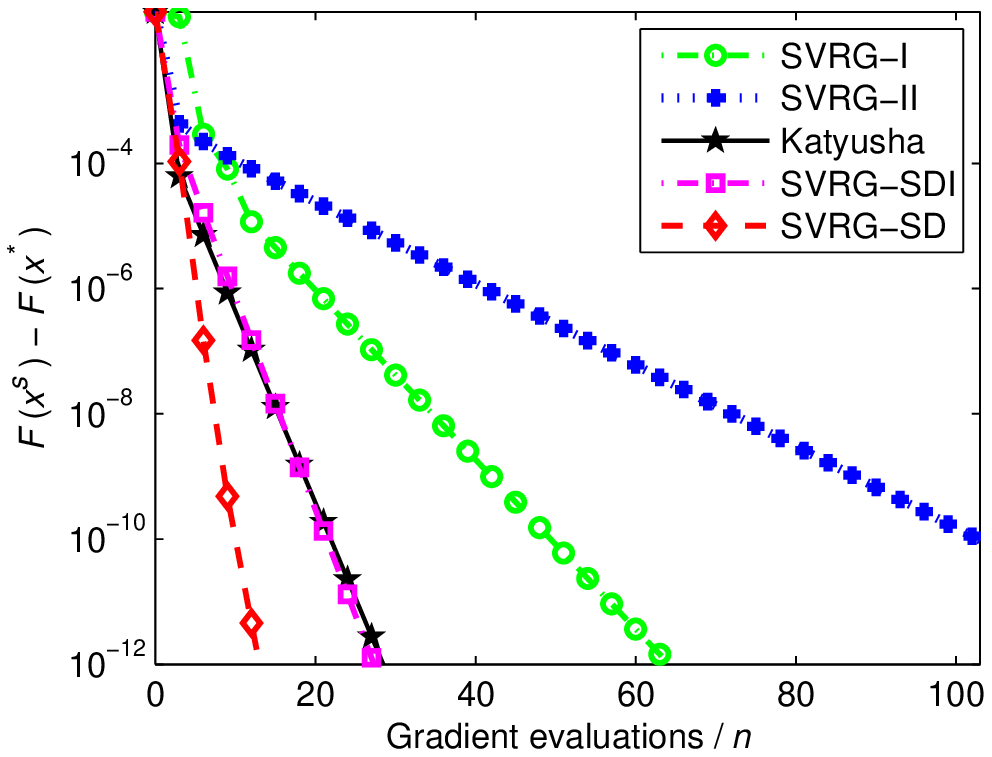}}\,
\subfigure[$\lambda\!=\!10^{-7}$]{\includegraphics[width=0.324\columnwidth]{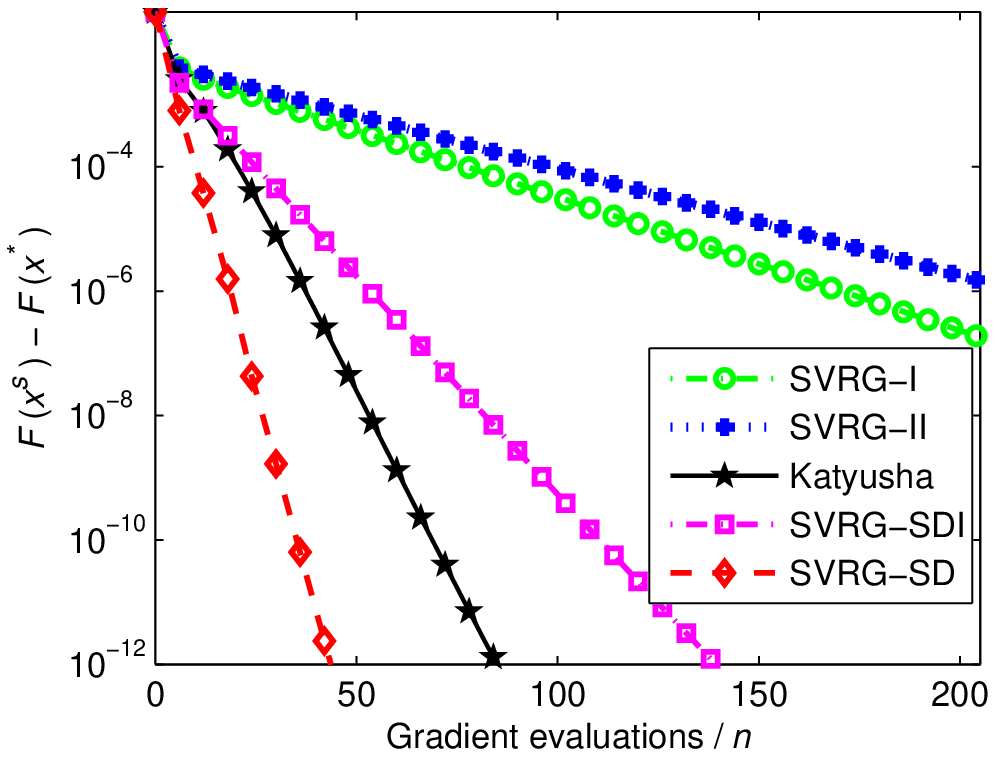}}\,
\subfigure[$\lambda\!=\!10^{-8}$]{\includegraphics[width=0.324\columnwidth]{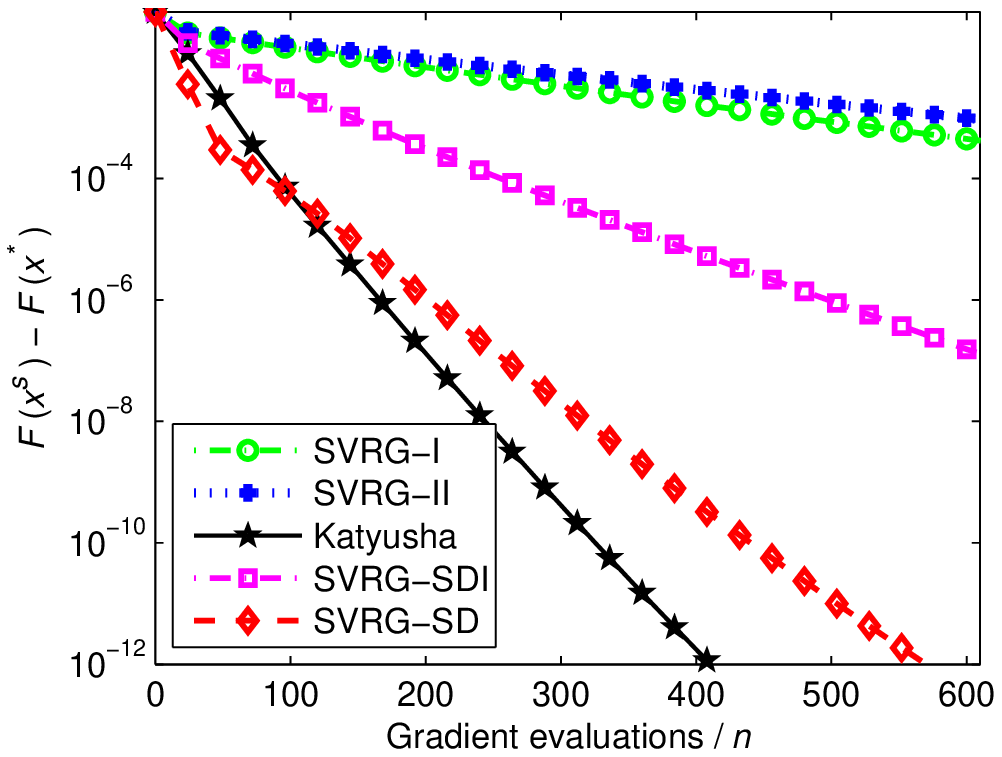}}
\subfigure[$\lambda\!=\!10^{-3}$]{\includegraphics[width=0.324\columnwidth]{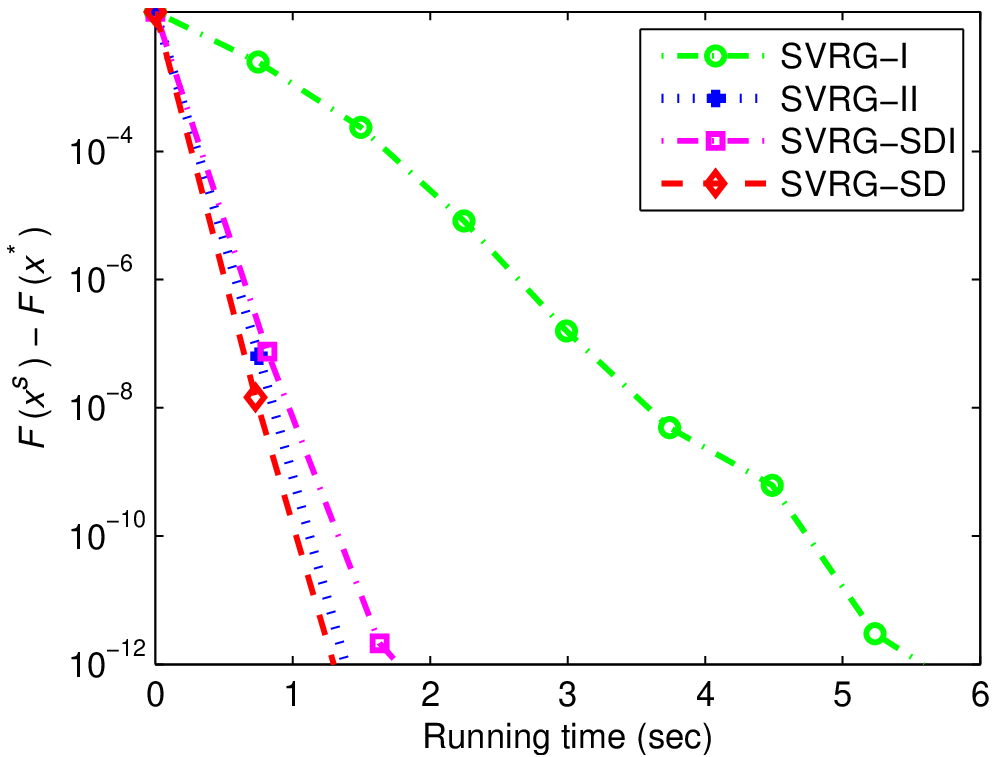}}\,
\subfigure[$\lambda\!=\!10^{-4}$]{\includegraphics[width=0.324\columnwidth]{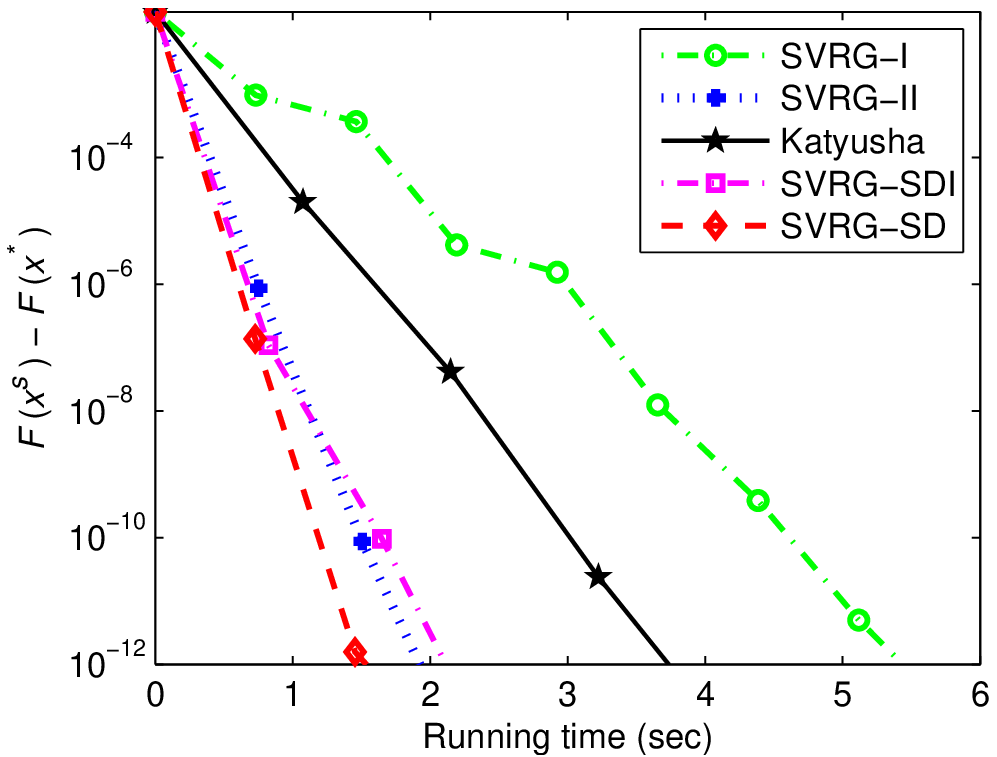}}\,
\subfigure[$\lambda\!=\!10^{-5}$]{\includegraphics[width=0.324\columnwidth]{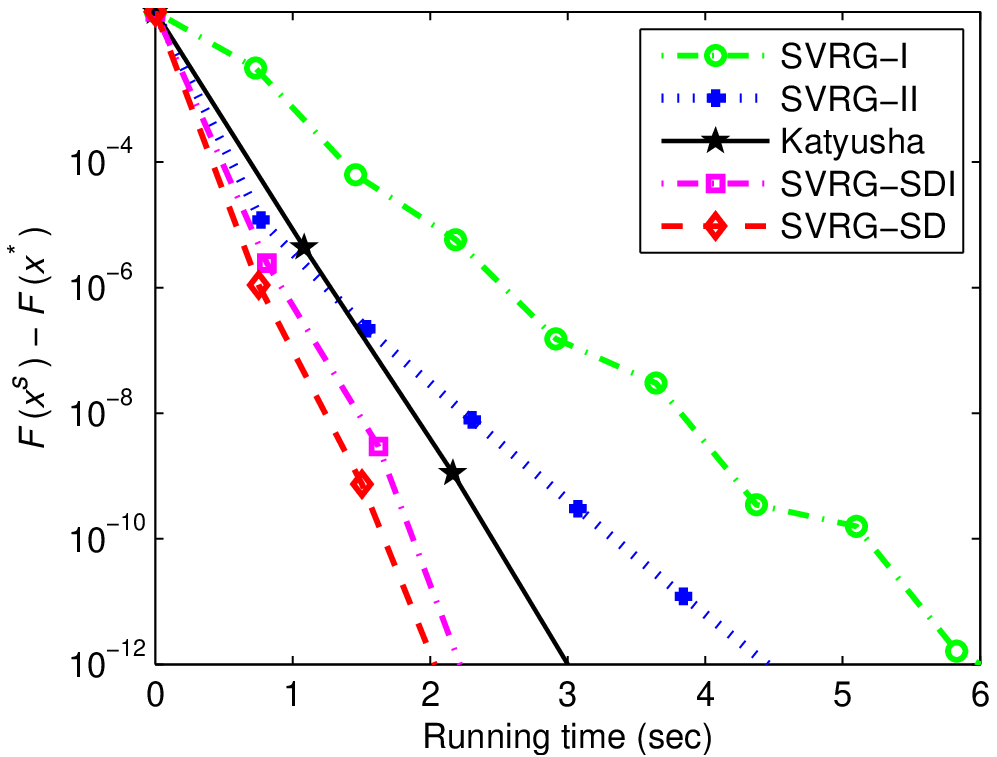}}
\subfigure[$\lambda\!=\!10^{-6}$]{\includegraphics[width=0.324\columnwidth]{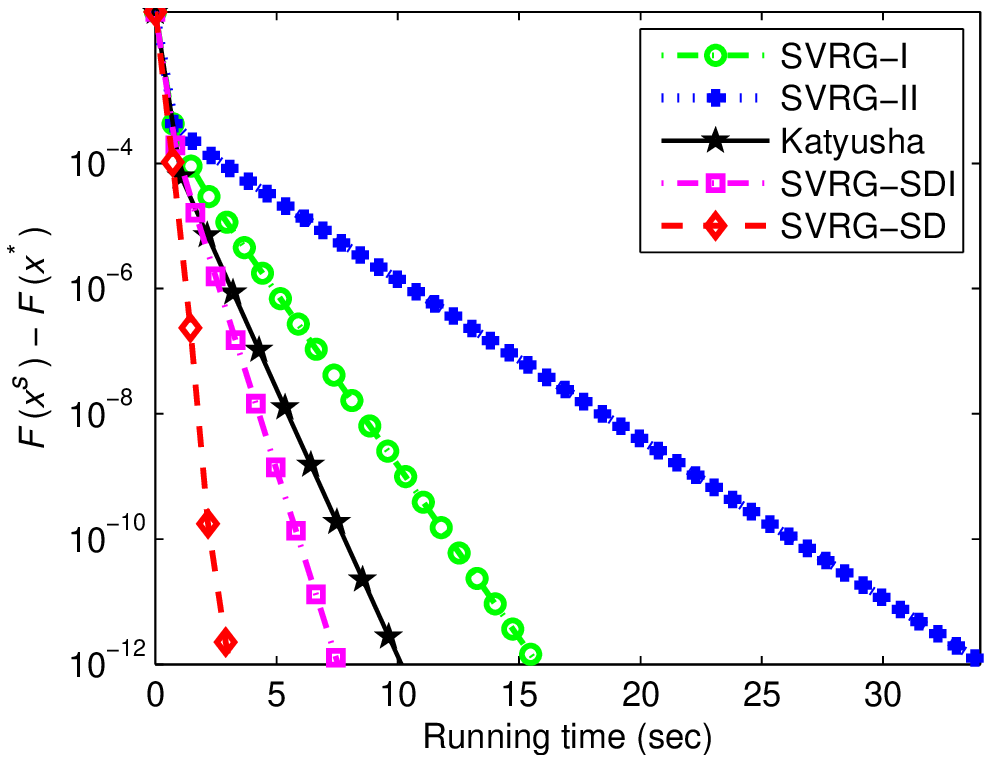}}\,
\subfigure[$\lambda\!=\!10^{-7}$]{\includegraphics[width=0.324\columnwidth]{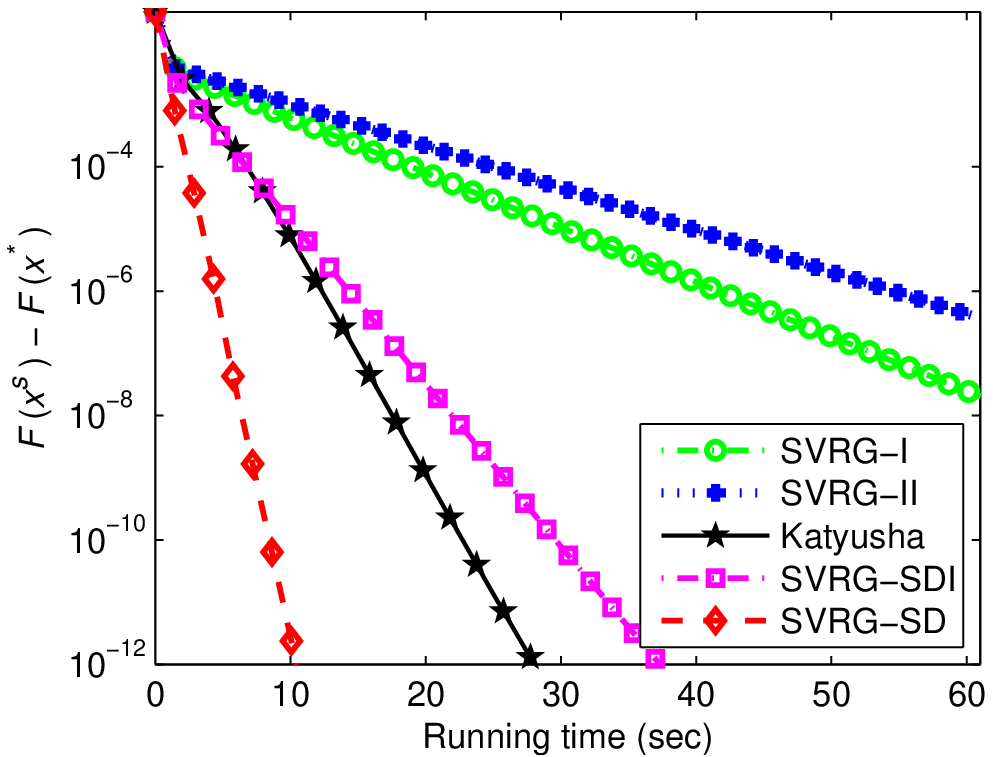}}\,
\subfigure[$\lambda\!=\!10^{-8}$]{\includegraphics[width=0.324\columnwidth]{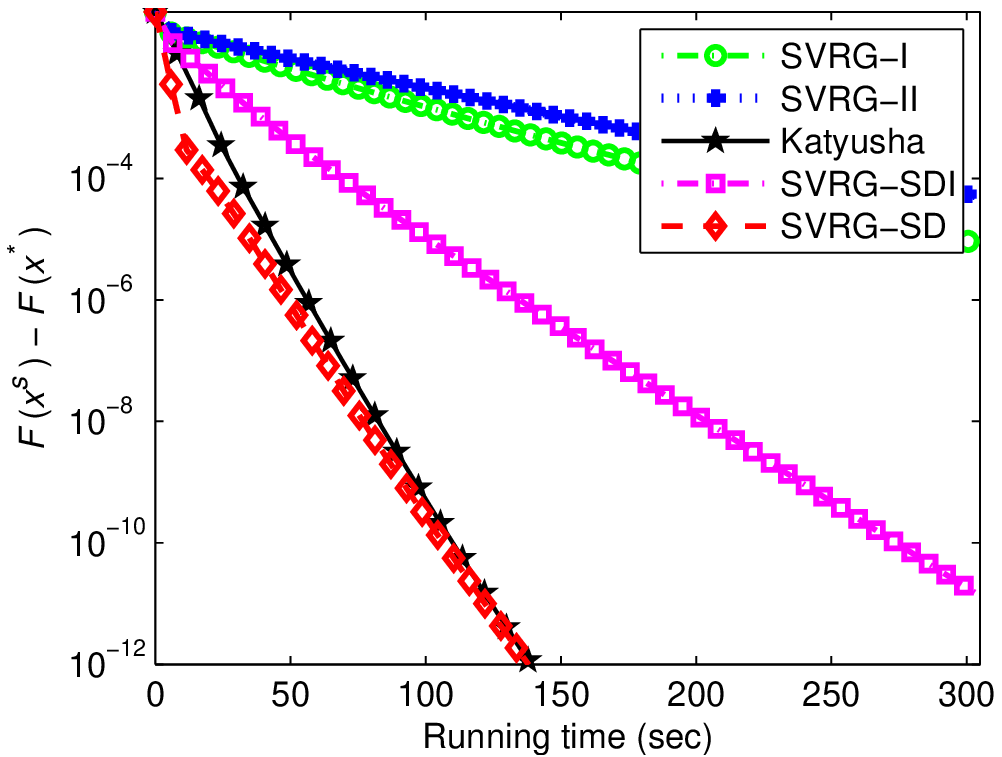}}
\caption{Comparison of SVRG-I, SVRG-II, Katyusha, and our SVRG-SD and SVRG-SDI methods for solving strongly convex ridge regression problems with different regularization parameters on the Covtype data set. The vertical axis represents the objective value minus the minimum, and the horizontal axis denotes the number of effective passes (a-f) or the running time (g-l).}
\label{fig_sim6}
\end{figure}

\small{
\bibliographystyle{plain}
\bibliography{nips2017}
}

\end{document}